\documentclass[11pt,letterpaper]{article}
\usepackage[T1]{fontenc}
\usepackage[margin=1in]{geometry}
\usepackage{times}
\usepackage[round,authoryear]{natbib}
\setcitestyle{authoryear,round,citesep={;},aysep={,},yysep={;}}
\usepackage{amsmath,amssymb,amsthm}
\usepackage{thm-restate}

\newtheorem{assumption}{Assumption}
\newtheorem{lemma}{Lemma}
\newtheorem{proposition}{Proposition}
\newtheorem{remark}{Remark}

\newtheorem{realizabilityproposition}{Proposition}[section]

\usepackage{amsmath,amsfonts,bm}

\def\Eqref#1{equation~\ref{#1}}
\def\Eqref#1{Equation~\ref{#1}}

\def\1{\bm{1}}

\DeclareMathAlphabet{\mathsfit}{\encodingdefault}{\sfdefault}{m}{sl}
\SetMathAlphabet{\mathsfit}{bold}{\encodingdefault}{\sfdefault}{bx}{n}

\newcommand{\KL}{D_{\mathrm{KL}}}

\usepackage{graphicx}
\usepackage{xcolor}
\usepackage[hypertexnames=false]{hyperref}
\usepackage{url}
\usepackage{algorithm}
\usepackage{algpseudocode}
\usepackage{etoolbox}
\usepackage{needspace}
\usepackage{booktabs,array}
\usepackage{tikz}
\usetikzlibrary{arrows.meta,positioning,calc}

\makeatletter
\let\paperappendixcontentsline\l@section
\patchcmd{\paperappendixcontentsline}{1.5em}{6.5em}{}{}
\newcommand{\l@appendixsection}[2]{%
  \begingroup
    \let\papernumberline\numberline
    \renewcommand{\numberline}[1]{\papernumberline{Appendix~##1}}%
    \paperappendixcontentsline{#1}{#2}%
  \endgroup
}
\newcommand{\appendixcontentsformat}{\let\l@section\l@appendixsection}
\newcommand{\proofcontentsentry}[2]{%
  \begingroup
    \small
    \@dottedtocline{2}{3.8em}{0em}%
      {\bfseries\hyperref[#1]{#2}}%
      {\bfseries\hyperref[#1]{\pageref*{#1}}}%
  \endgroup
}
\newcommand{\proofcontentschild}[2]{%
  \begingroup
    \small
    \@dottedtocline{2}{5.8em}{0em}%
      {\hyperref[#1]{#2}}%
      {\hyperref[#1]{\pageref*{#1}}}%
  \endgroup
}
\makeatother

\title{Learning without Overwriting: \\ A Theory of Self-Distillation and Supervised Fine-Tuning in Continual Reasoning}

\author{
   Shinichi Uemura\textsuperscript{$\ast$1}, Taiji Suzuki\textsuperscript{$\dagger$1, 2}\\[0.5em]
   \textsuperscript{1}Department of Mathematics Informatics, The University of Tokyo\\
   \textsuperscript{2}Center for Advanced Intelligence Project, RIKEN, Tokyo, Japan\\
   \textsuperscript{$\ast$}\texttt{uemura-shinichi666@g.ecc.u-tokyo.ac.jp}\\
   \textsuperscript{$\dagger$}\texttt{taiji@mist.i.u-tokyo.ac.jp}
}
\date{}

\begin{document}

\maketitle

\begin{abstract}
   On-policy self-distillation (OPSD) of large language models (LLMs) has demonstrated 
   the ability to improve reasoning capabilities while preserving previously acquired knowledge. 
   Despite substantial empirical success, the dynamics of OPSD in continual reasoning remain incompletely understood. 
   Modeling LLM reasoning as search over a directed acyclic graph, we provide a unified theoretical analysis of both the dynamics of 
   post-training---OPSD and supervised fine-tuning (SFT) in continual learning---and the impact of pre-training 
   on subsequent performance.
   Our findings establish three key insights with an optimization guarantee: 
   (i) OPSD with hints from correct outputs enables continual learning without forgetting by sparse yet effective gradient descent updates induced by the hint structure. 
   (ii) SFT on correct reasoning paths can lead to catastrophic forgetting due to dense updates along the training paths, which overwrite the information previously acquired. 
   (iii) Diversity in pre-training is crucial for enabling a post-trained model to reach a correct output when a rollout starts from an intermediate state. 
   Our results, supported by theoretical analysis, show that reliable continual reasoning depends 
   on how post-training updates interact with the reasoning structure established during pre-training.


\end{abstract}

\section{Introduction}
\label{sec:introduction}
Post-training has substantially improved the task-solving capabilities of large language models
through supervised fine-tuning
\citep{chung2022scalinginstructionfinetunedlanguagemodels, luo2025wizardcoderempoweringcodelarge,
zhang2024raftadaptinglanguagemodel} and reinforcement fine-tuning
\citep{Guo_2025, jin2025searchr1trainingllmsreason, zeng-etal-2025-acecoder}.
A central distinction between common implementations lies in the source of training trajectories:
off-policy SFT learns from fixed demonstrations
\citep{zhao2023surveylargelanguagemodels, zhang2023instructiontuninglargelanguage},
whereas on-policy reinforcement learning (RL) uses reward feedback on trajectories sampled from
the current model \citep{wang2023aligninglargelanguagemodels, zhao2026unified}.
In the latter, the generation distribution determines which experiences are encountered during
training, while itself evolving through subsequent updates
\citep{kaufmann2025survey, li2026reviewreinforcementlearning}.
The distribution inherited from pre-training therefore influences which reasoning paths are
sampled: even models with the same initial task accuracy can distribute probability
differently across paths and risk missing some useful reasoning steps.
Using a Markovian model of reasoning, \citet{bu2025distributional} showed that RL fine-tuning
can favor high-probability paths while suppressing rare but useful ones.
For SFT, preserving existing capabilities is a central challenge: sequential fine-tuning
can cause catastrophic forgetting of earlier tasks and degrade general capabilities
\citep{wang2023trace, li2024revisitingcatastrophicforgetting}.
Efforts to mitigate this problem include learning successive tasks in orthogonal low-rank
subspaces \citep{wang2023orthogonalsubspacelearning} and applying sharpness-aware minimization
to favor flatter loss landscapes \citep{li2024revisitingcatastrophicforgetting}.
Furthermore, in code instruction tuning, \citet{biderman2024loralearnslessforgets} reported both
reduced forgetting and better preservation of output diversity with LoRA than with full fine-tuning.

Self-distillation (SD) has also been explored as a means of learning new tasks while preserving
previously acquired capabilities. In recent on-policy variants, the student generates a trajectory, and
the same model---conditioned on an expert demonstration or informative feedback unavailable
to the deployed student---acts as a teacher. By deriving supervision from the model's own
behavior, self-distillation can reduce forgetting during adaptation
\citep{yang2024selfdistillationbridgesdistributiongap, shenfeld2026sdft}.
They also provide token-level supervision on student rollouts, with reported improvements
in reasoning accuracy over off-policy SFT and training efficiency over GRPO
\citep{hubotter2026sdpo, zhao2026sdr}. These methods leverage the pre-trained model's in-context
learning abilities \citep{shenfeld2026sdft, hubotter2026sdpo}, although SD can underperform RL
in some settings \citep{yang2026selfdistilledrlvr, kim2026doesselfdistillationsometimesdegrade}.
Their reliance on student trajectories also makes the diversity of the pre-trained
policy relevant to learning.

Despite empirical successes, the dynamics of SD and their dependence
on the pre-trained policy remain only partially understood.
\citet{wang2026statisticallimitsselfimprovingagents}, the general analysis of self-improvement, 
established high-probability guarantees for true-risk improvement and finite-sample generalization under self-modification.
Existing analyses for self-distillation characterize specific aspects of post-training.
Using a Gaussian mixture model,
\citet{balasubramanian2026quantitativecharacterizationforgettingposttraining} showed that SFT
on new-task data alone drives the mixture weight associated with previously learned behavior
to zero. For autoregressive models, \citet{sang2026crispcompressedreasoningiterative} bounded
policy drift within each fixed-teacher window in terms of the distillation loss and the shift
induced by a conciseness instruction. 
However, these analyses directly assume the teacher distribution
rather than derive it from hint information. They consequently leave unresolved how hints---unique to 
self-distillation---shape learning dynamics, which is more difficult because hint-based analysis requires
observations of the teacher behavior and its effect on the evolving student policy through multi-step reasoning.

Our objective is, therefore, to explain the difference in forgetting between on-policy SD and
off-policy SFT, while considering their ability to learn new tasks. We model Chain-of-Thought (CoT) reasoning as
a problem of reaching a task-specific goal in a DAG. 
This model links the sparsity of hint information to localized updates and
retention, and provides a quantitative comparison with SFT. We evaluate this comparison
numerically in Section~\ref{sec:numerical-experiments}. 
Although existing studies suggest that pre-training is important for the successful SD, there are few theoretical analyses on its role in on-policy SD.
Hence, as a complementary theoretical discussion, 
we examine in Section~\ref{sec:pretraining-reliability} how limited diversity in the pre-training
can leave relevant transitions unlearned and thereby impair reasoning from intermediate states.

\paragraph{Summary of contributions}Our contributions are summarized as follows:
\begin{itemize}
   \item We model the continual learning scenario in Section~\ref{sec:problem-setup} with a graph-based approach, which is often used
   to explain LLM reasoning process such as CoT reasoning (\citet{abbe2024fartransformersreasonglobality, NEURIPS2025_a8f12348,kim2025metastable}).
   In particular, we introduce a path-conflict setting to analyze the interference between tasks.
   \item
   We prove the learnability of a single task under both SD and SFT in Section~\ref{sec:single-task} (Theorems~\ref{thm:sd-correct-rollout} and~\ref{thm:sft-single-task}, respectively).
   Notably, Lemma~\ref{lem:sd-gradient} in Appendix~\ref{app:sd-gradient} shows that SD with a hint in \Eqref{eq:path-level-hint-feedback} can realize a sparse update that preserves the performance of the pre-trained model.
   Furthermore, we show that SD can learn new tasks while preserving previously acquired capabilities, whereas SFT can forget them in Theorem~\ref{thm:sd-no-forgetting} and Theorem~\ref{thm:sft-forgetting}.
   These results are aligned with the empirical findings of \citet{shenfeld2026sdft}.
   \item {\color{black}We conduct experiments in Section~\ref{sec:numerical-experiments} to support our theoretical bounds and intuitions to the real-world dynamics}. The results confirm that
   tracing the hint-based sparse updates of SD and the overwriting properties of SFT is helpful to explain the observed learning dynamics.
   \item As a complementary theoretical analysis, we study the effect of pre-training diversity on reasoning from intermediate states in Section~\ref{sec:pretraining-reliability} (Proposition~\ref{prop:restart-negative}).
   We show that limited coverage in the pre-trained policy can leave relevant transitions unlearned even after SD succeeds from the original start.
\end{itemize}
These results explain how the sparsity of hint information translates into localized updates
and retention in continual reasoning. The complementary pre-training analysis gives a
theoretical implication for reasoning from intermediate states.

\subsection{Related work}
In the following, we briefly discuss related work on self-distillation and diversity in LLM reasoning.
Additional related work are discussed in Appendix~\ref{app:related-works}.

\paragraph{Self-distillation for LLMs.}
Recent methods differ in the information supplied to the self-teacher.
\citet{shenfeld2026sdft} introduced SDFT, conditioning the teacher on pre-existing expert demonstrations to supervise student-generated trajectories.
Their experiments demonstrated new-task acquisition with reduced forgetting relative to SFT.
\citet{hubotter2026sdpo} instead conditioned the SDPO teacher on feedback obtained after evaluating the student's own attempts, such as execution errors, improving reasoning accuracy and training efficiency.
\citet{zhao2026sdr} proposed OPSD, providing the teacher with reference answers or verified reasoning traces, and demonstrated improved mathematical reasoning with greater token efficiency than GRPO.
\citet{ye2026opcd} introduced OPCD, including a self-distillation setting, using knowledge extracted from past solution traces or optimized system prompts as teacher context.
Despite substantial empirical evidence, the learning dynamics of SD remain theoretically underexplored.
\paragraph{Diversity in LLM reasoning.}
\citet{wang2023selfconsistency} demonstrated that aggregating answers across diverse reasoning paths improves reasoning accuracy
with numerical experiments.
\citet{yao2025diversity} proposed diversity-aware policy optimization, which promotes token-level diversity selectively on correct responses.
Their experiments showed improvements in both solution diversity and performance on mathematical reasoning benchmarks.
Using a Markovian model, \citet{bu2025distributional} theoretically showed that post-training can suppress rare but useful reasoning paths 
and emphasize the role of exploration in preserving rare but useful reasoning paths already present in the pre-trained model.


\begin{figure}[t]
   \centering
   \includegraphics[width=1.\textwidth]{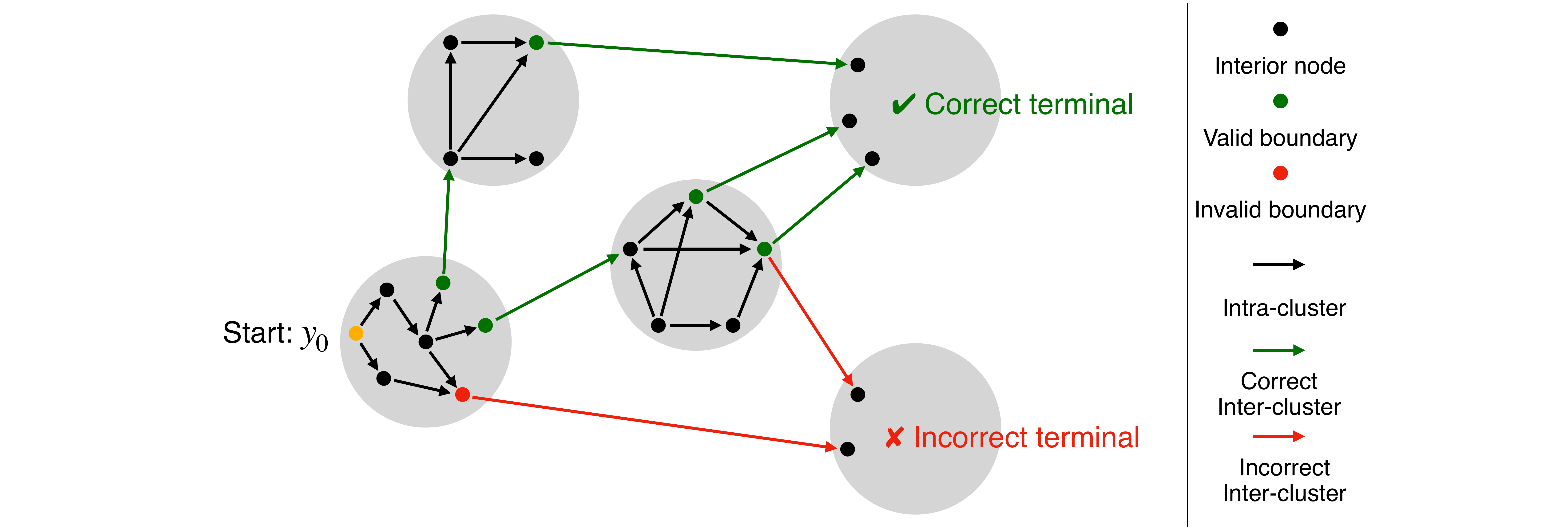}
   \caption{
      Example of the CoT graph we considered in Section~\ref{sec:problem-setup}. Each node represents a reasoning state,
      and edges correspond to reasoning transitions. Black edges indicate easy intra-cluster transitions,
      while red and green ones mean difficult inter-cluster transitions. 
      The goal of training is to improve difficult transitions and to reach the correct terminal.
      Therefore, the model is required to learn the inter-cluster transitions while preserving the diversity within clusters.
   }
   \label{fig:cot-graph}
\end{figure}
\section{Problem Setup}
\label{sec:problem-setup}

\subsection{Continual learning on CoT graph}
Following the reasoning-graph construction of \citet{kim2025metastable}, we model the LLM chain-of-thought process as Markov transitions on a directed acyclic graph (\textbf{DAG}).
Let the LLM CoT graph be $G = (V, E)$, where $V$ is the node set and $E$ is the directed edge set.
A model receives input $x \in \mathbb{R}^d$ and generates a CoT sequence $y = (y_0, \ldots, y_t, \ldots, y_T)$, where $y_t \in V$, $(y_t, y_{t + 1}) \in E$,
and $T$ denotes a rollout-dependent termination time.
Note that $y$ is node-simple since the graph is acyclic.
We consider a correspondence between these states and the graph nodes; that is, each node in the graph is regarded as a reasoning state $y_t$
like a logical sentence or a mathematical expression rather than a simple token.
Thus, the prediction goal is to reach a proper node or cluster through logical reasoning transitions.
The graph $G$ consists of $K$ clusters denoted by $C_1 \ldots C_K$,
and it is free from multiple edges. A cluster is a set of similar reasoning states like equivalent reformulations
obtained by straightforward manipulation. 
Hence, we consider the inter-cluster edges represent the essential and conceptual transitions,
while intra-cluster edges are something trivial.
For every node $u \in C_k$ and every cluster $C_{k'}$ with $k' \neq k$,
there is at most one node $v \in C_{k'}$ such that $(u,v) \in E$.
We separate the cluster nodes in $C_k$ into two types: one is the interior nodes denoted by $\mathring{C}_k$,
and the other is the boundary nodes denoted by $\partial C_k$
\begin{equation}
   \mathring{C}_k \mathrel{:=} \{u \in C_k: \forall v \in V, (u, v) \in E \Rightarrow v \in C_k\}, \quad
   \partial C_k \mathrel{:=} C_k \backslash \mathring{C}_k.
\end{equation}
For preserving the inner-cluster transition diversity, we assume that all destinations of inter-cluster edges are interior nodes.
\paragraph{Continual learning.}
We study continual learning of $L$ tasks $\mathcal{D}_{1}, \ldots, \mathcal{D}_{L}$.
The task $\mathcal{D}_{\ell}$ consists of a fixed input $x_\ell \in \left\{0, 1\right\}^d$,
a fixed CoT startpoint $y_{\ell, 0} \in \mathring{C}_{k_\ell}$, and a task-specific correct terminal cluster $\tilde{C}_{\ell} \in \left\{C_1, \ldots, C_K\right\}$.
We assume that $\tilde{C}_{\ell}$ has no outgoing inter-cluster transitions.
Furthermore, we assume that the tasks can be separated. More precisely, the matrix $(x_1, \ldots, x_L)$ is column full rank and the clusters $\tilde{C}_{1}, \ldots, \tilde{C}_{L}$ are disjoint.
A rollout for task $\mathcal{D}_{\ell}$ is defined to be correct if and only if $y_T \in \tilde{C}_{\ell}$.
\footnote{
We focus on reaching the terminal cluster.
For the subsequent search for a particular target node within it,
\citet[Lemma~B.9]{kim2025metastable} provide hitting-time bounds
under additional assumptions on within-cluster mixing and
stationary probabilities.
}
Then, we consider the following assumption.

\begin{assumption}[transition uniqueness and competition between tasks]
\label{ass:transition-compatibility}
   Each boundary node has at most one correct cluster transition for each task. 
   Whenever such a transition exists for multiple tasks 
   at the same boundary node, it is identical across any tasks that are not orthogonal in the input.
   Hence, competition between tasks can arise only through trivial interior transitions.
\end{assumption}

This assumption is strong, but seems reasonable since similar tasks share underlying reasoning properties, 
which we model as common inter-cluster transitions. {\color{black} Indeed, LLM was experimentally observed to reuse heavily the same reasoning process between different tasks
within specific parameters (\citet{hanna2023doesgpt2computegreaterthan, merullo2024circuitcomponentreusetasks}).}
$\mathcal{P}_\ell$ is the set of correct paths for task $\mathcal{D}_\ell$, and
$E_\ell$ is the union of the inter-cluster transitions appearing in those correct paths.
In learning, the model is trained sequentially on the tasks $\mathcal{D}_{1}, \ldots, \mathcal{D}_{L}$.

\subsection{Pre-training and fine-tuning}
\paragraph{Pre-training requirements}For each task $\mathcal{D}_{\ell}$ and cluster $C_k$, define the set of boundary nodes that admit a correct inter-cluster transition by
\begin{align}
   \partial^{\star}_{\ell}C_k = \left\{u \in \partial C_k \mid \exists v \in V\backslash C_k, (u, v) \in E_{\ell}\right\}.
\end{align}
Then, we further define $M = \max_{k, \ell} \left|\partial^{\star}_{\ell}C_{k}\right|$.
We call the initial node or an entry through a valid inter-cluster transition a relevant start $s$ in $C_k$.
For a rollout initialized at $Y_0=s$ under task input $x_\ell$, define
\begin{align*}
   \tau_{\ell,k}^{s}
   :=\inf\{t\geq0:Y_t\in\partial C_k\},\quad
   U_{\ell,k}^{s}
   :=Y_{\tau_{\ell,k}^{s}}.
\end{align*}
Thus $U_{\ell,k}^{s}$ is the first boundary node reached after the rollout starts from the relevant start $s$.
For the model, we consider a softmax predictor:
\begin{align}
   \pi_{\mathbf{W}}(\cdot \mid u, x) = \text{softmax}(x^{\top} \mathbf{W} e_{u, \cdot}), \quad \mathbf{W} \in \mathbb{R}^{d \times \left|E\right|},
\end{align}
where $e_{u, \cdot}$ is the embedded concatenation matrix of one-hot vectors $e_{u, v}$ for all $v \text{ s.t. }(u, v) \in E$.
The following assumption specifies the required properties of the pre-trained parameter $\mathbf{W}^{(0)}$.

\begin{assumption}[Pre-trained diversity]
\label{ass:pretrained-diversity}
   For every task $\mathcal{D}_{\ell}$ and non-terminal cluster $C_k$ satisfying $\partial^{\star}_{\ell}C_k \neq \emptyset$, 
   the first boundary-hitting time after every relevant start $s$ in $C_k$ is almost surely finite.
   Uniformly over all such $s$, the pretrained model satisfies
   \begin{align}
      \mathbb{P}_{W^{(0)}}\left(U_{\ell, k}^s \notin \partial^{\star}_{\ell}C_{k} \mid Y_0=s,x_\ell\right)
      &= O(\epsilon_1),\label{eq:pretrained-boundary-error}\\
      \mathbb{P}_{W^{(0)}}\left(U_{\ell, k}^s = u \mid Y_0=s,x_\ell\right)
      &=\Omega(1/M),\label{eq:pretrained-boundary-coverage}
   \end{align}
   for every $u \in \partial^{\star}_{\ell}C_{k}$. 
   Moreover, $\pi_{W^{(0)}}(v\mid u,x_\ell) = \Theta(\epsilon_2)$ for every inter-cluster edge $(u,v)$ and every task input $x_\ell$.
   $\epsilon_1$ and $\epsilon_2$ are supposed to be sufficiently small; at least, they satisfy $KM\epsilon_1 = o(1)$ and $M\epsilon_2 = o(1)$.
\end{assumption}

The diversity conditions state that 
(i) the pre-trained model reaches valid boundaries with high probability, and 
(ii) every inter-cluster transition initially has small but nonzero probability.
In other words, these assumptions suggest that the trivial reasoning steps are sufficiently acquired,
but essential transitions are difficult and the model is required further training to learn.
Appendix~\ref{app:realizability} gives a finite-parameter satisfying every condition of Assumption~\ref{ass:pretrained-diversity} on a family of layered DAGs.

\paragraph{Fine-tuning settings}
We use SD and SFT as on- and off-policy fine-tuning methods respectively. 
The detailed algorithms are described in Algorithm~\ref{alg:sequential-training}.
For each task $\mathcal{D}_{\ell}$, a teacher-forcing path $y^{*}=(y_0^{*},\ldots,y_T^{*})$ and a hint $c$ are constructed by sampling
\begin{align}
   \gamma \sim \operatorname{Unif}(\mathcal P_{\ell}).
   \label{eq:uniform-correct-path-sampling}
\end{align}
Across training iterations and across elements within each mini-batch, the teacher paths are sampled independently according to Equation~\ref{eq:uniform-correct-path-sampling}, conditional on the preceding training history.
Note that for SD, conditional on the same history, each teacher path used to construct the hint and its corresponding unhinted student rollout are sampled independently.
We define the teacher path to be $y^{*}=\gamma$, and the hint $c$ represents all cluster transitions in $\gamma$;
that is, the hint is $c = \left\{c_i\right\}_{i = 1}^I$,
where $c_i = \left\{(u_i, v_i) \text{ is traversed.}\right\}$. 
This hint construction reflects the human instruction that ignores trivial transitions and emphasizes the essential reasoning steps.
At each $\mathcal{D}_{\ell}$ training iteration, the model receives a minibatch 
of $\{x_{\ell}, c^{(b)}\}_{b = 1}^B$ or $\{x_{\ell}, y^{*(b)}\}_{b = 1}^B$ for SD or SFT respectively, and 
updates the parameter $\mathbf{W}$ by minimizing the loss function.
Note that such an external hint is consistent with \citet{shenfeld2026sdft} and \citet{liu2026selfdistilledpolicygradient}.
Since SD uses only hints about specific transitions rather than the more informative full paths used in SFT, pre-training is crucial for its success.
We use the same uniforum distribution for SD and SFT to compare their performance under a common source of supervision without favoring any correct path.
Batch size $B$ is fixed for all tasks and iterations.
The mini-batch loss functions for SD and SFT are defined as follows respectively:
\begin{align}
   L_{\text{SD}}(\mathbf{W}; \mathcal{D}_{\ell}) &= \frac{1}{B}\sum_{b = 1}^B\sum_{t=0}^{T^{(b)}-1} \KL\left(\pi_\mathbf{W}(\cdot \mid y_t^{(b)}, x_{\ell}) \parallel \text{SG}(\pi_{\mathbf{W}}(\cdot \mid y_t^{(b)}, x_{\ell}, c^{(b)}))\right),\\
   L_{\text{SFT}}(\mathbf{W}; \mathcal{D}_{\ell}) &= -\frac{1}{B}\sum_{b = 1}^B\sum_{t=0}^{T^{(b)}-1} \log \pi_{\mathbf{W}}(y_{t+1}^{*(b)} \mid y_{t}^{*(b)}, x_{\ell}),
\end{align}
where \(\text{SG}\) denotes stop-gradient.
Especially, we do not batch the KL-divergence to use the same SD loss as \cite{shenfeld2026sdft} and \cite{liu2026selfdistilledpolicygradient}.
Both SD and SFT update the parameters using SGD as $\mathbf{W}^{(n + 1)} \to \mathbf{W}^{(n)} - \eta_j\nabla_{\mathbf{W}}L_{\mathrm{method}}(\mathbf{W}^{(n)}; \mathcal{D}_{\ell}),$
where $\mathrm{method}\in \{\text{SD, SFT}\}$ denotes the corresponding training method.

\section{Main results}
If the hint $c$ strictly indicates that the model should follow all cluster shift $(u_i, v_i)$ in $\gamma$,
then the teacher support would be strictly subset of the student one, and the KL divergence would diverge.
Avoiding this divergence, fix a smoothing parameter \(\epsilon_{\mathrm{rel}}\in(0,1)\) and use a relaxed hint $c$:
\begin{align}
   \chi_{\mathbf{W}}(c_i \mid y_{0:T}, x) \propto
   \begin{cases}
      1 & (\tau_i < T, y_{\tau_i + 1} = v_i),\\
      \epsilon_{\mathrm{rel}} & (\tau_i < T, y_{\tau_i + 1} \neq v_i),\\
      Z_{\mathbf{W}, i}(x) & (\tau_i \geq T),
   \end{cases}
   \label{eq:path-level-hint-feedback}
\end{align}
where $\tau_i = \inf\left\{s \in \{0,\ldots,T\} \mid y_s = u_i\right\}$, with the convention \(\inf\emptyset=\infty\), and $Z_{\mathbf{W}, i}(x) = \pi_{\mathbf{W}}(v_i \mid u_i, x) + \epsilon_{\mathrm{rel}}(1 - \pi_{\mathbf{W}}(v_i \mid u_i, x))$.
$\chi_{\mathbf{W}}$ is a likelihood factor for the relaxed hint $c_i$, and $Z_{\mathbf{W}, i}(x)$ is equivalent to the expected local feedback before the transition from $u_i$ is observed.
When $u_i$ is not visited, using this expected value as feedback allows us to leave the corresponding transition parameters
unchanged, since the current student provides no information about transition out of $u_i$.
This edge-level feedback enables the sparse activation, which identifies where mistakes happen and adjusts to the student's response distribution
(\citet{hubotter2026sdpo}).
From \eqref{eq:path-level-hint-feedback}, the teacher distribution is determined by Bayes' rule:
\begin{align}
   \pi_{\mathbf{W}}(y_{0:T} \mid c, x) \propto \pi_{\mathbf{W}}(y_{0:T} \mid x)\cdot\chi_{\mathbf{W}}(c \mid y_{0:T}, x).
   \label{eq:bayes-rule}
\end{align}
The step-level teacher distribution $\pi_{\mathbf{W}}(y_{t+1} \mid y_{t}, c, x)$ is derived in Appendix~\ref{app:sd-gradient}.
\paragraph{Sparse structure learning.}
We believe that hint-induced sparse learning mitigates catastrophic forgetting.
Indeed, as shown in Lemma~\ref{lem:sd-gradient} in Appendix~\ref{app:sd-gradient},
the hint in \Eqref{eq:path-level-hint-feedback} yields a gradient that
updates only the parameters associated with outgoing edges from the hint-source nodes.
Consequently, SD preserves the interior transition diversity, which is crucial for
continual learning, whereas SFT updates the parameters at all source nodes along
the teacher path.


\subsection{Single-task learnability}
\label{sec:single-task}
\paragraph{SD as on-policy training} 
Before analyzing the continual learning performance, we first consider their ability to learn a single task.
Theorems~\ref{thm:sd-correct-rollout} and~\ref{thm:sft-single-task} suggest that SD and SFT are both effective for solving one task.
It is important to note that there is little difference between the two methods
in learning a single task.

\begin{restatable}[Single-task learning under SD]{theorem}{sdrollouttheorem}
\label{thm:sd-correct-rollout}
Fix a task $\mathcal D_\ell$ and let the stepsize $\eta_j=\Theta(j^{-\nu})$ for some
$\nu\in[0,1/2)$. Suppose Assumptions~\ref{ass:transition-compatibility} and~\ref{ass:pretrained-diversity} hold and $\epsilon_1$ is
sufficiently small. 
Define
\begin{gather}
   \xi_{\ell,n}^{(K)}(\delta)
   \mathrel{:=}
   \left[
      \epsilon_2+
      \epsilon_2^2\left(\log\frac{1}{\epsilon_{\mathrm{rel}}}\right)
      \left[
      \frac{e^{-\Theta(K\log\frac{1}{\epsilon_2})}n^{1 - \nu}}{M^K|\mathcal{P}_{\ell}|}
      -\sqrt{
         \frac{\log(KM/\delta)}{B}
      }n^{1/2 - \nu}
   \right]^+
   \right]^{-1}.
   \label{eq:theorem-one-xi}
\end{gather}
Then, with probability at least $1-\delta$,
\begin{align}
   \pi_{\mathbf W^{(n)}}\left(
      y_T\in\tilde C_\ell\mid y_{\ell,0},x_\ell
   \right)
   \geq
   1
   -O\!\left(
      KM\epsilon_1
   \right)
   -O\!\left(
      KM\xi_{\ell,n}^{(K)}(\delta)
   \right).
   \label{eq:theorem-one-correct-rollout}
\end{align}
\vspace{-1.em}
\end{restatable}

Theorem~\ref{thm:sd-correct-rollout} says that SD can learn a single task with a success probability of $1 - O(KM\epsilon_1)$.
The error term arises from visits to invalid boundaries, and in this setting,
we reduce this error by assuming $\epsilon_1$ is sufficiently small.
The error $O(KM\epsilon_1)$ reflects a trade-off between single-task accuracy and the ability
to retain previously acquired knowledge during continual learning.
Preserving the student's pre-trained diversity also leaves some of its initial errors intact.
On the other hand, this term might be improved by providing hints that
discourage the model from proceeding in an incorrect direction.
Indeed, \citet{zhu2025negative} show that penalizing incorrect responses alone can improve reasoning performance,
and their gradient analysis indicates that these updates suppress incorrect generations and redistribute probability mass
toward plausible alternatives favored by the model's prior.
During inference, \citet{yao2023tot} use negative evaluations of intermediate reasoning states to prune unpromising branches
in Tree of Thoughts, with ablation results showing that removing this pruning degrades search performance.
{\color{black}
Furthermore, $\xi_{\ell, n}^{(K)}$ is monotonic $\epsilon_{\text{rel}}$,
which means that the theorem effectively captures how hint strength affects self-distillation dynamics.}

\paragraph{SFT as off-policy learning} Theorem~\ref{thm:sft-single-task} is a counterpart of Theorem~\ref{thm:sd-correct-rollout} for SFT. SFT also learns a single task with high probability,
and updates the parameters at all source nodes along the teacher. It is a main difference from SD,
which only focuses on the source nodes related to the hint.

\begin{restatable}[Single-task learning under SFT]{theorem}{sftlearningtheorem}
\label{thm:sft-single-task}
Fix $\delta\in(0,1)$. Suppose the stepsize $\eta_j=\Theta(j^{-\nu})$ for some $\nu\in(0,1/2)$ and
$\mathbf W_{\mathrm{SFT}}^{(0)}=\mathbf W^{(0)}$.
Then, with probability at least $1-\delta$,
\begin{align}
   \min_{u\in\mathcal U_\ell}
   \pi_{\mathbf W_{\mathrm{SFT}}^{(n)}}
   (\mathcal A_\ell(u)\mid u,x_\ell)
   \geq
   1-\tilde{O}\!\left(|\mathcal P_\ell|\left[
   \frac{
      1
   }{
      n^{1-\nu}
   }
   +
   \frac{
      \log(e/\delta)
   }
   {
      Bn^{\nu}
   }
   \right]\right).
   \label{eq:theorem-two-final-phase}
\end{align}
Here $\mathcal U_\ell$ is the set of nodes appearing in the correct paths except for the terminal,
and $\mathcal A_\ell(u)$ is the set of their correct successors from $u$.
$\tilde{O}(\cdot)$ hides logarithmic factors in $n$ and $|\mathcal P_\ell|$.
\vspace{-.8em}
\end{restatable}

\begin{proof}[Sketch of Proof]
   Note that the optimal parameter $\mathbf{W}_{\text{opt}}$ is not finite because a softmax predictor with a finite parameter
   cannot assign probability one to any outcome. Thus, we analyze the convergence in two steps:
   (1) Set parameter $\mathbf{W}_R \to \mathbf{W}_{\mathrm{opt}} (R\to \infty)$, and use it as a reference.
   (2) Construct the potential function to measure 
   the difference between the current parameter, the reference parameter and the optimal one. 
   More concrete analysis is provided in Appendix~\ref{app:sft-learning}
   \vspace{-1em}
\end{proof}
Theorem~\ref{thm:sft-single-task} shows that SFT learns to select correct successors from states on correct paths.
Whereas SD carries over some of the pre-trained model's routing errors, Theorem~\ref{thm:sft-single-task} show that
SFT can learn to generate correct paths with probability approaching one.
This difference reflects the scope of the updates. SFT supervises all transitions along the teacher paths,
including intra-cluster transitions, and
can therefore correct the pre-trained routing errors reflected in the residual term of the SD lower bound.
In the next section, we use this overwriting property to explain why SFT can forget previously learned tasks.

\subsection{Continual learning and catastrophic forgetting}
\paragraph{Sustainability of SD}
In this section, we study the continual learning performance of SD and SFT.
A key feature of SD in our framework is that hints localize parameter
updates to a small subset of graph nodes while SFT fails and overwrites the parameters of the student model.
Intuitively, such localized updates may help preserve existing knowledge
and mitigate catastrophic forgetting.
Relatedly, \citet{wang2026entropypreservingsupervisedfinetuningadaptive}
suggest that self-distillation can mitigate knowledge forgetting by
selectively encouraging exploration at reasoning tokens while maintaining
stable targets for factual tokens.
The following theorem supports our intuition.
\begin{restatable}[No forgetting under SD]{theorem}{sdretentiontheorem}
\label{thm:sd-no-forgetting}
Let $\mathbf W_{\mathrm{SD}}^{[s]}$ denote the parameter after sequentially
training $\mathcal D_1,\ldots,\mathcal D_s$ by SD. Fix
$\ell\in\{1,\ldots,L\}$, and let $n_\ell$ be the number of SD updates used for
$\mathcal D_\ell$. Suppose Assumptions~\ref{ass:transition-compatibility} and~\ref{ass:pretrained-diversity} hold for all tasks,
and the stepsize $\eta_j=\Theta(j^{-\nu})$ for some $\nu\in[0,1/2)$.
Then, after any numbers of subsequent SD updates on
$\mathcal D_{\ell+1},\ldots,\mathcal D_L$, with probability at least
$1-\delta$,
\begin{align}
   \pi_{\mathbf W_{\mathrm{SD}}^{[L]}}\left(
      y_T\in\tilde C_\ell\mid y_{\ell,0},x_\ell
   \right)
   \geq
   1
   -O\!\left(
      KM\epsilon_1
   \right)
   -O\!\left(
      KM\xi_{\ell,n_\ell}^{(K)}(\delta)
   \right).
   \label{eq:theorem-three-no-forgetting}
\end{align}
Moreover, every correct transition learned for task $\mathcal D_\ell$ satisfies
\begin{align}
   \pi_{\mathbf W_{\mathrm{SD}}^{[L]}}(v\mid u,x_\ell)
   \geq
   \pi_{\mathbf W_{\mathrm{SD}}^{[\ell]}}(v\mid u,x_\ell),
   \qquad (u,v)\in E_\ell.\label{eq:theorem-three-edgewise-monotonicity}
\end{align}
\end{restatable}
Theorem~\ref{thm:sd-no-forgetting} shows that, after learning task
$\mathcal D_\ell$, the student can continue learning subsequent tasks
through self-distillation while retaining its ability to generate
correct rollouts from $y_{\ell,0}$
(Equations~\ref{eq:theorem-three-no-forgetting}
and~\ref{eq:theorem-three-edgewise-monotonicity}).
Equation~\ref{eq:theorem-three-edgewise-monotonicity} also shows that the probability of correct inter-cluster transitions
may even increase by subsequent tasks when the tasks share the same correct inter-cluster transitions because the hints
associated with these transitions are sampled as well during training on subsequent tasks.
\vspace{-.5em}
\paragraph{Forgetting in SFT} SFT, on the other hand, is a learning system that forces the student to follow the teacher path, which is a 
relatively dense graph structure.
Therefore, if any path conflict exists between tasks, the student rollout is likely to follow
new task paths and forget the old ones.
\textsl{We consider this overwriting is precisely what causes catastrophic forgetting.}
For the SFT comparison, fix two tasks $\mathcal D_{\mathrm{old}}$ and
$\mathcal D_{\mathrm{new}}$ learned in this order. Let
$\mathbf W_{\mathrm{new}}^{(0)}$ be the parameter immediately after learning
$\mathcal D_{\mathrm{old}}$, and let $\mathbf W_{\mathrm{new}}^{(n)}$ be the
parameter after $n$ subsequent SFT updates on $\mathcal D_{\mathrm{new}}$.
Let $\theta\in[0,1]$ denote the fraction of old correct paths that conflict with
the new task. Formal definitions are deferred to Appendix~\ref{app:sft-forgetting}.

\begin{restatable}[Catastrophic forgetting under SFT]{theorem}{sftforgettingtheorem}
\label{thm:sft-forgetting}
Fix $\delta\in(0,1)$ and set $\lambda := \frac{\langle x_{\mathrm{old}},x_{\mathrm{new}}\rangle}{\|x_{\mathrm{new}}\|_2^2}>0$.
Suppose $\eta_j=\Theta(j^{-\nu})$ for some $\nu\in(0,1/2)$.
Then, for arbitrary pre-new-task old-path
masses, with probability at least $1-\delta$,
\begin{gather}
   \Pi(\mathbf W_{\mathrm{new}}^{(n)})
   \leq
   \underbrace{\mathfrak C_{\mathrm{old}}(1-\theta)}_{\substack{
      \text{\normalfont{upper bound on}}\\
      \text{\normalfont{surviving knowledge}}}}
   +\widetilde O\!\left(\theta|\mathcal P_{\mathrm{old}}|
   |\mathcal P_{\mathrm{new}}|^\lambda
   \left[
   \frac{1+B^{-1}\log(1/\delta)}{n^{\nu}} 
   \right]^\lambda
   \right),\label{eq:theorem-four-late-upper}\\
   \Pi(\mathbf W_{\mathrm{new}}^{(0)})
   -\mathfrak C_{\mathrm{old}}(\theta)
   \leq
   \Pi(\mathbf W_{\mathrm{new}}^{(n)}),\label{eq:theorem-four-late-lower}
\end{gather}
where $\Pi(\mathbf{W}) = \pi_{\mathbf{W}}(Y_T \in \tilde C_{\mathrm{old}} \mid Y_0 = y_{\mathrm{old},0}, x_{\mathrm{old}})$ and
\begin{align}
   \mathfrak C_{\mathrm{old}}(\phi) := \max_{\substack{\mathcal P \subseteq \mathcal P_{\mathrm{old}}\\|\mathcal P| \leq \phi|\mathcal P_{\mathrm{old}}|}}
   \sum_{\gamma \in \mathcal P} \pi_{\mathbf W_{\mathrm{new}}^{(0)}}(Y_{0:T} = \gamma \mid Y_0 = y_{\mathrm{old},0}, x_{\mathrm{old}}).
\end{align}
$\tilde{O}(\cdot)$ hides logarithmic factors in $n$ and $|\mathcal P_{\mathrm{new}}|$. 
\end{restatable}
During sequential learning, Theorem~\ref{thm:sft-forgetting} shows that SFT drives the probability mass of old correct paths
that conflict with the new task toward zero as training proceeds (Equation~\ref{eq:theorem-four-late-upper}).
Despite some forgetting, SFT retains previously acquired knowledge that does not conflict with the new knowledge (Equation~\ref{eq:theorem-four-late-lower}).
Although our current setting excludes conflicts between tasks at inter-cluster transitions,
this result may help explain why SD could remain less prone to forgetting than SFT in broader
settings that allow such conflicts.
Whereas SFT updates transitions along teacher paths, SD restricts its updates only to source nodes selected by hints, which may leave more old correct paths unaffected by subsequent training.
If these paths account for most of the old task's success probability at the start of new task training, their preserved solution paths would alleviate forgetting.

\begin{restatable}[Explicit representation of $\mathfrak C_{\mathrm{old}}(\phi)$]{remark}{oldconcentrationremark}
\label{rem:old-concentration}
Let $\kappa_{\max}^{\mathrm{old}}$ and $\kappa_{\min}^{\mathrm{old}}$ be the
largest and smallest old-path masses normalized by their average mass, as
defined in Appendix~\ref{app:sft-forgetting}. Then
\begin{align}
   \mathfrak C_{\mathrm{old}}(\phi)
   &\le
   \Pi(\mathbf W_{\mathrm{new}}^{(0)})
   \min\left\{
      \kappa_{\max}^{\mathrm{old}}\phi,
      1-\kappa_{\min}^{\mathrm{old}} + \kappa_{\min}^{\mathrm{old}}\phi
   \right\}.
   \label{eq:theorem-four-explicit-theta-bound}
\end{align}
\end{restatable}

\begin{remark}[Error accumulation across tasks]
\label{rem:task-accumulation}
For multiple later tasks, the union of their conflicted old-path sets can
increase, and the fraction of non-conflicted old correct paths can decrease. 
Consequently, $\theta_{\mathrm{total}}$, the multi-task counterpart of $\theta$, can increase as the number of tasks
grows, thereby exacerbating the forgetting effect and
further decreasing the corresponding concentration term in
the success probability of the old tasks.
This suggests an asymptotic upper bound of $\mathfrak{C}_{\mathrm{old}}(1 - \theta_{\mathrm{total}})$ on the old-task success probability.
\end{remark}

\subsection{A theoretical implication for pre-training diversity}
\label{sec:pretraining-reliability}
The preceding results concern new-task learning and retention from each task's original start.
As a complementary theoretical discussion, we now ask when the learned policy can also reason
successfully from intermediate states. The role of the pre-trained model in SD
\citep{shenfeld2026sdft, hubotter2026sdpo} and the benefits of diverse reasoning paths
\citep{wang2023selfconsistency} motivate this question. We examine diversity through the
coverage of reasoning states in the pre-trained policy: because SD updates depend on visited
states, limited coverage can leave some transitions unlearned that are needed for intermediate restarts.
We fix a task $\mathcal{D}_{\ell}$ and run the same SD procedure from its original start; restarts are used only for evaluation.
Under the coverage condition of Assumption~\ref{ass:pretrained-diversity}, every correct boundary receives training opportunities and they suffice to 
support all relevant starts simultaneously (Proposition~\ref{prop:restart-positive} in Appendix~\ref{app:restart}).
However, if the pre-trained model is not diverse enough, some relevant starts may not pass through proper cluster transitions and
thus the model cannot generate a correct path from them.

\begin{assumption}[Existence of a rare entrance]
\label{ass:restart-rare}
   For the fixed task, retain Assumption~\ref{ass:pretrained-diversity} except its coverage lower bound at one correct boundary node $u^*$;
   {\color{black}that is, we change only the coverage condition represented by Equation~\eqref{eq:pretrained-boundary-coverage} in Assumption~\ref{ass:pretrained-diversity}.}
   The pretrained probability of ever visiting $u^*$ from the original start is $O(1/M')$, with $M'=\Omega(M)$ and a uniform implied constant.
   Its correct transition is the only incoming edge to a nonterminal cluster $C_{k}$ that does not contain the original start $y_{\ell, 0}$.
   {\color{black}This means some boundary nodes may remain unvisited during training, leaving their outgoing clusters untraversed.}
\end{assumption}

\begin{restatable}[Restart failure beyond a rare entrance]{proposition}{restartnegativeproposition}
Suppose\label{prop:restart-negative} Assumptions~\ref{ass:transition-compatibility} and~\ref{ass:restart-rare} hold,
$\eta_j=\Theta(j^{-\nu})$ for fixed $\nu\in [0, 1/2)$, and
$0<\epsilon_1,\epsilon_2,\epsilon_{\mathrm{rel}},\delta<1$.
If
\begin{gather}
n=\Omega\!\left(
\max\left\{\frac{\log(2KM/\delta)}{B\underline\rho_{\ell,K, M}^{\,2}}\right.,
\left.\left[\frac{\epsilon_1^{-1}-1}{\epsilon_2^2\log(1/\epsilon_{\mathrm{rel}})\underline\rho_{\ell,K, M}}\right]^{1/(1-\nu)}\right\}\right),
\label{eq:restart-explicit-n}\\
M'=\Omega\!\left(\frac1{\epsilon_2^{|E_{\ell}|}}
\max\!\left\{\frac1{KM\epsilon_1},\frac{Bn}{\delta}\right\}\right),
\label{eq:restart-budget}
\end{gather}
with sufficiently large uniform constants, then, with probability at least
$1-\delta$, no training rollout visits $u^*$, and jointly
\begin{gather}
\pi_{\mathbf W^{(n)}}(y_T\in\tilde C_\ell\mid y_{\ell,0},x_\ell)
= 1-O(KM\epsilon_1)\quad (\text{\normalfont{Misleading high accuracy}}), \label{eq:restart-original}\\
\sup_{s\in\mathcal S_{\ell,k}}
\pi_{\mathbf W^{(n)}}(y_T\in\tilde C_\ell\mid y_0=s,x_\ell)
=O(M\epsilon_2)\quad (\text{\normalfont{Restart failure}}). 
\label{eq:restart-restart-low}
\end{gather}
Here $\mathcal S_{\ell,k}$ contains the relevant starts in $C_k$ and
$\underline\rho_{\ell,K, M}$ is specified in
Appendix~\ref{app:sd-transitions}.
\end{restatable}

Proposition~\ref{prop:restart-negative} states that if $C_k$ is a difficult cluster for the student to reach and the amount of training is insufficient,
the student may still have a low success probability when restarting from relevant intermediate states in $C_k$
even though the model seems to have highly learned the task from its original start $y_{\ell, 0}$ (Equation~\ref{eq:restart-original}).
Rare entrances satisfying Assumption~\ref{ass:restart-rare} make some clusters severely difficult to reach from the original start,
and thus leave these clusters unlearned, preventing the success probability from increasing. To overcome this issue, 
self-distillation requires an amount of training that grows at least linearly with $M'$.
The gap in success probabilities between the original start and the downstream cluster suggests that
limited diversity can pose risks at inference time that are difficult to detect during training. 
Improved performance from the original start may not translate into 
better performance when restarting from downstream clusters.
Pretrained coverage therefore supports intermediate restartability beyond original-start
accuracy. 

\begin{restatable}[Separation of training times]{remark}{restarttimeremark}
\label{rem:restart-time}
Under Assumptions~\ref{ass:transition-compatibility} and~\ref{ass:pretrained-diversity}, the model trained by SD can resume from any relevant start at the same training length \Eqref{eq:restart-explicit-n}
(Proposition~\ref{prop:restart-positive} in Appendix~\ref{app:restart}.)
Holding other quantities fixed, the nonvisit horizon grows linearly with
$M'$.
If the nonvisit horizon grows asymptotically faster than the sufficient learning scale, there is a range of training lengths for which the two settings exhibit different restart behavior. Under the full-coverage condition,
the trained model can reason successfully from every relevant start, whereas under the rare-entrance condition, the cluster beyond the rare entrance remains unlearned.
\end{restatable}
\vspace{-1em}
\begin{figure}[t]
   \centering
   \includegraphics[width=1.\textwidth]{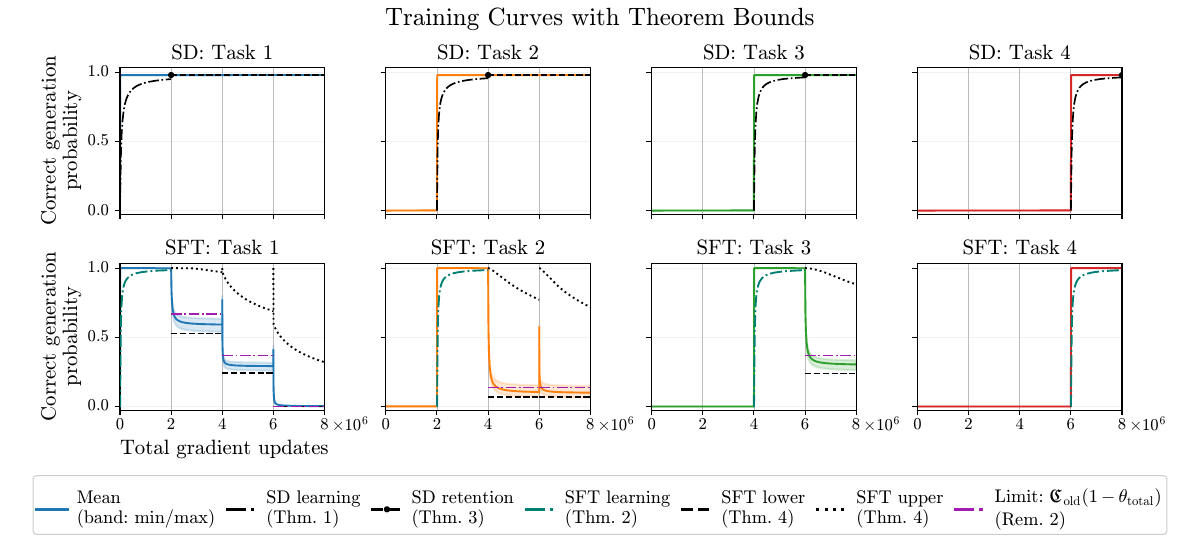}
   \setlength{\abovecaptionskip}{0pt}
   \caption{Results of \textbf{(A)} in Section~4. Training dynamics under SD (top row) and
   SFT (bottom row) for Tasks 1 to 4.
   Solid curves show the mean correct-generation probability over ten seeds,
   and shaded regions indicate the minimum--maximum range.
   Dash-dotted and dashed black curves represent the SD population learning
   references and conditional retention lower bounds.
   Dashed and dotted black curves represent the SFT conditional retention
   lower bounds and population upper references, respectively.
   Purple dash-dotted lines show the concentration-profile term
   $\mathfrak{C}_{\mathrm{old}}(1-\theta_{\mathrm{total}})$.
   Vertical gray lines indicate task transitions.\vspace{-1.em}}
   \label{fig:training-dynamics}
\end{figure}
\section{Numerical experiments}
\vspace{-1em}
We have conducted two experiments \textbf{(A)} \textit{to validate our theoretical results} and \textbf{(B)} \textit{to confirm our intuition and theories
are helpful to explain the real-world learning dynamics.}
\vspace{-.75em}
\label{sec:numerical-experiments}
\paragraph{\textbf{(A)} Validation of theoretical predictions}
{\color{black} In this section, we would like to evaluate whether the theoretical bounds are effective.
For this purpose, \textbf{we design a synthetic CoT graph} that satisfies Assumptions~\ref{ass:transition-compatibility} and~\ref{ass:pretrained-diversity}.
Figure~\ref{fig:training-dynamics} shows training curves for every task between SD and SFT in continual learning. 
Following results confirm that tracing the sparse updates in SD and the conflict between tasks in SFT can explain the observed retention and forgetting.
\vspace{-.75em}
\subparagraph{Experimental setup}
The graph is a layered DAG and consists of $K = 13$ clusters. 
The number of tasks is $L = 4$, and each task has exactly 10 correct paths leading to a unique correct terminal cluster. 
Task~1 to 4 is different cluster visiting tasks. Each goal is to visit their target cluster respectively, so \textbf{the model is required
to reach the target terminal cluster for each task} through continual learning.
The tasks are designed to have overlapping paths to induce potential forgetting in SFT.
We train the softmax predictor with both SD and SFT for $2 \times 10^6$ iterations on each task to improve its generation performance.
The more detailed information is written in Appendix~\ref{app:numerical}.}
\vspace{-.75em}
\subparagraph{SD results} {\color{black}About self-distillation panels, SD allows the model to retain its previous knowledge without forgetting.
The black dashed and dash-dotted curves represent the theoretical lower bounds for self-distillation given in Theorems 1 and 3.
They successfully lower bound the dynamics curves of self-distillation, learning sparse structure and knowledge retention.}
\vspace{-.75em}
\subparagraph{SFT results} {\color{black}SFT exhibits a significant drop in performance on task 1 to 3.
The black dashed lines in the SFT panels show conditional retention lower bounds, and the black dotted curves show the  population upper references.
They successfully bound the SFT sharp forgetting curves. Especially, the limit of upper bounds, purple lines, well capture the convergence of SFT.
while SD preserves each task's performance during subsequent-task training.} 

\begin{figure}[t]
   \centering
   \includegraphics[width=1.\textwidth]{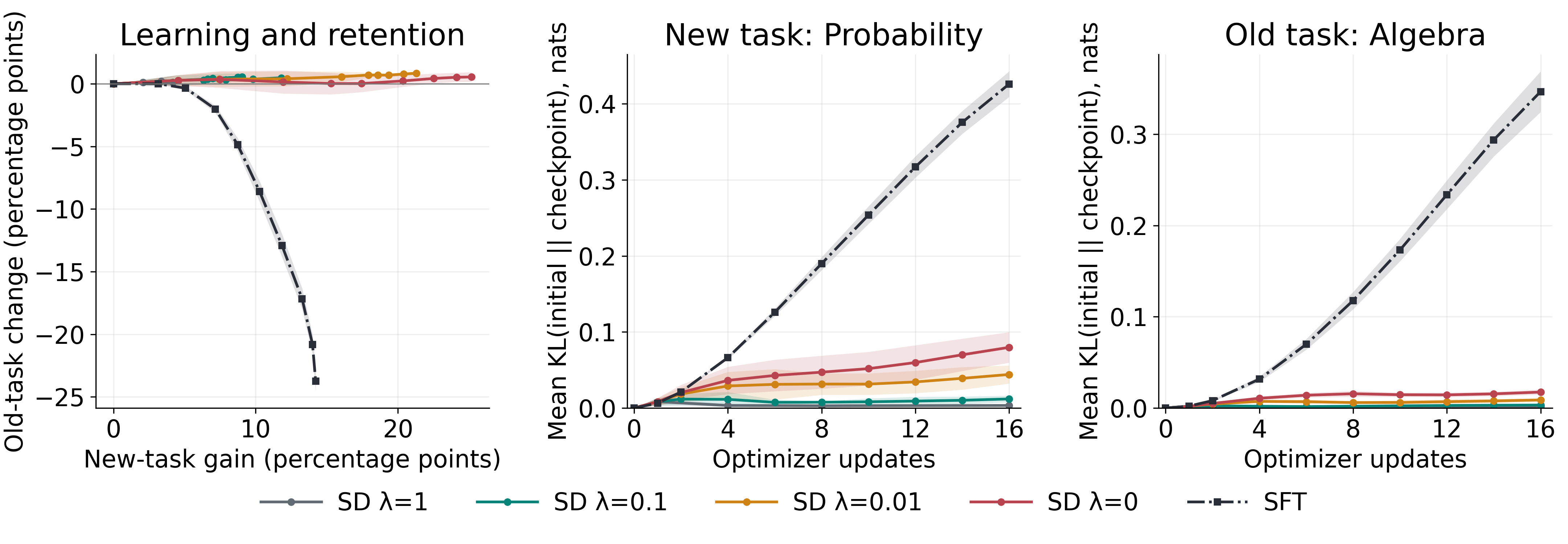}
   \setlength{\abovecaptionskip}{0pt}
   \caption{{\color{black}
      Results of \textbf{(B)} in Section~4. 
      Each experiment has been conducted using Qwen3-1.7B with four different random seeds. All curves and bunds are the average and
      the standard deviation of results.
      The parameter \(\lambda\) controls gradient localization in self-distillation. 
      The CoT graph is constructed from two mathematical tasks: algebra and probability. 
      \textbf{The model learns transitions on this graph through sequential training from algebra to probability.}
      The left panel shows that SD learns the new task while preserving algebra performance, whereas SFT exhibits forgetting as it learns probability. 
      The middle and right panels show KL divergence at unhinted nodes in self-distillation, indicating that hint-based learning localizes changes in transition distributions.
   }\vspace{-1em}
   }
   \label{fig:real-world dynamics}
\end{figure}
\vspace{-.75em}
{\color{black}
\paragraph{\textbf{(B)} Empirical support for the proposed mechanism}
This section is motivated to explain whether our theoretical modelling is effective to understand real-world learning
dynamics.
\vspace{-.75em}
\subparagraph{Experimetal setup} 
Following graph-construction methods \citep{wang2024understanding,minegishi2025topology}, we construct a CoT graph from Qwen3-1.7B trajectories from algebra to probability problems and train the same model.
During self-distillation, each hint modifies the teacher distribution only at an erroneous branching state, promoting one alternative edge that preserves goal reachability.
For stability, we use an EMA teacher and cap the KL divergence between its hinted and unhinted distributions at 0.1 nats, following general teacher-stabilization principles \citep{hubotter2026sdpo,shenfeld2026sdft}.
For SFT, we remove cycles from successful demonstrations and supervise all remaining transitions from the start to the correct goal.
Details appear in Appendix~\ref{app:numerical}.
\vspace{-.75em}
\subparagraph{Results} 
Our experiments mainly show the following two points: (i) conflict with previous tasks could intensify forgetting, and (ii) hint-based fine-tuning localizes changes in model distributions. In our experiments, the dataset tends to provide conflicting teacher paths, and this seems to significantly affect the results, as the left panel of Figure~\ref{fig:real-world dynamics} shows a sharp drop in previous-task performance. Furthermore, the other panels show low KL divergence for SD but high divergence for SFT, which would support our intuition that localized hints could curb model changes locally.
}
\section{Conclusion}
We analyze continual learning under SD and SFT using a CoT reasoning graph.
Our theory shows how the sparsity of hint information localizes SD updates to preserve previously
acquired capabilities, whereas dense SFT updates along the teacher paths can cause forgetting.
As a complementary theoretical implication, we show that limited coverage in the pre-trained
policy can prevent successful reasoning from intermediate states even after SD succeeds from
the original start.
\vspace{-.8em}
\paragraph{Future work.}
The major improvement is to extend the analysis from the current outcome reward format such as Equation~\ref{eq:path-level-hint-feedback} to that of the process reward. 
As shown in \citet{hubotter2026sdpo}, the self-teacher in on-policy self-distillation can act as an implicit process reward model through retrospection when provided with rich feedback.
This setting might also be helpful to analyze the
importance of in-context learning ability in the pre-training strategy.

\section*{Acknowledgements}
SU was partially supported by JSPS KAKENHI (24K02905). TS was partially supported by JST CREST (JPMJCR2115). This work was partially supported by JST ERATO (JPMJER2601).
This research is supported by the National Research Foundation, Singapore and the Ministry of Digital Development and Information under the AI Visiting Professorship Programme (award number AIVP-2024-004). Any opinions, findings and conclusions or recommendations expressed in this material are those of the author(s) and do not reflect the views of National Research Foundation, Singapore and the Ministry of Digital Development and Information.

\section*{AI use statement}
We used AI, mainly ChatGPT, to assist in the following ways.
\begin{itemize}
   \item To aid and polish writing: We used AI to improve the clarity and grammar of the text,
   and to help us express our ideas more clearly.
   \item For retrieval of relevant references and related work: Mainly, we used AI to find and summarize
   relevant works and we incorporated them into this paper.
   \item For research ideation: We used AI to brainstorm ideas and to discuss the research direction.
   Assumption~\ref{ass:transition-compatibility}, Equations~\ref{eq:pretrained-boundary-error} and~\ref{eq:pretrained-boundary-coverage} in Assumption~\ref{ass:pretrained-diversity} and Section~\ref{sec:pretraining-reliability}
   is constructed during the discussion with AI.
   \item For the purpose of drafting sections of all contents in the paper: We made AI-generated drafts
   at all sections, which were then revised and modified repeatedly.
   \item For numerical experiments: We made AI create some code and the graph we used to validate our theoretical results.
   These codes and graphs were revised and modified by us to fit our purpose.
   \item For proving mathematical claims: AI helped us to check the correctness of the proofs and
   formalized our ideas to mathematical statements and provided proofs.
   In particular, the following is a list of the major works of AI in the proofs.
   \begin{itemize}
      \item Lemma~\ref{lem:restart-coverage}: AI helped us to evaluate the times of visitation to the correct boundaries.
      \item Lemma~\ref{lem:sft-potential}: AI helped us to define the potential function in \Eqref{eq:analysis-potential} and then incorporated it into our original
      proof to produce the current version.
      \item Theorem~\ref{thm:sft-forgetting}: AI helped us to prove the bounds in Theorem~\ref{thm:sft-forgetting}.
   \end{itemize}
\end{itemize}
On the other hand, we did not use AI at all in the following points.
\begin{itemize}
   \item \textbf{To create intuitive explanations of the results}: Theoretical frameworks like problem setting are not created by AI,
   but is constructed only based on our original ideas and previous related work.
   \item \textbf{To create the base ideas for proofs}: All mathematical claims such as theorems and propositions are based on our original proofs.
   Especially, in Theorems~\ref{thm:sd-correct-rollout} and~\ref{thm:sd-no-forgetting} and Proposition~\ref{prop:restart-negative}, no AI solutions are used to construct them.
\end{itemize}

\section*{Reproducibility}
The training and evaluation code, configuration files, fixed inputs, and saved results for experiments
\textbf{(A)} and \textbf{(B)} are available in the
\href{https://github.com/yuukiizayoi/continual-reasoning-experiments}{\texttt{continual-reasoning-experiments} GitHub repository}.\footnote{\url{https://github.com/yuukiizayoi/continual-reasoning-experiments}}
The directories \texttt{experiment\_a\_synthetic} and \texttt{experiment\_b\_llm} contain experiments \textbf{(A)} and \textbf{(B)}, respectively.
Each directory's README provides instructions for rerunning training and regenerating the figures and tables from the saved results.
Experimental settings and results are detailed in Appendix~\ref{app:numerical}.

\bibliography{iclr2027_conference,empirical_appendix/references}

@inproceedings{hendrycks2021math,
  title = {{Measuring Mathematical Problem Solving With the MATH Dataset}},
  author = {Dan Hendrycks and Collin Burns and Saurav Kadavath and Akul Arora and Steven Basart and Eric Tang and Dawn Song and Jacob Steinhardt},
  year = {2021},
  url = {https://datasets-benchmarks-proceedings.neurips.cc/paper_files/paper/2021/hash/be83ab3ecd0db773eb2dc1b0a17836a1-Abstract-round2.html},
  booktitle = {Proceedings of the Neural Information Processing Systems Track on Datasets and Benchmarks}
}

@inproceedings{schaul2015uvfa,
  title = {{Universal Value Function Approximators}},
  author = {Tom Schaul and Daniel Horgan and Karol Gregor and David Silver},
  year = {2015},
  url = {https://proceedings.mlr.press/v37/schaul15.html},
  booktitle = {International Conference on Machine Learning}
}

@inproceedings{abdolmaleki2018mpo,
  title = {{Maximum a Posteriori Policy Optimisation}},
  author = {Abbas Abdolmaleki and Jost Tobias Springenberg and Yuval Tassa and Remi Munos and Nicolas Heess and Martin Riedmiller},
  year = {2018},
  url = {https://openreview.net/pdf?id=S1ANxQW0b},
  booktitle = {International Conference on Learning Representations}
}

@misc{yang2025qwen3,
  title = {{Qwen3 Technical Report}},
  author = {An Yang and Anfeng Li and Baosong Yang and Beichen Zhang and Binyuan Hui and others},
  year = {2025},
  url = {https://arxiv.org/abs/2505.09388},
  note = {Technical report or software; accessed 2026-11-10},
  eprint = {2505.09388},
  archivePrefix = {arXiv}
}

@misc{hanna2023doesgpt2computegreaterthan,
      title={How does GPT-2 compute greater-than?: Interpreting mathematical abilities in a pre-trained language model}, 
      author={Michael Hanna and Ollie Liu and Alexandre Variengien},
      year={2023},
      eprint={2305.00586},
      archivePrefix={arXiv},
      primaryClass={cs.CL},
      url={https://arxiv.org/abs/2305.00586}, 
}

@misc{wang2024understanding,
      title={Understanding Reasoning Ability of Language Models From the Perspective of Reasoning Paths Aggregation}, 
      author={Xinyi Wang and Alfonso Amayuelas and Kexun Zhang and Liangming Pan and Wenhu Chen and William Yang Wang},
      year={2024},
      eprint={2402.03268},
      archivePrefix={arXiv},
      primaryClass={cs.LG},
      url={https://arxiv.org/abs/2402.03268}, 
}

@misc{minegishi2025topology,
      title={Topology of Reasoning: Understanding Large Reasoning Models through Reasoning Graph Properties}, 
      author={Gouki Minegishi and Hiroki Furuta and Takeshi Kojima and Yusuke Iwasawa and Yutaka Matsuo},
      year={2025},
      eprint={2506.05744},
      archivePrefix={arXiv},
      primaryClass={cs.AI},
      url={https://arxiv.org/abs/2506.05744}, 
}

@misc{merullo2024circuitcomponentreusetasks,
      title={Circuit Component Reuse Across Tasks in Transformer Language Models}, 
      author={Jack Merullo and Carsten Eickhoff and Ellie Pavlick},
      year={2024},
      eprint={2310.08744},
      archivePrefix={arXiv},
      primaryClass={cs.CL},
      url={https://arxiv.org/abs/2310.08744}, 
}

@misc{yang2026selfdistilledrlvr,
      title={Self-Distilled {RLVR}},
      author={Chenxu Yang and Chuanyu Qin and Qingyi Si and Minghui Chen and Naibin Gu and Dingyu Yao and Zheng Lin and Weiping Wang and Jiaqi Wang and Nan Duan},
      year={2026},
      eprint={2604.03128},
      archivePrefix={arXiv},
      primaryClass={cs.LG},
      url={https://arxiv.org/abs/2604.03128}, 
}

@misc{yao2025diversity,
      title={Diversity-Aware Policy Optimization for Large Language Model Reasoning}, 
      author={Jian Yao and Ran Cheng and Xingyu Wu and Jibin Wu and Kay Chen Tan},
      year={2025},
      eprint={2505.23433},
      archivePrefix={arXiv},
      primaryClass={cs.LG},
      url={https://arxiv.org/abs/2505.23433}, 
}

@misc{wang2026statisticallimitsselfimprovingagents,
      title={On The Statistical Limits of Self-Improving Agents}, 
      author={Charles L. Wang and Keir Dorchen and Peter Jin},
      year={2026},
      eprint={2510.04399},
      archivePrefix={arXiv},
      primaryClass={cs.AI},
      url={https://arxiv.org/abs/2510.04399}, 
}

@misc{luo2025wizardcoderempoweringcodelarge,
      title={{WizardCoder}: Empowering Code Large Language Models with {Evol-Instruct}},
      author={Ziyang Luo and Can Xu and Pu Zhao and Qingfeng Sun and Xiubo Geng and Wenxiang Hu and Chongyang Tao and Jing Ma and Qingwei Lin and Daxin Jiang},
      year={2025},
      eprint={2306.08568},
      archivePrefix={arXiv},
      primaryClass={cs.CL},
      url={https://arxiv.org/abs/2306.08568}, 
}

@misc{zhang2024raftadaptinglanguagemodel,
      title={{RAFT}: Adapting Language Model to Domain Specific {RAG}},
      author={Tianjun Zhang and Shishir G. Patil and Naman Jain and Sheng Shen and Matei Zaharia and Ion Stoica and Joseph E. Gonzalez},
      year={2024},
      eprint={2403.10131},
      archivePrefix={arXiv},
      primaryClass={cs.CL},
      url={https://arxiv.org/abs/2403.10131}, 
}

@inproceedings{zeng-etal-2025-acecoder,
    title = "{ACECODER}: Acing Coder {RL} via Automated Test-Case Synthesis",
    author = "Zeng, Huaye  and
      Jiang, Dongfu  and
      Wang, Haozhe  and
      Nie, Ping  and
      Chen, Xiaotong  and
      Chen, Wenhu",
    editor = "Che, Wanxiang  and
      Nabende, Joyce  and
      Shutova, Ekaterina  and
      Pilehvar, Mohammad Taher",
    booktitle = "Proceedings of the 63rd Annual Meeting of the Association for Computational Linguistics (Volume 1: Long Papers)",
    month = jul,
    year = "2025",
    address = "Vienna, Austria",
    publisher = "Association for Computational Linguistics",
    url = "https://aclanthology.org/2025.acl-long.587/",
    doi = "10.18653/v1/2025.acl-long.587",
    pages = "12023--12040",
    ISBN = "979-8-89176-251-0"
}

@misc{jin2025searchr1trainingllmsreason,
      title={{Search-R1}: Training {LLMs} to Reason and Leverage Search Engines with Reinforcement Learning},
      author={Bowen Jin and Hansi Zeng and Zhenrui Yue and Jinsung Yoon and Sercan Arik and Dong Wang and Hamed Zamani and Jiawei Han},
      year={2025},
      eprint={2503.09516},
      archivePrefix={arXiv},
      primaryClass={cs.CL},
      url={https://arxiv.org/abs/2503.09516}, 
}

@article{Guo_2025,
   title={{DeepSeek-R1} incentivizes reasoning in {LLMs} through reinforcement learning},
   volume={645},
   ISSN={1476-4687},
   url={http://dx.doi.org/10.1038/s41586-025-09422-z},
   DOI={10.1038/s41586-025-09422-z},
   number={8081},
   journal={Nature},
   publisher={Springer Science and Business Media LLC},
   author={Guo, Daya and Yang, Dejian and Zhang, Haowei and Song, Junxiao and Wang, Peiyi and Zhu, Qihao and Xu, Runxin and Zhang, Ruoyu and Ma, Shirong and Bi, Xiao and Zhang, Xiaokang and Yu, Xingkai and Wu, Yu and Wu, Z. F. and Gou, Zhibin and Shao, Zhihong and Li, Zhuoshu and Gao, Ziyi and Liu, Aixin and Xue, Bing and Wang, Bingxuan and Wu, Bochao and Feng, Bei and Lu, Chengda and Zhao, Chenggang and Deng, Chengqi and Ruan, Chong and Dai, Damai and Chen, Deli and Ji, Dongjie and Li, Erhang and Lin, Fangyun and Dai, Fucong and Luo, Fuli and Hao, Guangbo and Chen, Guanting and Li, Guowei and Zhang, H. and Xu, Hanwei and Ding, Honghui and Gao, Huazuo and Qu, Hui and Li, Hui and Guo, Jianzhong and Li, Jiashi and Chen, Jingchang and Yuan, Jingyang and Tu, Jinhao and Qiu, Junjie and Li, Junlong and Cai, J. L. and Ni, Jiaqi and Liang, Jian and Chen, Jin and Dong, Kai and Hu, Kai and You, Kaichao and Gao, Kaige and Guan, Kang and Huang, Kexin and Yu, Kuai and Wang, Lean and Zhang, Lecong and Zhao, Liang and Wang, Litong and Zhang, Liyue and Xu, Lei and Xia, Leyi and Zhang, Mingchuan and Zhang, Minghua and Tang, Minghui and Zhou, Mingxu and Li, Meng and Wang, Miaojun and Li, Mingming and Tian, Ning and Huang, Panpan and Zhang, Peng and Wang, Qiancheng and Chen, Qinyu and Du, Qiushi and Ge, Ruiqi and Zhang, Ruisong and Pan, Ruizhe and Wang, Runji and Chen, R. J. and Jin, R. L. and Chen, Ruyi and Lu, Shanghao and Zhou, Shangyan and Chen, Shanhuang and Ye, Shengfeng and Wang, Shiyu and Yu, Shuiping and Zhou, Shunfeng and Pan, Shuting and Li, S. S. and Zhou, Shuang and Wu, Shaoqing and Yun, Tao and Pei, Tian and Sun, Tianyu and Wang, T. and Zeng, Wangding and Liu, Wen and Liang, Wenfeng and Gao, Wenjun and Yu, Wenqin and Zhang, Wentao and Xiao, W. L. and An, Wei and Liu, Xiaodong and Wang, Xiaohan and Chen, Xiaokang and Nie, Xiaotao and Cheng, Xin and Liu, Xin and Xie, Xin and Liu, Xingchao and Yang, Xinyu and Li, Xinyuan and Su, Xuecheng and Lin, Xuheng and Li, X. Q. and Jin, Xiangyue and Shen, Xiaojin and Chen, Xiaosha and Sun, Xiaowen and Wang, Xiaoxiang and Song, Xinnan and Zhou, Xinyi and Wang, Xianzu and Shan, Xinxia and Li, Y. K. and Wang, Y. Q. and Wei, Y. X. and Zhang, Yang and Xu, Yanhong and Li, Yao and Zhao, Yao and Sun, Yaofeng and Wang, Yaohui and Yu, Yi and Zhang, Yichao and Shi, Yifan and Xiong, Yiliang and He, Ying and Piao, Yishi and Wang, Yisong and Tan, Yixuan and Ma, Yiyang and Liu, Yiyuan and Guo, Yongqiang and Ou, Yuan and Wang, Yuduan and Gong, Yue and Zou, Yuheng and He, Yujia and Xiong, Yunfan and Luo, Yuxiang and You, Yuxiang and Liu, Yuxuan and Zhou, Yuyang and Zhu, Y. X. and Huang, Yanping and Li, Yaohui and Zheng, Yi and Zhu, Yuchen and Ma, Yunxian and Tang, Ying and Zha, Yukun and Yan, Yuting and Ren, Z. Z. and Ren, Zehui and Sha, Zhangli and Fu, Zhe and Xu, Zhean and Xie, Zhenda and Zhang, Zhengyan and Hao, Zhewen and Ma, Zhicheng and Yan, Zhigang and Wu, Zhiyu and Gu, Zihui and Zhu, Zijia and Liu, Zijun and Li, Zilin and Xie, Ziwei and Song, Ziyang and Pan, Zizheng and Huang, Zhen and Xu, Zhipeng and Zhang, Zhongyu and Zhang, Zhen},
   year={2025},
   month=sep, pages={633–638} }

@misc{chung2022scalinginstructionfinetunedlanguagemodels,
      title={Scaling Instruction-Finetuned Language Models}, 
      author={Hyung Won Chung and Le Hou and Shayne Longpre and Barret Zoph and Yi Tay and William Fedus and Yunxuan Li and Xuezhi Wang and Mostafa Dehghani and Siddhartha Brahma and Albert Webson and Shixiang Shane Gu and Zhuyun Dai and Mirac Suzgun and Xinyun Chen and Aakanksha Chowdhery and Alex Castro-Ros and Marie Pellat and Kevin Robinson and Dasha Valter and Sharan Narang and Gaurav Mishra and Adams Yu and Vincent Zhao and Yanping Huang and Andrew Dai and Hongkun Yu and Slav Petrov and Ed H. Chi and Jeff Dean and Jacob Devlin and Adam Roberts and Denny Zhou and Quoc V. Le and Jason Wei},
      year={2022},
      eprint={2210.11416},
      archivePrefix={arXiv},
      primaryClass={cs.LG},
      url={https://arxiv.org/abs/2210.11416}, 
}

@misc{sang2026crispcompressedreasoningiterative,
      title={{CRISP}: Compressed Reasoning via Iterative Self-Policy Distillation},
      author={Hejian Sang and Yuanda Xu and Zhengze Zhou and Ran He and Zhipeng Wang and Jiachen Sun},
      year={2026},
      eprint={2603.05433},
      archivePrefix={arXiv},
      primaryClass={cs.LG},
      url={https://arxiv.org/abs/2603.05433}, 
}

@misc{balasubramanian2026quantitativecharacterizationforgettingposttraining,
      title={A Quantitative Characterization of Forgetting in Post-Training}, 
      author={Krishnakumar Balasubramanian and Shiva Prasad Kasiviswanathan},
      year={2026},
      eprint={2603.12163},
      archivePrefix={arXiv},
      primaryClass={cs.LG},
      url={https://arxiv.org/abs/2603.12163}, 
}

@misc{kim2026doesselfdistillationsometimesdegrade,
      title={Why Does Self-Distillation (Sometimes) Degrade the Reasoning Capability of {LLMs}?},
      author={Jeonghye Kim and Xufang Luo and Minbeom Kim and Sangmook Lee and Dohyung Kim and Jiwon Jeon and Dongsheng Li and Yuqing Yang},
      year={2026},
      eprint={2603.24472},
      archivePrefix={arXiv},
      primaryClass={cs.CL},
      url={https://arxiv.org/abs/2603.24472}, 
}

@misc{yang2024selfdistillationbridgesdistributiongap,
      title={Self-Distillation Bridges Distribution Gap in Language Model Fine-Tuning}, 
      author={Zhaorui Yang and Tianyu Pang and Haozhe Feng and Han Wang and Wei Chen and Minfeng Zhu and Qian Liu},
      year={2024},
      eprint={2402.13669},
      archivePrefix={arXiv},
      primaryClass={cs.CL},
      url={https://arxiv.org/abs/2402.13669}, 
}

@misc{abbe2024fartransformersreasonglobality,
      title={How Far Can Transformers Reason? The Globality Barrier and Inductive Scratchpad}, 
      author={Emmanuel Abbe and Samy Bengio and Aryo Lotfi and Colin Sandon and Omid Saremi},
      year={2024},
      eprint={2406.06467},
      archivePrefix={arXiv},
      primaryClass={cs.LG},
      url={https://arxiv.org/abs/2406.06467}, 
}

@inproceedings{NEURIPS2025_a8f12348,
 author = {Yang, Tong and Huang, Yu and Liang, Yingbin and Chi, Yuejie},
 booktitle = {Advances in Neural Information Processing Systems},
 doi = {10.52202/085713-3875},
 editor = {D. Belgrave and C. Zhang and H. Lin and R. Pascanu and P. Koniusz and M. Ghassemi and N. Chen},
 pages = {116194--116285},
 publisher = {Curran Associates, Inc.},
 title = {Multi-head Transformers Provably Learn Symbolic Multi-step Reasoning via Gradient Descent},
 url = {https://proceedings.neurips.cc/paper_files/paper/2025/file/a8f123480ede36612314d013aba6340c-Paper-Conference.pdf},
 volume = {38, Main Conference},
 year = {2025}
}

@misc{wang2026entropypreservingsupervisedfinetuningadaptive,
      title={Entropy-Preserving Supervised Fine-Tuning via Adaptive Self-Distillation for Large Reasoning Models}, 
      author={Hao Wang and Hao Gu and Hongming Piao and Kaixiong Gong and Yuxiao Ye and Xiangyu Yue and Sirui Han and Yike Guo and Dapeng Wu},
      year={2026},
      eprint={2602.02244},
      archivePrefix={arXiv},
      primaryClass={cs.LG},
      url={https://arxiv.org/abs/2602.02244}, 
}

@inproceedings{liu2024lastiterate,
title={Revisiting the Last-Iterate Convergence of Stochastic Gradient Methods},
author={Liu, Zijian and Zhou, Zhengyuan},
booktitle={The Twelfth International Conference on Learning Representations},
year={2024},
url={https://openreview.net/forum?id=xxaEhwC1I4}
}

@inproceedings{agarwal2024onpolicy,
  title = {{On-Policy Distillation of Language Models: Learning from Self-Generated Mistakes}},
  author = {Agarwal, Rishabh and Vieillard, Nino and Zhou, Yongchao and Stanczyk, Piotr and Ramos, Sabela and Geist, Matthieu and Bachem, Olivier},
  year = {2024},
  booktitle = {International Conference on Learning Representations},
  url = {https://arxiv.org/abs/2306.13649}
}

@inproceedings{lopezpaz2017gem,
  title = {{Gradient Episodic Memory for Continual Learning}},
  author = {Lopez-Paz, David and Ranzato, Marc'Aurelio},
  year = {2017},
  booktitle = {Advances in Neural Information Processing Systems},
  volume = {30},
  url = {https://papers.nips.cc/paper_files/paper/2017/hash/f87522788a2be2d171666752f97ddebb-Abstract.html}
}

@misc{liu2026selfdistilledpolicygradient,
      title={Self-Distilled Policy Gradient}, 
      author={Yifeng Liu and Shiyuan Zhang and Yifan Zhang and Quanquan Gu},
      year={2026},
      eprint={2606.04036},
      archivePrefix={arXiv},
      primaryClass={cs.LG},
      url={https://arxiv.org/abs/2606.04036}, 
}

@inproceedings{shenfeld2026sdft,
  title = {{Self-Distillation Enables Continual Learning}},
  author = {Shenfeld, Idan and Damani, Mehul and H{\"{u}}botter, Jonas and Agrawal, Pulkit},
  year = {2026},
  booktitle = {International Conference on Machine Learning},
  url = {https://arxiv.org/abs/2601.19897}
}

@inproceedings{hubotter2026sdpo,
  title = {{Reinforcement Learning via Self-Distillation}},
  author = {H{\"{u}}botter, Jonas and L{\"{u}}beck, Frederike and Behric, Lejs and Baumann, Anton and Bagatella, Marco and Marta, Daniel and Hakimi, Ido and Shenfeld, Idan and Kleine Buening, Thomas and Guestrin, Carlos and Krause, Andreas},
  year = {2026},
  booktitle = {International Conference on Machine Learning},
  url = {https://arxiv.org/abs/2601.20802}
}

@inproceedings{zhao2026sdr,
  title = {{Self-Distilled Reasoner: On-Policy Self-Distillation for Large Language Models}},
  author = {Zhao, Siyan and Xie, Zhihui and Liu, Mengchen and Huang, Jing and Pang, Guan and Chen, Feiyu and Grover, Aditya},
  year = {2026},
  booktitle = {International Conference on Machine Learning},
  url = {https://arxiv.org/abs/2601.18734}
}

@inproceedings{wang2023selfconsistency,
  title = {{Self-Consistency Improves Chain of Thought Reasoning in Language Models}},
  author = {Wang, Xuezhi and Wei, Jason and Schuurmans, Dale and Le, Quoc and Chi, Ed and Narang, Sharan and Chowdhery, Aakanksha and Zhou, Denny},
  year = {2023},
  booktitle = {International Conference on Learning Representations},
  url = {https://arxiv.org/abs/2203.11171}
}

@inproceedings{zhu2025negative,
  title = {{The Surprising Effectiveness of Negative Reinforcement in LLM Reasoning}},
  author = {Zhu, Xinyu and Xia, Mengzhou and Wei, Zhepei and Chen, Wei-Lin and Chen, Danqi and Meng, Yu},
  year = {2025},
  booktitle = {Advances in Neural Information Processing Systems},
  volume = {38},
  url = {https://arxiv.org/abs/2506.01347}
}

@inproceedings{yao2023tot,
  title = {{Tree of Thoughts: Deliberate Problem Solving with Large Language Models}},
  author = {Yao, Shunyu and Yu, Dian and Zhao, Jeffrey and Shafran, Izhak and Griffiths, Thomas L. and Cao, Yuan and Narasimhan, Karthik},
  year = {2023},
  booktitle = {Advances in Neural Information Processing Systems},
  url = {https://arxiv.org/abs/2305.10601}
}

@inproceedings{kim2025metastable,
  title = {{Metastable Dynamics of Chain-of-Thought Reasoning: Provable Benefits of Search, RL and Distillation}},
  author = {Juno Kim and Denny Wu and Jason D. Lee and Taiji Suzuki},
  year = {2025},
  booktitle = {Proceedings of the 42nd International Conference on Machine Learning},
  url = {https://proceedings.mlr.press/v267/kim25aj.html},
  pages = {30791--30825}
}

@article{bu2025distributional,
  title = {{Distributional Biases in Post-Training: A Markovian Analysis of Reasoning Trajectories}},
  author = {Bu, Dake and Huang, Wei and Han, Andi and Nitanda, Atsushi and Xue, Bo and Zhang, Qingfu and Wong, Hau-San and Suzuki, Taiji},
  year = {2025},
  journal = {arXiv preprint arXiv:2511.07368},
  url = {https://arxiv.org/abs/2511.07368},
  note = {Preprint}
}

@article{wei2026backtrack,
  title = {{Provable Benefits of RLVR over SFT for Reasoning Models: Learning to Backtrack Efficiently}},
  author = {Wei, Stanley and Kim, Juno},
  year = {2026},
  journal = {arXiv preprint arXiv:2606.22938},
  url = {https://arxiv.org/abs/2606.22938},
  note = {Preprint}
}

@inproceedings{shenfeld2026razor,
  title = {{RL's Razor: Why Online Reinforcement Learning Forgets Less}},
  author = {Shenfeld, Idan and Pari, Jyothish and Agrawal, Pulkit},
  year = {2026},
  booktitle = {International Conference on Learning Representations},
  url = {https://proceedings.iclr.cc/paper_files/paper/2026/hash/618c95f4557c15b253fb0e6f548ea0c0-Abstract-Conference.html}
}

@article{zhao2026unified,
  title={Large Language Model Post-Training: A Unified View of Off-Policy and On-Policy Learning},
  author={Zhao, Shiwan and Wang, Zhihu and Zhao, Xuyang and Zhou, Jiaming and Xu, Caiyue and Liu, Chenfei and Zhang, Liting and Jia, Yuhang and Zhang, Yanzhe and Yu, Hualong and Xu, Zichen and Li, Qicheng and Qin, Yong},
  year={2026},
  journal={arXiv preprint arXiv:2604.07941},
  url={https://arxiv.org/abs/2604.07941}
}

@article{luo2023forgetting,
  title={An Empirical Study of Catastrophic Forgetting in Large Language Models During Continual Fine-tuning},
  author={Luo, Yun and Yang, Zhen and Meng, Fandong and Li, Yafu and Zhou, Jie and Zhang, Yue},
  year={2023},
  journal={arXiv preprint arXiv:2308.08747},
  url={https://arxiv.org/abs/2308.08747}
}

@article{lai2025reinforcement,
  title={Reinforcement Fine-Tuning Naturally Mitigates Forgetting in Continual Post-Training},
  author={Lai, Song and Zhao, Haohan and Feng, Rong and Ma, Changyi and Liu, Wenzhuo and Zhao, Hongbo and Lin, Xi and Yi, Dong and Zhang, Qingfu and Liu, Hongbin and Meng, Gaofeng and Zhu, Fei},
  year={2025},
  journal={arXiv preprint arXiv:2507.05386},
  url={https://arxiv.org/abs/2507.05386}
}

@article{chen2025retaining,
  title={Retaining by Doing: The Role of On-Policy Data in Mitigating Forgetting},
  author={Chen, Howard and Razin, Noam and Narasimhan, Karthik and Chen, Danqi},
  year={2025},
  journal={arXiv preprint arXiv:2510.18874},
  url={https://arxiv.org/abs/2510.18874}
}

@article{zhu2026sftconflicts,
  title={{SFT} Conflicts, {RL} Coexists: A Theoretical and Empirical Analysis of Multi-Task Learning for {LLMs}},
  author={Zhu, Kejian and Jin, Zhuoran and Tu, Shangqing and Yuan, Hongbang and Bai, Yushi and Liu, Kang and Li, Juanzi and Zhao, Jun},
  year={2026},
  journal={arXiv preprint arXiv:2608.03573},
  url={https://arxiv.org/abs/2608.03573}
}

@article{ye2026opcd,
  title={On-Policy Context Distillation for Language Models},
  author={Ye, Tianzhu and Dong, Li and Wu, Xun and Huang, Shaohan and Wei, Furu},
  year={2026},
  journal={arXiv preprint arXiv:2602.12275},
  url={https://arxiv.org/abs/2602.12275}
}

@inproceedings{sanford2024graph,
  title = {{Understanding Transformer Reasoning Capabilities via Graph Algorithms}},
  author = {Sanford, Clayton and Fatemi, Bahare and Hall, Ethan and Tsitsulin, Anton and Kazemi, Mehran and Halcrow, Jonathan and Perozzi, Bryan and Mirrokni, Vahab},
  booktitle = {Advances in Neural Information Processing Systems},
  volume = {37},
  year = {2024},
  url = {https://proceedings.neurips.cc/paper_files/paper/2024/hash/8f395480c04ac6dfb2c2326a639df88e-Abstract-Conference.html}
}

@inproceedings{bachmann2024pitfalls,
  title = {{The Pitfalls of Next-Token Prediction}},
  author = {Bachmann, Gregor and Nagarajan, Vaishnavh},
  booktitle = {Proceedings of the 41st International Conference on Machine Learning},
  volume = {235},
  pages = {2296--2318},
  series = {Proceedings of Machine Learning Research},
  year = {2024},
  url = {https://proceedings.mlr.press/v235/bachmann24a.html}
}

@article{zhao2023surveylargelanguagemodels,
  title = {{A Survey of Large Language Models}},
  author = {Zhao, Wayne Xin and Zhou, Kun and Li, Junyi and Tang, Tianyi and Wang, Xiaolei and Hou, Yupeng and Min, Yingqian and Zhang, Beichen and Zhang, Junjie and Dong, Zican and Du, Yifan and Yang, Chen and Chen, Yushuo and Chen, Zhipeng and Jiang, Jinhao and Ren, Ruiyang and Li, Yifan and Tang, Xinyu and Liu, Zikang and Liu, Peiyu and Nie, Jian-Yun and Wen, Ji-Rong},
  year = {2023},
  journal = {arXiv preprint arXiv:2303.18223},
  url = {https://arxiv.org/abs/2303.18223}
}

@article{zhang2023instructiontuninglargelanguage,
  title = {{Instruction Tuning for Large Language Models: A Survey}},
  author = {Zhang, Shengyu and Dong, Linfeng and Li, Xiaoya and Zhang, Sen and Sun, Xiaofei and Wang, Shuhe and Li, Jiwei and Hu, Runyi and Zhang, Tianwei and Wu, Fei and Wang, Guoyin},
  year = {2023},
  journal = {arXiv preprint arXiv:2308.10792},
  url = {https://arxiv.org/abs/2308.10792}
}

@article{wang2023aligninglargelanguagemodels,
  title = {{Aligning Large Language Models with Human: A Survey}},
  author = {Wang, Yufei and Zhong, Wanjun and Li, Liangyou and Mi, Fei and Zeng, Xingshan and Huang, Wenyong and Shang, Lifeng and Jiang, Xin and Liu, Qun},
  year = {2023},
  journal = {arXiv preprint arXiv:2307.12966},
  url = {https://arxiv.org/abs/2307.12966}
}

@article{kaufmann2025survey,
  title = {{A Survey of Reinforcement Learning from Human Feedback}},
  author = {Kaufmann, Timo and Weng, Paul and Bengs, Viktor and H{\"{u}}llermeier, Eyke},
  year = {2025},
  journal = {Transactions on Machine Learning Research},
  url = {https://arxiv.org/abs/2312.14925}
}

@article{li2026reviewreinforcementlearning,
  title = {{Review of Reinforcement Learning for Large Language Models: Formulations, Algorithms, and Opportunities}},
  author = {Li, Ziniu and Wang, Pengyuan and Xu, Tian and Ding, Tian and Sun, Ruoyu and Yu, Yang},
  year = {2026},
  journal = {Transactions on Machine Learning Research},
  url = {https://www.liziniu.org/docs/RL4LLM_Survey.pdf}
}

@article{wang2023trace,
  title = {{TRACE: A Comprehensive Benchmark for Continual Learning in Large Language Models}},
  author = {Wang, Xiao and Zhang, Yuansen and Chen, Tianze and Gao, Songyang and Jin, Senjie and Yang, Xianjun and Xi, Zhiheng and Zheng, Rui and Zou, Yicheng and Gui, Tao and Zhang, Qi and Huang, Xuanjing},
  year = {2023},
  journal = {arXiv preprint arXiv:2310.06762},
  url = {https://arxiv.org/abs/2310.06762}
}

@inproceedings{li2024revisitingcatastrophicforgetting,
  title = {{Revisiting Catastrophic Forgetting in Large Language Model Tuning}},
  author = {Li, Hongyu and Ding, Liang and Fang, Meng and Tao, Dacheng},
  year = {2024},
  booktitle = {Findings of the Association for Computational Linguistics: EMNLP 2024},
  publisher = {Association for Computational Linguistics},
  pages = {4297--4308},
  doi = {10.18653/v1/2024.findings-emnlp.249},
  url = {https://aclanthology.org/2024.findings-emnlp.249/}
}

@inproceedings{wang2023orthogonalsubspacelearning,
  title = {{Orthogonal Subspace Learning for Language Model Continual Learning}},
  author = {Wang, Xiao and Chen, Tianze and Ge, Qiming and Xia, Han and Bao, Rong and Zheng, Rui and Zhang, Qi and Gui, Tao and Huang, Xuanjing},
  year = {2023},
  booktitle = {Findings of the Association for Computational Linguistics: EMNLP 2023},
  publisher = {Association for Computational Linguistics},
  pages = {10658--10671},
  doi = {10.18653/v1/2023.findings-emnlp.715},
  url = {https://aclanthology.org/2023.findings-emnlp.715/}
}

@article{biderman2024loralearnslessforgets,
  title = {{LoRA Learns Less and Forgets Less}},
  author = {Biderman, Dan and Portes, Jacob and Gonzalez Ortiz, Jose Javier and Paul, Mansheej and Greengard, Philip and Jennings, Connor and King, Daniel and Havens, Sam and Chiley, Vitaliy and Frankle, Jonathan and Blakeney, Cody and Cunningham, John P.},
  year = {2024},
  journal = {Transactions on Machine Learning Research},
  url = {https://openreview.net/forum?id=aloEru2qCG}
}
\bibliographystyle{plainnat}

\clearpage
\appendix

\addtocontents{toc}{\protect\appendixcontentsformat}
\begingroup
\setlength{\parskip}{0pt}
\tableofcontents
\endgroup
\clearpage

\section{Additional related work}
\label{app:related-works}
In this section, we discuss additional related work.
\paragraph{On-policy versus off-policy learning.}
A central focus of this paragraph is the preservation of previously acquired capabilities.
Continual SFT can degrade existing knowledge and reasoning abilities through catastrophic forgetting \citep{luo2023forgetting}.
Recent comparisons report that RL often preserves prior capabilities better than SFT at comparable new-task performance \citep{lai2025reinforcement,shenfeld2026razor,chen2025retaining}.
Controlled experiments further identify on-policy sampling as an important contributor to this retention advantage \citep{chen2025retaining}.
\citet{zhu2026sftconflicts} attribute the contrasting forgetting behavior of SFT and RL to inter-task interference: SFT exhibits substantial conflicts during sequential training, whereas RL induces less interferences that allow multiple task capabilities to coexist.
From a complementary distributional perspective, \citet{shenfeld2026razor} argue that standard SFT fits external demonstrations without explicitly accounting for proximity to the base model, whereas on-policy updates are biased toward solutions closer to the initial policy, helping explain their reduced forgetting.
Building on \citet{zhu2026sftconflicts}'s task-interference perspective, we develop more theoretical analysis of how SFT and on-policy SD differ in learning new tasks and preserving previously acquired capabilities.

\paragraph{CoT reasoning as graph search.}
Graph reasoning is a broadly accepted approach for analyzing the reasoning capabilities of language models.
\citet{sanford2024graph} characterize Transformer expressivity on graph reasoning tasks
in terms of depth, width, and additional computation tokens,
and empirically evaluate their theoretical predictions on GraphQA.
\citet{bachmann2024pitfalls} use path-star graphs to demonstrate
how teacher-forced next-token prediction can favor shortcuts that hinder learning to plan.
\citet{kim2025metastable} model CoT as a metastable Markov process on clustered graphs
and establish benefits of search, RL, and distillation for crossing rare inter-cluster transitions.
\citet{wei2026backtrack} study backtracking in graph pathfinding
and prove an exponential inference-time advantage of RLVR over SFT
when SFT is trained only on golden shortest paths without negative examples.
Our work uses a CoT graph model to analyze how SFT and self-distillation affect the retention of what they learned.

\section{Realizability and proofs}
\label{app:proofs}
The proofs are organized as sketches, supporting lemmas, and main proofs.
The algorithm of SD and SFT is described in the following.

\begin{algorithm}[htbp]
\caption{Sequential SD and SFT in the reasoning-graph model}
\label{alg:sequential-training}
\begin{algorithmic}[1]
\Require method $\in \{\text{SD, SFT}\}$
\Require $\mathbf W^{(0)}$, $B$, $\{n_\ell\}_{\ell=1}^L$, $\{\eta_j\}_{j\geq 1}$, $\epsilon_{\mathrm{rel}}\in(0,1)$

\State $\mathbf W \gets \mathbf W^{(0)}$
\For{$\ell=1,\ldots,L$}
    \For{$j=1,\ldots,n_\ell$}
        \For{$b=1,\ldots,B$}
            \State Sample
            $y^{(b)} \sim \pi_{\mathbf W}(\cdot \mid x_{\ell})$ and
            $\gamma^{(b)}\sim\operatorname{Unif}(\mathcal P_\ell)$.
            \If{method $=$ SFT}
                \State $y^{*(b)}\gets\gamma^{(b)}$
            \Else
                \State Construct $c^{(b)}$ from all inter-cluster
                transitions in $\gamma^{(b)}$.
                \State Construct $\pi_{\mathbf W}(\cdot \mid x_{\ell}, c^{(b)})$.
            \EndIf
            \EndFor
        \State $\mathbf W\gets\mathbf W
        -\eta_j\nabla_{\mathbf W}
        L_{\mathrm{method}}(\mathbf W;\mathcal D_\ell)$
        \EndFor
\EndFor
\State \Return $\mathbf W$
\end{algorithmic}
\end{algorithm}

\addtocontents{toc}{\protect\Needspace{5\protect\baselineskip}}
\subsection{Realizability on routing-compatible layered DAGs}\label{app:realizability}
\addtocontents{toc}{\protect\nopagebreak[4]}
\addtocontents{toc}{\protect\proofcontentsentry{prop:routing-realizability}{Proposition~\protect\ref*{prop:routing-realizability}. Realizability on layered DAGs}}
\addtocontents{toc}{\protect\nopagebreak[4]}
\addtocontents{toc}{\protect\proofcontentschild{lem:routing-kernel}{Lemma~\protect\ref*{lem:routing-kernel}. Positive routing with boundary coverage}}
\addtocontents{toc}{\protect\proofcontentschild{lem:boundary-kernel}{Lemma~\protect\ref*{lem:boundary-kernel}. Boundary kernels with small transition probabilities}}
\addtocontents{toc}{\protect\proofcontentschild{lem:finite-realization}{Lemma~\protect\ref*{lem:finite-realization}. Simultaneous finite softmax realization}}

We consider layered DAGs satisfying the support conditions below.
For this graph family, the proposition following the description shows that
Assumption~\ref{ass:pretrained-diversity} can be realized by a finite pre-trained parameter
$\mathbf W^{(0)}$.
For $u\in V$, write $\mathcal N(u):=\{v\in V:(u,v)\in E\}$.
Suppose that $G$ is a finite layered DAG, meaning that
\begin{align}
   V=V_0\mathbin{\dot\cup}\cdots\mathbin{\dot\cup}V_H,
   \qquad
   (u,v)\in E,\ u\in V_h,\ v\in V_{h'}
   \quad\Longrightarrow\quad h<h'.
   \label{eq:routing-layered-dag}
\end{align}
We call such a graph \emph{routing-compatible} for the task family if the
following conditions hold:
\begin{enumerate}
   \item Every boundary node has at least one same-cluster successor:
   \begin{align}
      R_k(u):=\mathcal N(u)\cap C_k\neq\emptyset,
      \qquad u\in\partial C_k.
      \label{eq:routing-boundary-support}
   \end{align}
   These later-layer nodes provide support for the residual probability mass
   of a finite softmax distribution; they are not thereby declared relevant
   starts.
   \item Every relevant start belongs to the interior of its cluster, and for
   every task $\mathcal D_\ell$, every non-terminal cluster $C_k$ with
   $\partial_\ell^\star C_k\neq\emptyset$, and every relevant start $s$ in
   $C_k$,
   \begin{align}
      \partial_\ell^\star C_k
      \subseteq\mathcal N(s)
      \subseteq\partial C_k.
      \label{eq:routing-start-support}
   \end{align}
\end{enumerate}
The numbers of layers, clusters, nodes, valid boundaries, and residual
successors are arbitrary within this graph family.

\begin{realizabilityproposition}[Realizability on routing-compatible layered DAGs]
\label{prop:routing-realizability}
Suppose that $G$ is a routing-compatible layered DAG in the sense of
Equations~\ref{eq:routing-layered-dag}--\ref{eq:routing-start-support}.
Under the full-rank condition on $(x_1,\ldots,x_L)$ in the problem setup, fix
$\epsilon_1\in(0,1)$ and $\epsilon_2>0$ such that
\begin{align}
   (K-1)\epsilon_2\leq\frac12.
   \label{eq:routing-epsilon-two}
\end{align}
Then there exists a finite
$\mathbf W^{(0)}\in\mathbb R^{d\times|E|}$ satisfying every clause of
Assumption~\ref{ass:pretrained-diversity} simultaneously for all task inputs $x_1,\ldots,x_L$.
\end{realizabilityproposition}

\paragraph{Sketch of proof of Proposition~\ref{prop:routing-realizability}.}
\begin{itemize}
   \item Construct positive within-cluster routing kernels with small invalid-boundary probability and uniform coverage of valid boundaries (Lemma~\ref{lem:routing-kernel}).
   \item Construct normalized boundary kernels assigning the prescribed small probability to every inter-cluster edge (Lemma~\ref{lem:boundary-kernel}).
   \item Realize all task-wise kernels simultaneously with one finite softmax parameter (Lemma~\ref{lem:finite-realization}).
\end{itemize}

\begin{lemma}[A positive routing kernel with boundary coverage]
\label{lem:routing-kernel}
Under \Eqref{eq:routing-start-support}, there is a strictly positive kernel at every
relevant start for which the first boundary is reached in one step,
the invalid-boundary probability is at most $\epsilon_1$, and each valid
boundary has probability at least $1/(2M)$.
\end{lemma}
\begin{proof}
Fix a task $\mathcal D_\ell$, a non-terminal cluster $C_k$ with
$\partial_\ell^\star C_k\neq\emptyset$, and a relevant start $s$ in $C_k$.
Define
\begin{align}
   I_{\ell,k}(s)
   :=\mathcal N(s)\setminus\partial_\ell^\star C_k.
\end{align}
By Equation~\ref{eq:routing-start-support}, every outgoing neighbor of $s$ is
a boundary node and every valid boundary is an outgoing neighbor.  Prescribe
the transition kernel at $s$ by
\begin{align}
   P_\ell(v\mid s)
   :=
   \begin{cases}
      \displaystyle
      \frac{1}{(1+\epsilon_1)|\partial_\ell^\star C_k|},
      &v\in\partial_\ell^\star C_k,\quad I_{\ell,k}(s)\neq\emptyset,\\[7pt]
      \displaystyle
      \frac{\epsilon_1}{(1+\epsilon_1)|I_{\ell,k}(s)|},
      &v\in I_{\ell,k}(s),\quad I_{\ell,k}(s)\neq\emptyset,\\[7pt]
      \displaystyle
      \frac{1}{|\partial_\ell^\star C_k|},
      &v\in\partial_\ell^\star C_k,\quad I_{\ell,k}(s)=\emptyset.
   \end{cases}
   \label{eq:routing-start-kernel}
\end{align}
This is a strictly positive probability distribution on $\mathcal N(s)$.
Because $\mathcal N(s)\subseteq\partial C_k$, the first boundary is reached
in exactly one step, and hence
\begin{align}
   \mathbb P_{P_\ell}\!\left(
      U_{\ell,k}^{s}\notin\partial_\ell^\star C_k
      \mid Y_0=s,x_\ell
   \right)
   &\leq\frac{\epsilon_1}{1+\epsilon_1}
   \leq\epsilon_1,
   \label{eq:routing-error}\\
   \mathbb P_{P_\ell}\!\left(
      U_{\ell,k}^{s}=u
      \mid Y_0=s,x_\ell
   \right)
   &\geq
   \frac{1}{(1+\epsilon_1)|\partial_\ell^\star C_k|}
   \geq\frac{1}{2M}
   \label{eq:routing-diversity}
\end{align}
for every $u\in\partial_\ell^\star C_k$.  Thus the first-boundary hitting
time is almost surely finite, the routing error is $O(\epsilon_1)$, and every
valid boundary receives probability $\Omega(1/M)$, uniformly over all
relevant starts.
\end{proof}

\begin{lemma}[Boundary kernels with small inter-cluster probabilities]
\label{lem:boundary-kernel}
Under \Eqref{eq:routing-boundary-support} and \Eqref{eq:routing-epsilon-two},
each boundary admits a strictly positive normalized kernel assigning exactly
$\epsilon_2$ to every inter-cluster edge.
\end{lemma}
\begin{proof}
We next prescribe the kernel at boundary nodes.  For
$u\in\partial C_k$, let
\begin{align}
   d_{u,k}^{\mathrm{out}}
   :=|\mathcal N(u)\setminus C_k|.
\end{align}
For every task input, set
\begin{align}
   P_\ell(v\mid u)
   :=
   \begin{cases}
      \epsilon_2,
      &v\in\mathcal N(u)\setminus C_k,\\[4pt]
      \displaystyle
      \frac{1-d_{u,k}^{\mathrm{out}}\epsilon_2}{|R_k(u)|},
      &v\in R_k(u).
   \end{cases}
   \label{eq:routing-boundary-kernel}
\end{align}
The graph convention in the problem setup gives
$d_{u,k}^{\mathrm{out}}\leq K-1$.  Therefore
Equation~\ref{eq:routing-epsilon-two}, together with
Equation~\ref{eq:routing-boundary-support}, makes
Equation~\ref{eq:routing-boundary-kernel} a strictly positive probability
distribution.  
\end{proof}

\begin{lemma}[Simultaneous finite softmax realization]
\label{lem:finite-realization}
Every family of strictly positive task-wise transition kernels $P_\ell$
on the given edge support is realized by one finite $\mathbf W^{(0)}$
when the task-input matrix has full column rank.
\end{lemma}
\begin{proof}
For each task input, assign any strictly positive probability distribution to
the outgoing edges of each remaining node whose kernel has not yet been
specified.  No kernel is needed at a terminal node with no outgoing edge.  We
now realize all these task-wise kernels with a single finite parameter matrix.
Let
\begin{align}
   X
   :=
   \begin{bmatrix}
      x_1^\top\\
      \vdots\\
      x_L^\top
   \end{bmatrix}
   \in\mathbb R^{L\times d}.
\end{align}
The full-rank input condition gives $\operatorname{rank}(X)=L$, so $XX^\top$
is invertible.  For every $(u,v)\in E$, define
\begin{align}
   \color{black}{
   \ell_{u,v}
   :={}
   \begin{bmatrix}
      \log P_1(v\mid u)\\
      \vdots\\
      \log P_L(v\mid u)
   \end{bmatrix}, \quad
   \mathbf W^{(0)}e_{u,v}
   :={}
   X^\top(XX^\top)^{-1}\ell_{u,v}.}
   \label{eq:routing-parameter-construction}
\end{align}
All prescribed edge probabilities are strictly positive, so
$\mathbf W^{(0)}$ is finite, and
\begin{align}
   x_\ell^\top\mathbf W^{(0)}e_{u,v}
   =\log P_\ell(v\mid u).
\end{align}
Since $P_\ell(\cdot\mid u)$ is normalized over $\mathcal N(u)$, the softmax
predictor satisfies
\begin{align}
   \pi_{\mathbf W^{(0)}}(v\mid u,x_\ell)
   &=
   \frac{\exp(\log P_\ell(v\mid u))}
   {\sum_{w\in\mathcal N(u)}\exp(\log P_\ell(w\mid u))}
   =P_\ell(v\mid u).
   \label{eq:routing-softmax-realization}
\end{align}
\end{proof}

\paragraph{Main proof of Proposition~\ref{prop:routing-realizability}.}
\label{proof:routing-realizability}
\begin{proof}
Lemmas~\ref{lem:routing-kernel} and~\ref{lem:boundary-kernel} specify compatible
kernels, since relevant starts are interior nodes. Lemma~\ref{lem:finite-realization}
realizes them simultaneously. Finally, Equation~\ref{eq:routing-layered-dag} implies that every transition
strictly increases the layer index. Equations~\ref{eq:routing-error},
\ref{eq:routing-diversity}, and~\ref{eq:routing-boundary-kernel} verify Assumption~\ref{ass:pretrained-diversity}.
\end{proof}

\paragraph{Scope of the graph conditions.}
The support conditions above give a concrete finite-parameter realization of
Assumption~\ref{ass:pretrained-diversity}.
However, our learning and forgetting results apply more broadly under their respective assumptions.
Equation~\ref{eq:routing-boundary-support} ensures normalization when inter-cluster
transitions are rare.
Equation~\ref{eq:routing-start-support} simplifies the construction by ensuring
that the first boundary is reached in one step.
More generally, the within-cluster routing conditions allow multiple steps,
provided boundary hitting is almost surely finite and the bounds on
invalid-boundary probability and valid-boundary coverage in
Equations~\ref{eq:pretrained-boundary-error} and~\ref{eq:pretrained-boundary-coverage}
hold uniformly over relevant starts.

\addtocontents{toc}{\protect\Needspace{4\protect\baselineskip}}
\subsection{Gradient of the SD loss}\label{app:sd-gradient}
\addtocontents{toc}{\protect\nopagebreak[4]}
\addtocontents{toc}{\protect\proofcontentsentry{lem:sd-gradient}{Lemma~\protect\ref*{lem:sd-gradient}. Gradient of the SD loss}}
\addtocontents{toc}{\protect\nopagebreak[4]}
\addtocontents{toc}{\protect\proofcontentschild{lem:sd-kl-identity}{Lemma~\protect\ref*{lem:sd-kl-identity}. The one-state stop-gradient identity}}
\addtocontents{toc}{\protect\proofcontentschild{lem:sd-hint-cancellation}{Lemma~\protect\ref*{lem:sd-hint-cancellation}. Cancellation of inactive hints}}

\begin{lemma}[Gradient of the SD loss]
\label{lem:sd-gradient}
\begin{align}
   \nabla_{\mathbf{W}}L_{\mathrm{SD}}(\mathbf{W};\mathcal{D}_{\ell})
   =&\frac{\log\epsilon_{\mathrm{rel}}}{B}\sum_{b = 1}^B\sum_{i=1}^{I^{(b)}}\mathbf{1}\{\tau_i^{(b)}<T^{(b)}\}
   \pi_{\mathbf{W}}(v_i^{(b)}\mid u_i^{(b)},x_{\ell}^{(b)})\\
   &\cdot x_{\ell}^{(b)}\left(
   e_{u_i^{(b)},v_i^{(b)}}
   -\sum_{v\in\mathcal{N}(u_i^{(b)})}
   \pi_{\mathbf{W}}(v\mid u_i^{(b)},x_{\ell}^{(b)})e_{u_i^{(b)},v}
   \right)^{\top}
   \label{eq:sd-gradient}
\end{align}
\end{lemma}
\paragraph{Sketch of proof of Lemma~\ref{lem:sd-gradient}.}
\begin{itemize}
   \item Differentiate the one-state stop-gradient KL and show that successor-independent teacher factors contribute zero gradient (Lemma~\ref{lem:sd-kl-identity}).
   \item Identify the active hint factor and show that the local gradient vanishes away from hint-source states (Lemma~\ref{lem:sd-hint-cancellation}).
   \item Sum the local gradients over visited states and average over the mini-batch to obtain the stated sparse gradient formula (Lemma~\ref{lem:sd-gradient}).
\end{itemize}

\begin{lemma}[The one-state stop-gradient identity]
\label{lem:sd-kl-identity}
For $p_t(v)=\pi_{\mathbf W}(v\mid y_t,x_\ell)$ and the stop-gradient
teacher $q_t$, the KL and softmax derivatives satisfy
\Eqref{eq:mf-kl-grad}--\Eqref{eq:mf-softmax-grad}. In particular,
a teacher factor independent of the proposed successor contributes zero gradient.
\end{lemma}
\begin{proof}
For one batch element and one time $t<T$, set
\begin{align*}
u:=y_t,\quad
h_t:=(y_0,\ldots,y_t),\quad
p_t(v):=\pi_{\mathbf W}(v\mid u,x_\ell),\quad 
q_t(v):=\pi_{\mathbf W}(v\mid h_t,x_\ell,c).
\end{align*}
With
\begin{gather}
\chi_{\mathbf W}(c\mid h_t,v,x_\ell)
:=
\prod_{i=1}^{I}\chi_{\mathbf W}(c_i\mid h_t,v,x_\ell),
\\
q_t(v)
=
\frac{
\chi_{\mathbf W}(c\mid h_t,v,x_\ell)p_t(v)
}{
\sum_{z\in\mathcal N(u)}
\chi_{\mathbf W}(c\mid h_t,z,x_\ell)p_t(z)
},
\label{eq:mf-hinted-teacher}
\end{gather}
we have
\begin{gather}
   \color{black}{
\nabla_{\mathbf W}\KL(p_t\parallel\operatorname{SG}(q_t))
=
\sum_{v\in\mathcal N(u)}
\nabla_{\mathbf W}p_t(v)
\log\frac{p_t(v)}{q_t(v)},}
\label{eq:mf-kl-grad}
\\
\color{black}{
\nabla_{\mathbf W}p_t(v)
=
p_t(v)x_\ell
\left(
e_{u,v}-\sum_{z\in\mathcal N(u)}p_t(z)e_{u,z}
\right)^\top,}
\label{eq:mf-softmax-grad}
\\
\color{black}{\sum_{v\in\mathcal N(u)}\nabla_{\mathbf W}p_t(v)=0}.
\end{gather}
\end{proof}

\begin{lemma}[Cancellation of inactive hints]
\label{lem:sd-hint-cancellation}
With the hint in \Eqref{eq:path-level-hint-feedback}, the one-state gradient
vanishes unless $t=\tau_i<T$ for a hint edge $(u_i,v_i)$. At such a time it equals
$\log(\epsilon_{\mathrm{rel}})\nabla_{\mathbf W}\pi_{\mathbf W}(v_i\mid u_i,x_\ell)$.
\end{lemma}
\begin{proof}
By Lemma~\ref{lem:sd-kl-identity}, it suffices to identify the successor-dependent hint factor.
For each hint edge $c_i$,
\begin{align*}
\chi_{\mathbf W}(c_i\mid h_t,v,x_\ell)
\propto
\begin{cases}
1,
&\tau_i=t,\ v=v_i,
\\
\epsilon_{\mathrm{rel}},
&\tau_i=t,\ v\neq v_i,
\\
1,
&\tau_i<t,\ y_{\tau_i+1}=v_i,
\\
\epsilon_{\mathrm{rel}},
&\tau_i<t,\ y_{\tau_i+1}\neq v_i,
\\
Z_{\mathbf W,i}(x_\ell),
&\tau_i>t,
\end{cases}
\end{align*}
where
\begin{align*}
Z_{\mathbf W,i}(x_\ell)
&=
\pi_{\mathbf W}(v_i\mid u_i,x_\ell)
+
\epsilon_{\mathrm{rel}}\left(
1-\pi_{\mathbf W}(v_i\mid u_i,x_\ell)
\right).
\end{align*}
Hence, if
\begin{align*}
t\notin\{\tau_1,\ldots,\tau_I\},
\end{align*}
then
\begin{align*}
\chi_{\mathbf W}(c\mid h_t,v,x_\ell)
&\equiv \text{const. for } v
\quad\Longrightarrow\quad
q_t(v)=p_t(v)
\quad\Longrightarrow\quad
\nabla_{\mathbf W}\KL(p_t\parallel\operatorname{SG}(q_t))=0.
\end{align*}

If $t=\tau_i<T$, node-simplicity of the sampled teacher path gives
$\tau_{i'}\neq t$ for $i'\neq i$, and
\begin{gather*}
\color{black}{Z_t:=p_t(v_i)+\epsilon_{\mathrm{rel}}(1-p_t(v_i)),}
\\
\color{black}{
q_t(v)
=
\begin{cases}
p_t(v)/Z_t,&v=v_i,\\
\epsilon_{\mathrm{rel}} p_t(v)/Z_t,&v\neq v_i.
\end{cases}}
\end{gather*}
Thus
\begin{align*}
\nabla_{\mathbf W}\KL(p_t\parallel\operatorname{SG}(q_t))
&=
\sum_v\nabla_{\mathbf W}p_t(v)\log Z_t
-
(\log\epsilon_{\mathrm{rel}})
\sum_{v\neq v_i}\nabla_{\mathbf W}p_t(v)
\\
&=
(\log\epsilon_{\mathrm{rel}})\nabla_{\mathbf W}p_t(v_i)
\\
&=
(\log\epsilon_{\mathrm{rel}})
p_t(v_i)x_\ell
\left(
e_{u_i,v_i}
-
\sum_{v\in\mathcal N(u_i)}
p_t(v)e_{u_i,v}
\right)^\top.
\end{align*}
\end{proof}

\paragraph{Main proof of Lemma~\ref{lem:sd-gradient}.}
\label{proof:sd-gradient}
\begin{proof}
Apply Lemma~\ref{lem:sd-hint-cancellation} to each visited state.
Node simplicity permits each hint source to contribute at most once.
Averaging over the $B$ paths gives
\begin{align*}
\nabla_{\mathbf W}L_{\mathrm{SD}}(\mathbf W;\mathcal D_\ell)
&=
\frac{\log\epsilon_{\mathrm{rel}}}{B}
\sum_{b=1}^{B}
\sum_{i=1}^{I^{(b)}}
\mathbf 1\{\tau_i^{(b)}<T^{(b)}\}
\pi_{\mathbf W}(v_i^{(b)}\mid u_i^{(b)},x_\ell^{(b)})
\\
&\qquad\cdot
x_\ell^{(b)}
\left(
e_{u_i^{(b)},v_i^{(b)}}
-
\sum_{v\in\mathcal N(u_i^{(b)})}
\pi_{\mathbf W}(v\mid u_i^{(b)},x_\ell^{(b)})
e_{u_i^{(b)},v}
\right)^\top.
\end{align*}
\end{proof}

\addtocontents{toc}{\protect\Needspace{7\protect\baselineskip}}
\subsection{Learning hint-supported inter-cluster transitions}
\label{app:sd-transitions}
\addtocontents{toc}{\protect\nopagebreak[4]}
\addtocontents{toc}{\protect\proofcontentsentry{prop:sd-transition-learning}{Proposition~\protect\ref*{prop:sd-transition-learning}. Learning hint-supported transitions}}
\addtocontents{toc}{\protect\nopagebreak[4]}
\addtocontents{toc}{\protect\proofcontentschild{lem:sd-gradient}{Lemma~\protect\ref*{lem:sd-gradient}. Gradient of the SD loss}}
\addtocontents{toc}{\protect\proofcontentschild{lem:restart-update}{Lemma~\protect\ref*{lem:restart-update}. Local updates and invariant interior transitions}}
\addtocontents{toc}{\protect\proofcontentschild{lem:restart-odds}{Lemma~\protect\ref*{lem:restart-odds}. Error from weighted update counts}}
\addtocontents{toc}{\protect\proofcontentschild{lem:restart-coverage}{Lemma~\protect\ref*{lem:restart-coverage}. Update opportunities from retained coverage}}
\addtocontents{toc}{\protect\proofcontentschild{lem:restart-concentration}{Lemma~\protect\ref*{lem:restart-concentration}. Simultaneous learning with finite mini-batches}}

We introduce the notation required for the following statement.
Fix a task \(\mathcal D_\ell\), and consider \(n\) consecutive SD updates on
this task starting from \(\mathbf W^{(0)}\).
Let \(N_\ell\mathrel{:=}|E_\ell|\), and enumerate the elements of \(E_\ell\) as
\(e_m=(u_m,v_m)\), \(m\in\{1,\ldots,N_\ell\}\).
Assumption~\ref{ass:transition-compatibility} and the definition of $M$ imply
$N_\ell\leq\sum_{k=1}^K|\partial_\ell^\star C_k|\leq KM$.

Let \(\mathcal S_{\ell,k}\subseteq\mathring C_k\) be the
set of relevant starts in \(C_k\): the task start \(y_{\ell,0}\) when it
belongs to \(C_k\), an endpoint in \(C_k\) of a correct inter-cluster
transition entering \(C_k\).
For \(s\in\mathcal S_{\ell,k}\), write \(U_{\ell,k}^{s}\) for the first
boundary node reached by the rollout started from \(s\), and define
\begin{align}
   \lambda_\ell
   \mathrel{:=}
   \min_{\substack{k:\,\partial_\ell^\star C_k\neq\emptyset\\
                   s\in\mathcal S_{\ell,k},\,
                   u\in\partial_\ell^\star C_k}}
   \mathbb P_{\mathbf W^{(0)}}\left(
      U_{\ell,k}^{s}=u\mid s,x_\ell
   \right).
   \label{eq:routing-lower-bound}
\end{align}
Assumption~\ref{ass:pretrained-diversity} implies the explicit, $K$-uniform bound
\(\lambda_\ell\geq c_{\mathrm{route}}/M>0\).
Lemma~\ref{lem:sd-gradient} shows that SD updates only edges whose sources are hint-source
boundary nodes. Hence all transitions whose sources are interior nodes
remain unchanged, and the first-boundary distribution in
Equation~\ref{eq:routing-lower-bound} is preserved during SD training.

For each \(e_m\), fix a shortest sequence
\(\mathcal Q_m=(C_{k_0},\ldots,C_{k_{L_m}})\) in the cluster-level graph
induced by \(E_\ell\), from the cluster containing \(y_{\ell,0}\) to the
cluster containing \(u_m\).
The sequence can be chosen without repeated clusters, and therefore
\(L_m\leq K-1\).
By Assumption~\ref{ass:pretrained-diversity}, every correct inter-cluster edge has pretrained probability at least
$c_{\mathrm{int}}\epsilon_2$.  We therefore use the $K$-explicit lower bound
\begin{align}
   \underline\rho_{\ell,K, M}
   &:=
   \frac{1}{|\mathcal P_\ell|}
   \left(\frac{c_{\mathrm{route}}}{M}\right)^K
   (c_{\mathrm{int}}\epsilon_2)^{K-1}
   =
   \frac{
      c_{\mathrm{route}}^Kc_{\mathrm{int}}^{K-1}
      \epsilon_2^{K-1}
   }{
      M^K|\mathcal P_\ell|
   }.
   \label{eq:effective-update-probability-MK}
\end{align}
Under \Eqref{eq:uniform-correct-path-sampling}, every path in \(\mathcal P_\ell\) is sampled with probability exactly \(1/|\mathcal P_\ell|\).
The conditional independence properties used below are exactly those specified in the fine-tuning setup above.
Let \(\{\eta_j\}_{j\geq1}\) be a diminishing positive step-size sequence,
and let \(\delta\in(0,1)\).
Define the finite-sample effective drift by
\begin{align}
   \mathfrak D_{\ell,n}^{(K)}(\delta)
   &:={}
   \left[
      \underline\rho_{\ell,K, M}\sum_{j=1}^{n}\eta_j
      -
      \sqrt{
         \frac{\log(KM/\delta)}{2B}
         \sum_{j=1}^{n}\eta_j^2
      }
   \right]^+.
   \label{eq:theorem-one-drift-K}
\end{align}

\begin{proposition}[Learning hint-supported inter-cluster transitions]
\label{prop:sd-transition-learning}
Suppose Assumptions~\ref{ass:transition-compatibility} and~\ref{ass:pretrained-diversity} hold and \(0<\epsilon_2\leq1\).  With
probability at least $1-\delta$, simultaneously for every
$(u_m,v_m)\in E_\ell$,
\begin{align}
   \pi_{\mathbf W^{(n)}}(v_m\mid u_m,x_\ell)
   \geq
   1-O\!\left(
      \left[
         \epsilon_2+
         \epsilon_2^2\left(\log\frac{1}{\epsilon_{\mathrm{rel}}}\right)
         \mathfrak D_{\ell,n}^{(K)}(\delta)
      \right]^{-1}
   \right),
   \label{eq:prop-one-transition-bound}
\end{align}
where $[a]_+:=\max\{a,0\}$.  If in addition
$\eta_j=\Theta(j^{-\nu})$ for some $\nu\in[0,1)$, then for sufficiently large
$n$ the drift term dominates the fluctuation term and
\begin{align}
   \pi_{\mathbf W^{(n)}}(v_m\mid u_m,x_\ell)
   \geq
   1-O\!\left(
      \left[
         \epsilon_2+
         \epsilon_2^2\left(\log\frac{1}{\epsilon_{\mathrm{rel}}}\right)
         \underline\rho_{\ell,K, M} n^{1-\nu}
      \right]^{-1}
   \right).
   \label{eq:prop-one-transition-bound-asymptotic}
\end{align}
The learning term inside the brackets has the explicit $M,K$ dependence
\begin{align*}
   \epsilon_2^2\underline\rho_{\ell,K, M} n^{1-\nu}
   &=
   \frac{
      c_{\mathrm{route}}^K
      c_{\mathrm{int}}^{K-1}
   }{
      M^K|\mathcal P_\ell|
   }
   \epsilon_2^{K+1}n^{1-\nu}.
\end{align*}
\end{proposition}

\paragraph{Sketch of proof of Proposition~\ref{prop:sd-transition-learning}.}
\begin{itemize}
   \item Show that local SD updates increase correct logit gaps and preserve the interior transition kernel (Lemma~\ref{lem:restart-update}).
   \item Use retained boundary coverage to lower-bound the probability of an update opportunity for each correct edge (Lemma~\ref{lem:restart-coverage}).
   \item Control the weighted update counts simultaneously using conditional mini-batch concentration (Lemma~\ref{lem:restart-concentration}).
   \item Convert the weighted update counts into transition-error bounds through the odds estimate (Lemma~\ref{lem:restart-odds}).
\end{itemize}

\paragraph{Shared notation and constants.}
The local update and odds estimates below apply under either Assumptions~\ref{ass:transition-compatibility} and~\ref{ass:pretrained-diversity}
or Assumptions~\ref{ass:transition-compatibility} and~\ref{ass:restart-rare}.
Unless stated otherwise, the SD lemmas concern training
on the fixed task alone.
Throughout the SD lemmas, $C_{\mathrm{bad}}\epsilon_1$ bounds the initial
invalid-boundary probability and $c_{\mathrm{int}}\epsilon_2$ bounds every initial
correct inter-cluster probability from below. The constants are positive and uniform;
we may take $c_{\mathrm{int}}\le1$ and $c_{\mathrm{route}}/M\le1$.
For the rare-entrance case, let $e_m=(u_m,v_m)$ be the edge in
Assumption~\ref{ass:restart-rare}, so $u_m=u^*$.
The first-boundary lower bound is retained at all other correct boundaries;
$c_{\mathrm{route}}$ and $\underline\rho_{\ell,K, M}$ use only this retained lower bound
in the rare case.

Let $\mathcal F_{j-1}$ be the training history before iteration $j$. Define
\begin{align*}
 \tau_i^{(b,j)}&=\inf\{t\ge0:y_t^{(b,j)}=u_i\},&
 I_i^{(b,j)}&=\mathbf1\{e_i\in c^{(b,j)}\}
                   \mathbf1\{\tau_i^{(b,j)}<T^{(b,j)}\},\\
 \mathfrak n_i^{(j)}&=\sum_{b=1}^B I_i^{(b,j)},&
 p_i^{(j)}&=p_{u_i,v_i}^{(j)},\\
 p_{u,v}^{(j)}&=\pi_{\mathbf W^{(j)}}(v\mid u,x_\ell),&
 z_{u,v}^{(j)}&=x_\ell^\top\mathbf W^{(j)}e_{u,v},\\
 r_i^{(j)}&=(1-p_i^{(j)})/p_i^{(j)},&
 a_{i,j}&=\frac{\eta_j}{B}\log(1/\epsilon_{\mathrm{rel}})
             \mathfrak n_i^{(j)}p_i^{(j-1)}.
\end{align*}
The symbols $e_{u,v}$ denote one-hot edge features; $e_i$ denotes a graph edge.
Assumption~\ref{ass:transition-compatibility} gives distinct sources $u_i$ and
$N_\ell=\sum_h|\partial_\ell^\star C_h|\le KM$.

\begin{lemma}[Local update and invariant interior transitions]
\label{lem:restart-update}
For every training history and every outgoing edge of $u_i$,
\begin{equation}
 \mathbf W^{(j)}e_{u_i,v}
 =\mathbf W^{(j-1)}e_{u_i,v}
 +a_{i,j}(\mathbf1\{v=v_i\}-p_{u_i,v}^{(j-1)})x_\ell.
 \label{eq:restart-update}
\end{equation}
Only columns whose sources occur in $E_\ell$ can change. Moreover,
\begin{equation}
 p_i^{(j)}\ge p_i^{(j-1)},\qquad
 \frac{p_{u_i,v}^{(j)}}{p_i^{(j)}}
 \le\frac{p_{u_i,v}^{(j-1)}}{p_i^{(j-1)}}.
 \label{eq:restart-ratio-step}
\end{equation}
First-boundary hitting distributions from relevant starts remain equal to their initial values.
\end{lemma}
\begin{proof}
Lemma~\ref{lem:sd-gradient} gives the visited-source contribution
\[
 (\log\epsilon_{\mathrm{rel}})p_i^{(j-1)}x_\ell
 \left(e_{u_i,v_i}-\sum_{v\in\mathcal N(u_i)}p_{u_i,v}^{(j-1)}e_{u_i,v}\right)^\top.
\]
The hint's normalization term vanishes because $\sum_v\nabla p_{u_i,v}=0$.
At other sources the hinted and unhinted distributions agree, giving zero gradient.
Uniqueness, node simplicity, and mini-batch averaging yield \Eqref{eq:restart-update}.
For $v\ne v_i$, the correct-minus-incorrect logit gap increases by
\[
 a_{i,j}\|x_\ell\|_2^2
 (1-p_i^{(j-1)}+p_{u_i,v}^{(j-1)})\ge0.
\]
Exponentiating proves the ratio inequality, and summing the incorrect-to-correct ratios
proves monotonicity of $p_i^{(j)}$. Interior sources never occur in hints, so all transitions
before the first boundary are unchanged. This also preserves boundary-hitting properness.
\end{proof}

\begin{lemma}[Error from weighted update counts]
\label{lem:restart-odds}
For every history and every $i$,
\begin{equation}
 1-p_i^{(n)}\le
 \left[1+c_{\mathrm{int}}^2\epsilon_2^2\log(1/\epsilon_{\mathrm{rel}})
 \frac1B\sum_{j=1}^n\eta_j\mathfrak n_i^{(j)}\right]^{-1}.
 \label{eq:restart-odds}
\end{equation}
If $p_i^{(0)}<1$, the stronger estimate is
\begin{equation}
r_i^{(n)}\le\left[(r_i^{(0)})^{-1}
+\|x_\ell\|_2^2(p_i^{(0)})^2\log(1/\epsilon_{\mathrm{rel}})
\frac1B\sum_{j=1}^n\eta_j\mathfrak n_i^{(j)}\right]^{-1}.
\label{eq:mf-prop-odds}
\end{equation}
\end{lemma}
\begin{proof}
If $p_i^{(0)}=1$, the assertion is immediate. Otherwise the logit-gap estimate above gives
$r_i^{(j)}\le r_i^{(j-1)}
 \exp[-a_{i,j}\|x_\ell\|_2^2(1-p_i^{(j-1)})]$.
Since $r_i^{(j-1)}\le r_i^{(0)}$ and
$p_i^{(j-1)}(1-p_i^{(j-1)})\ge r_i^{(j-1)}/(1+r_i^{(0)})^2$,
the inequality $e^t\ge1+t$ implies
\[
 \frac1{r_i^{(j)}}\ge\frac1{r_i^{(j-1)}}+
 \frac{\eta_j}{B}\|x_\ell\|_2^2(p_i^{(0)})^2
 \log(1/\epsilon_{\mathrm{rel}})\mathfrak n_i^{(j)}.
\]
Telescoping proves \Eqref{eq:mf-prop-odds}. For the error bound, use $p_i^{(0)}\ge c_{\mathrm{int}}\epsilon_2$ and
$\|x_\ell\|_2^2\ge1$, and write $1-p_i^{(n)}=(1+1/r_i^{(n)})^{-1}$.
\end{proof}

\begin{lemma}[Update opportunities from retained coverage]
\label{lem:restart-coverage}
Under Assumptions~\ref{ass:transition-compatibility} and~\ref{ass:pretrained-diversity}, every $e_i\in E_\ell$ satisfies
\begin{equation}
 \mathbb E[I_i^{(b,j)}\mid\mathcal F_{j-1}]\ge\underline\rho_{\ell,K, M}.
 \label{eq:restart-update-rate}
\end{equation}
Under Assumptions~\ref{ass:transition-compatibility} and~\ref{ass:restart-rare}, the same holds for every $e_i\ne e_m$ whose source cluster is
reachable from the original-start cluster in the cluster graph induced by
$E_\ell\setminus\{e_m\}$.
\end{lemma}
\begin{proof}
Fix an edge in the specified set. Choose a shortest sequence of clusters leading to its
source cluster, using correct inter-cluster edges (excluding $e_m$ in the second case).
The sequence has no repeated cluster and at most $K-1$ transitions. Specify one such edge
between each adjacent pair, and finally specify boundary $u_i$ in the last cluster.
In the second case all specified boundary sources differ from $u_m$, since its unique
correct edge is $e_m$. Also, the sequence cannot enter $C_k$ by Assumption~\ref{ass:restart-rare}.

Every internal segment begins at a relevant start. Its probability of first reaching the
specified boundary is at least $c_{\mathrm{route}}/M$, by retained coverage and
Lemma~\ref{lem:restart-update}. At each required crossing the correct edge probability is
at least $c_{\mathrm{int}}\epsilon_2$, by monotonicity. Apply the Markov property successively
at these boundary and entry times. There are at most $K$ internal segments and $K-1$
crossings, so
\[
 \mathbb P(\tau_i^{(b,j)}<T^{(b,j)}\mid\mathcal F_{j-1})
 \ge (c_{\mathrm{route}}/M)^K(c_{\mathrm{int}}\epsilon_2)^{K-1}.
\]
Both lower-bound factors can be taken at most one. Since $e_i\in E_\ell$, at least one correct path
contains $e_i$, so the uniform hint contains it with probability at least $1/|\mathcal P_\ell|$.
The stipulated conditional independence of hint and student gives \Eqref{eq:restart-update-rate}.
\end{proof}

\begin{lemma}[Simultaneous learning with finite mini-batches]
\label{lem:restart-concentration}
For the applicable edge set in Lemma~\ref{lem:restart-coverage}, with probability at least
$1-\delta$, every such edge satisfies
\begin{equation}
 \frac1B\sum_{j=1}^n\eta_j\mathfrak n_i^{(j)}
 \ge\mathfrak D_{\ell,n}^{(K)}(\delta).
 \label{eq:restart-counts}
\end{equation}
Combining this event with Lemma~\ref{lem:restart-odds} gives simultaneous error bounds for all these edges.
\end{lemma}
\begin{proof}
For a fixed $i$, conditional mini-batch independence and Hoeffding's lemma give, for $t>0$,
\begin{align*}
 &\mathbb E\!\left[\exp\!\left\{-\frac{t\eta_j}{B}\sum_{b=1}^B
 (I_i^{(b,j)}-\mathbb E[I_i^{(b,j)}\mid\mathcal F_{j-1}])\right\}
 \middle|\mathcal F_{j-1}\right]
 \le\exp\!\left(\frac{t^2\eta_j^2}{8B}\right).
\end{align*}
Iterating this bound over $j$ gives the exponent $t^2\sum_j\eta_j^2/(8B)$ for the full sum.
Exponential Markov and optimization over $t$ imply
\begin{align*}
 &\mathbb P\!\left(\frac1B\sum_{j=1}^n\eta_j\sum_{b=1}^B
 [I_i^{(b,j)}-\mathbb E[I_i^{(b,j)}\mid\mathcal F_{j-1}]]
 <-\sqrt{\frac{\log(KM/\delta)}{2B}\sum_{j=1}^n\eta_j^2}\right)
 \le\frac{\delta}{KM}.
\end{align*}
Use Lemma~\ref{lem:restart-coverage} and union-bound over at most $KM$ edges.
Nonnegativity of the weighted counts permits taking the positive part, giving
\Eqref{eq:restart-counts}. Lemma~\ref{lem:restart-odds} gives the error bound. If the edge set is empty, the simultaneous assertion is vacuous.
\end{proof}

\begin{lemma}[Converting edge accuracy to rollout accuracy]
\label{lem:restart-accuracy}
Under Assumptions~\ref{ass:transition-compatibility} and~\ref{ass:pretrained-diversity}, consider any predictor whose interior
kernel equals the pretrained kernel, and denote its edge probabilities by $p_i^{(n)}$.
Then every relevant start $s\in\mathcal S_{\ell,h}$ satisfies
\begin{align}
 &1-\pi_{\mathbf W^{(n)}}(y_T\in\tilde C_\ell\mid y_0=s,x_\ell)
 \le C_{\mathrm{bad}}(1+N_\ell)\epsilon_1+\sum_{i=1}^{N_\ell}(1-p_i^{(n)}).
 \label{eq:restart-accuracy-all}
\end{align}
Under Assumptions~\ref{ass:transition-compatibility} and~\ref{ass:restart-rare}, if the edges specified in Lemma~\ref{lem:restart-coverage}
have error at most $\epsilon_1$, then
\begin{align}
 &1-\pi_{\mathbf W^{(n)}}(y_T\in\tilde C_\ell\mid y_{\ell,0},x_\ell)
 \notag\\
 &\quad\le C_{\mathrm{bad}}(1+KM)\epsilon_1+KM\epsilon_1
 +\pi_{\mathbf W^{(n)}}(\exists t<T:y_t=u_m\mid y_{\ell,0},x_\ell).
 \label{eq:restart-accuracy-rare}
\end{align}
The first inequality is deterministic and does not require the predictor to have
been trained on the fixed task alone.
\end{lemma}
\begin{proof}
For \Eqref{eq:restart-accuracy-all}, follow a rollout until reaching the target or the first
invalid boundary or incorrect choice at a valid boundary. Before these events, each internal
segment starts at a relevant start and is followed by a correct inter-cluster edge.
Conditional on the history up to each segment's start, the invalid-boundary probability
is at most $C_{\mathrm{bad}}\epsilon_1$. By node simplicity there are at most $1+N_\ell$
segments and each correct source is visited at most once. Summing conditional probabilities,
including the events that the segments actually start, gives respectively
$C_{\mathrm{bad}}(1+N_\ell)\epsilon_1$ and $\sum_i(1-p_i^{(n)})$.
An incorrect choice includes choosing an intra-cluster edge at a valid boundary.

These events cover failure: any reached nonterminal cluster has a correct exit, because
the arriving correct edge lies on a correct path with a remaining suffix. The analogous
statement holds for the original start. Properness makes each internal segment finite.
If the target were never reached while both errors were avoided, one would traverse more
than $N_\ell$ distinct correct sources, contradicting node simplicity.
For a start already in the target the conclusion is immediate.

For \Eqref{eq:restart-accuracy-rare}, start from $y_{\ell,0}$ and stop additionally upon visiting
$u_m$. Before that visit every current cluster is reachable via $E_\ell\setminus\{e_m\}$,
so each encountered correct boundary other than $u_m$ belongs to the learned set in
Lemma~\ref{lem:restart-coverage}. Its choice error is at most $\epsilon_1$.
The invalid-boundary term is unchanged, and the extra stopping event is bounded by the
unconditional evaluation visitation probability displayed in \Eqref{eq:restart-accuracy-rare}.
The same properness argument covers failure.
\end{proof}

\paragraph{Main proof of Proposition~\ref{prop:sd-transition-learning}.}
\label{proof:sd-transition-learning}
\begin{proof}
Lemmas~\ref{lem:restart-coverage} and~\ref{lem:restart-concentration} give
$B^{-1}\sum_j\eta_j\mathfrak n_i^{(j)}\ge\mathfrak D_{\ell,n}^{(K)}(\delta)$
simultaneously for all $i$, with probability at least $1-\delta$.
If $p_i^{(0)}=1$, the transition is already correct with probability one.
Otherwise, $(r_i^{(0)})^{-1}=p_i^{(0)}/(1-p_i^{(0)})\ge c_{\mathrm{int}}\epsilon_2$
and $(p_i^{(0)})^2\ge c_{\mathrm{int}}^2\epsilon_2^2$.
The stronger estimate in Lemma~\ref{lem:restart-odds}, together with
$1-p_i^{(n)}\le r_i^{(n)}$, proves \Eqref{eq:prop-one-transition-bound}.
For fixed $0\le\nu<1$, $\sum_{j=1}^n\eta_j=\Theta(n^{1-\nu})$,
whereas $(\sum_{j=1}^n\eta_j^2)^{1/2}$ is respectively
$O(n^{1/2-\nu})$, $O(\sqrt{\log(en)})$, or $O(1)$ as
$\nu<1/2$, $\nu=1/2$, or $\nu>1/2$.
Thus the fluctuation is smaller than the drift for sufficiently large $n$,
which proves \Eqref{eq:prop-one-transition-bound-asymptotic}.
\end{proof}

\addtocontents{toc}{\protect\Needspace{5\protect\baselineskip}}
\subsection{Single-task rollout accuracy under SD}
\label{app:sd-accuracy}
\addtocontents{toc}{\protect\nopagebreak[4]}
\addtocontents{toc}{\protect\proofcontentsentry{proof:sd-correct-rollout}{Theorem~\protect\ref*{thm:sd-correct-rollout}. Single-task learning under SD}}
\addtocontents{toc}{\protect\nopagebreak[4]}
\addtocontents{toc}{\protect\proofcontentschild{prop:sd-transition-learning}{Proposition~\protect\ref*{prop:sd-transition-learning}. Learning hint-supported transitions}}
\addtocontents{toc}{\protect\proofcontentschild{lem:restart-update}{Lemma~\protect\ref*{lem:restart-update}. Local updates and invariant interior transitions}}
\addtocontents{toc}{\protect\proofcontentschild{lem:restart-accuracy}{Lemma~\protect\ref*{lem:restart-accuracy}. From edge accuracy to rollout accuracy}}
\sdrollouttheorem*

\paragraph{Sketch of proof of Theorem~\ref{thm:sd-correct-rollout}.}
\begin{itemize}
   \item Bound the errors of all correct inter-cluster transitions on one common training event (Proposition~\ref{prop:sd-transition-learning}).
   \item Show that the interior transition kernel retains its pre-trained routing guarantees (Lemma~\ref{lem:restart-update}).
   \item Bound rollout failure by the probabilities of invalid-boundary encounters and incorrect transition choices (Lemma~\ref{lem:restart-accuracy}).
\end{itemize}

\paragraph{Main proof of Theorem~\ref{thm:sd-correct-rollout}.}
\label{proof:sd-correct-rollout}
\begin{proof}
Proposition~\ref{prop:sd-transition-learning} gives an event of probability
at least $1-\delta$ on which $1-p_i^{(n)}\le C_{\mathrm{tr}}\xi_{\ell,n}^{(K)}(\delta)$
for every $i$. Lemma~\ref{lem:restart-update} preserves the interior kernel.
Apply Lemma~\ref{lem:restart-accuracy} with $s=y_{\ell,0}$ to obtain
\begin{align*}
1-\pi_{\mathbf W^{(n)}}(y_T\in\tilde C_\ell\mid y_{\ell,0},x_\ell)
&\le C_{\mathrm{bad}}(1+N_\ell)\epsilon_1
+N_\ell C_{\mathrm{tr}}\xi_{\ell,n}^{(K)}(\delta)\\
&=O(KM\epsilon_1)+O(KM\xi_{\ell,n}^{(K)}(\delta)).
\end{align*}
This proves \Eqref{eq:theorem-one-correct-rollout} on that same event.
In the regime where the learning term dominates, substitution of
\Eqref{eq:effective-update-probability-MK} gives
\[
KM\xi_{\ell,n}^{(K)}(\delta)
=O\!\left(\frac{KM^{K+1}|\mathcal P_\ell|}
{c_{\mathrm{route}}^Kc_{\mathrm{int}}^{K-1}\epsilon_2^{K+1}
n^{1-\nu}\log(1/\epsilon_{\mathrm{rel}})}\right).
\]
\end{proof}

\addtocontents{toc}{\protect\Needspace{10\protect\baselineskip}}
\subsection{Single-task learning under SFT}
\label{app:sft-learning}
\addtocontents{toc}{\protect\nopagebreak[4]}
\addtocontents{toc}{\protect\proofcontentsentry{proof:sft-single-task}{Theorem~\protect\ref*{thm:sft-single-task}. Single-task learning under SFT}}
\addtocontents{toc}{\protect\nopagebreak[4]}
\addtocontents{toc}{\protect\proofcontentschild{lem:sft-gradient}{Lemma~\protect\ref*{lem:sft-gradient}. Sampled SFT gradient and update support}}
\addtocontents{toc}{\protect\proofcontentschild{lem:sft-potential}{Lemma~\protect\ref*{lem:sft-potential}. Convex potential, drift, and support error}}
\addtocontents{toc}{\protect\proofcontentschild{lem:sft-comparator}{Lemma~\protect\ref*{lem:sft-comparator}. A finite comparator}}
\addtocontents{toc}{\protect\proofcontentschild{lem:sft-noise}{Lemma~\protect\ref*{lem:sft-noise}. Conditional mini-batch noise}}
\addtocontents{toc}{\protect\proofcontentschild{lem:sft-smoothness}{Lemma~\protect\ref*{lem:sft-smoothness}. Smoothness and the initial transient}}
\addtocontents{toc}{\protect\proofcontentschild{lem:sft-last-iterate}{Lemma~\protect\ref*{lem:sft-last-iterate}. A finite-comparator last-iterate bound}}
\addtocontents{toc}{\protect\proofcontentschild{lem:sft-step-sums}{Lemma~\protect\ref*{lem:sft-step-sums}. Polynomial step-size sums}}
\addtocontents{toc}{\protect\proofcontentschild{lem:sft-support-learning}{Lemma~\protect\ref*{lem:sft-support-learning}. Simultaneous successor-set learning}}

Set
\begin{align*}
   m_\nu
   &:=
   \min\{\nu,1-\nu\}.
\end{align*}
We prove the stronger batch-explicit bound
\begin{align}
   \min_{u\in\mathcal U_\ell}
   \pi_{\mathbf W_{\mathrm{SFT}}^{(n)}}
   (\mathcal A_\ell(u)\mid u,x_\ell)
   \geq
   1-C_\ell |\mathcal P_\ell|
   \left[
      \frac1n
      +
      \frac{\log^2(e+n|\mathcal P_\ell|)}{n^{1-\nu}}
      +
      \frac{
         (1+2\log(2/\delta))\log(en)
      }{
         B\,n^{m_\nu}
      }
   \right],
   \label{eq:theorem-two-batch-explicit}
\end{align}
where $C_\ell$ is independent of $n,B,\delta$.

Let $\Gamma$ be one teacher path.  Define
\begin{align}
   \omega_{\ell,u}
   :=
   \mathbb E[N_u(\Gamma)],\quad
   r_{\ell,u}(v)
   :=
   \frac{\mathbb E[N_{u,v}(\Gamma)]}{\omega_{\ell,u}},
   \qquad
   v\in\mathcal A_\ell(u).
   \label{eq:omega-r}
\end{align}
Then
\begin{align}
   \omega_{\ell,u}
   &\geq
   \frac{1}{|\mathcal P_\ell|}
   ,\quad
   r_{\ell,u}(v)>0.
   \label{eq:visitation-lower}
\end{align}

\sftlearningtheorem*

\paragraph{Sketch of proof of Theorem~\ref{thm:sft-single-task}.}
\begin{itemize}
   \item Derive the sampled SFT gradient and identify a convex potential whose gradient is its conditional expectation (Lemmas~\ref{lem:sft-gradient} and~\ref{lem:sft-potential}).
   \item Construct a finite comparator with small potential and a controlled distance from the initial parameter (Lemma~\ref{lem:sft-comparator}).
   \item Verify the conditional mini-batch noise and smoothness bounds, including the bounded initial transient (Lemmas~\ref{lem:sft-noise} and~\ref{lem:sft-smoothness}).
   \item Apply the finite-comparator last-iterate bound and evaluate the polynomial step-size sums (Lemmas~\ref{lem:sft-last-iterate} and~\ref{lem:sft-step-sums}).
   \item Convert the resulting potential bound into successor-set accuracy simultaneously at every teacher source (Lemmas~\ref{lem:sft-potential} and~\ref{lem:sft-support-learning}).
\end{itemize}

At iteration $j$,
\begin{gather}
   \color{black}{
   L_{\mathrm{SFT}}^{(j)}(\mathbf W)
   :=
   -\frac1B
   \sum_{b=1}^{B}
   \sum_{t=0}^{T^{(b,j)}-1}
   \log
   \pi_{\mathbf W}
   (y_{t+1}^{*(b,j)}\mid y_t^{*(b,j)},x_\ell),}
   \label{eq:minibatch-loss}\\
   \color{black}{
   \mathbf g_j
   :=
   \nabla_{\mathbf W}
   L_{\mathrm{SFT}}^{(j)}
   (\mathbf W_{\mathrm{SFT}}^{(j-1)}),}
   \notag\\
   \color{black}{
   \mathbf W_{\mathrm{SFT}}^{(j)}
   :=
   \mathbf W_{\mathrm{SFT}}^{(j-1)}
   -
   \eta_j\mathbf g_j.}
   \label{eq:minibatch-sgd-recursion}
\end{gather}
\begin{lemma}[Exact sampled SFT gradient and update support]
\label{lem:sft-gradient}
The mini-batch update \Eqref{eq:minibatch-sgd-recursion} has gradient
\Eqref{eq:exact-minibatch-gradient}, whose Frobenius norm is at most
$G_\ell=\sqrt{2d(|V|-1)}$. Every update lies in $\operatorname{span}\{x_\ell\}$
within each source block, and only sources in $\mathcal U_\ell$ can change.
\end{lemma}
\begin{proof}
Direct differentiation gives
\begin{align}
   \mathbf g_j
   =
   \frac1B
   \sum_{b=1}^{B}
   \sum_{u\in\mathcal U_\ell}
   x_\ell
   \Biggl(
      N_u(\Gamma_j^{(b)})
      \sum_{v\in\mathcal N(u)}
      \pi_{\mathbf W_{\mathrm{SFT}}^{(j-1)}}
      (v\mid u,x_\ell)e_{u,v}
      -
      \sum_{v\in\mathcal A_\ell(u)}
      N_{u,v}(\Gamma_j^{(b)})e_{u,v}
   \Biggr)^{\!\top}.
   \label{eq:exact-minibatch-gradient}
\end{align}

For one path $\Gamma$, let $\mathbf h(\Gamma;\mathbf W)$ denote the
corresponding single-path gradient.  If $v_\Gamma(u)$ is the successor of
the visited source $u$, then
\begin{align*}
   \left\|
      \sum_{v\in\mathcal N(u)}
      \pi_{\mathbf W}(v\mid u,x_\ell)e_{u,v}
      -
      e_{u,v_\Gamma(u)}
   \right\|_2^2
   &=
   \sum_{v\in\mathcal N(u)}
   \pi_{\mathbf W}(v\mid u,x_\ell)^2
   +
   1
   -
   2\pi_{\mathbf W}(v_\Gamma(u)\mid u,x_\ell)
   \\
   &\leq 2.
\end{align*}
Since the source-node edge blocks are disjoint,
\begin{align*}
   \|\mathbf h(\Gamma;\mathbf W)\|_{\text{F}}^2
   &\leq
   2\|x_\ell\|_2^2(|\Gamma| - 1)
   \leq
   2d(|V|-1).
\end{align*}
Hence
\begin{align}
   \|\mathbf g_j\|_{\text{F}}
   \leq
   G_\ell
   =
   \sqrt{2d(|V|-1)}.
   \label{eq:gradient-bound}
\end{align}
\end{proof}

Define
\begin{align*}
   P_{\ell,u}(\mathbf W)
   &:=
   \pi_{\mathbf W}
   (\mathcal A_\ell(u)\mid u,x_\ell),
\end{align*}
and
\begin{align}
   F_\ell(\mathbf W)
   &:=
   \sum_{u\in\mathcal U_\ell}
   \omega_{\ell,u}
   \left[
      \KL\left(
         r_{\ell,u}
         \middle\|
         \frac{
            \pi_{\mathbf W}(\cdot\mid u,x_\ell)
         }{
            P_{\ell,u}(\mathbf W)
         }
      \right)
      -
      \log P_{\ell,u}(\mathbf W)
   \right].
   \label{eq:analysis-potential}
\end{align}
\begin{lemma}[Convex potential, unbiased drift, and support error]
\label{lem:sft-potential}
The potential $F_\ell$ in \Eqref{eq:analysis-potential} is nonnegative and convex,
the sampled gradient satisfies \Eqref{eq:conditional-drift}, and
$\max_{u\in\mathcal U_\ell}[1-P_{\ell,u}(\mathbf W)]\le |\mathcal P_\ell| F_\ell(\mathbf W)$
for every finite $\mathbf W$.
\end{lemma}
\begin{proof}
For every $u$,
\begin{align*}
   \color{black}{
   \KL\left(
      r_{\ell,u}
      \middle\|
      \frac{\pi_{\mathbf W}}{P_{\ell,u}}
   \right)
   -
   \log P_{\ell,u}
   =
   \sum_{v\in\mathcal A_\ell(u)}
   r_{\ell,u}(v)\log r_{\ell,u}(v)
   -
   \sum_{v\in\mathcal A_\ell(u)}
   r_{\ell,u}(v)
   \log\pi_{\mathbf W}(v\mid u,x_\ell).}
\end{align*}
Therefore
\begin{align}
   F_\ell(\mathbf W)
   =
   C_\ell'
   -
   \sum_{u\in\mathcal U_\ell}
   \sum_{v\in\mathcal A_\ell(u)}
   \mathbb E[N_{u,v}(\Gamma)]
   \log\pi_{\mathbf W}(v\mid u,x_\ell),
   \label{eq:potential-crossentropy-form}
\end{align}
so $F_\ell$ is convex.  We next verify the conditional-drift identity explicitly.
Let $\mathcal F_{j-1}$ be the sigma-field generated by all training randomness up to
iteration $j-1$.  Conditional on $\mathcal F_{j-1}$, the parameter
$\mathbf W_{\mathrm{SFT}}^{(j-1)}$ is fixed, while
$\Gamma_j^{(1)},\ldots,\Gamma_j^{(B)}$ are conditionally independent and each is
uniformly distributed on $\mathcal P_\ell$.  Therefore, for every
$u\in\mathcal U_\ell$, $v\in\mathcal A_\ell(u)$, and $b\in\{1,\ldots,B\}$,
\begin{align*}
   \mathbb E\!\left[
      N_u(\Gamma_j^{(b)})
      \mid \mathcal F_{j-1}
   \right]
   &=
   \mathbb E[N_u(\Gamma)]
   =
   \omega_{\ell,u},
   &&\text{by conditional uniform sampling},
   \\
   \mathbb E\!\left[
      N_{u,v}(\Gamma_j^{(b)})
      \mid \mathcal F_{j-1}
   \right]
   &=
   \mathbb E[N_{u,v}(\Gamma)]
   =
   \omega_{\ell,u}r_{\ell,u}(v),
   &&\text{by Equation~\ref{eq:omega-r}}.
\end{align*}
Moreover, for every correct teacher path $\Gamma$,
\begin{align*}
   \color{black}{
   N_u(\Gamma)
   =
   \sum_{v\in\mathcal A_\ell(u)}N_{u,v}(\Gamma),
   \quad\text{each source occurrence has one correct successor}.}
\end{align*}
Taking expectations gives
\begin{align*}
   \omega_{\ell,u}
   &=
   \sum_{v\in\mathcal A_\ell(u)}
   \mathbb E[N_{u,v}(\Gamma)].
\end{align*}
Now apply conditional expectation to Equation~\ref{eq:exact-minibatch-gradient}.
Write $\mathbf W_{j-1}:=\mathbf W_{\mathrm{SFT}}^{(j-1)}$ only within the
following calculation.  Since $\mathbf W_{j-1}$ is
$\mathcal F_{j-1}$-measurable, the softmax probabilities are fixed under the
conditional expectation.  Substituting the two conditional-expectation
identities above into Equation~\ref{eq:exact-minibatch-gradient} gives
\begin{align*}
   \mathbb E\!\left[
      \mathbf g_j\mid\mathcal F_{j-1}
   \right]
   &=
   \frac1B\sum_{b=1}^{B}
   \sum_{u\in\mathcal U_\ell}
   x_\ell
   \Biggl(
      \omega_{\ell,u}
      \sum_{z\in\mathcal N(u)}
      \pi_{\mathbf W_{j-1}}(z\mid u,x_\ell)e_{u,z}
      -\sum_{v\in\mathcal A_\ell(u)}
      \mathbb E[N_{u,v}(\Gamma)]e_{u,v}
   \Biggr)^{\!\top}
   \\
   &=
   \sum_{u\in\mathcal U_\ell}
   x_\ell
   \Biggl(
      \omega_{\ell,u}
      \sum_{z\in\mathcal N(u)}
      \pi_{\mathbf W_{j-1}}(z\mid u,x_\ell)e_{u,z}
      -\sum_{v\in\mathcal A_\ell(u)}
      \mathbb E[N_{u,v}(\Gamma)]e_{u,v}
   \Biggr)^{\!\top}.
\end{align*}
The second equality holds because the summand no longer depends on $b$, so the
factor $B^{-1}$ cancels with the sum over the $B$ mini-batch elements.

We next differentiate the potential.  Equation~\ref{eq:potential-crossentropy-form}
gives
\begin{align*}
   \nabla F_\ell(\mathbf W)
   &=
   -\sum_{u\in\mathcal U_\ell}
   \sum_{v\in\mathcal A_\ell(u)}
   \mathbb E[N_{u,v}(\Gamma)]
   \nabla_{\mathbf W}
   \log\pi_{\mathbf W}(v\mid u,x_\ell),
\end{align*}
because $C_\ell'$ is independent of $\mathbf W$.  For the softmax predictor,
\begin{align*}
   -\nabla_{\mathbf W}
   \log\pi_{\mathbf W}(v\mid u,x_\ell)
   &=
   x_\ell
   \left(
      \sum_{z\in\mathcal N(u)}
      \pi_{\mathbf W}(z\mid u,x_\ell)e_{u,z}
      -e_{u,v}
   \right)^{\!\top}.
\end{align*}
Substituting this identity and collecting the terms with the same source $u$
yields
\begin{align*}
   \color{black}{
   \nabla F_\ell(\mathbf W)
   =
   \sum_{u\in\mathcal U_\ell}
   x_\ell
   \Biggl(
      \omega_{\ell,u}
      \sum_{z\in\mathcal N(u)}
      \pi_{\mathbf W}(z\mid u,x_\ell)e_{u,z}
      -\sum_{v\in\mathcal A_\ell(u)}
      \mathbb E[N_{u,v}(\Gamma)]e_{u,v}
   \Biggr)^{\!\top}.}
\end{align*}
Here we used
$\omega_{\ell,u}=\sum_{v\in\mathcal A_\ell(u)}\mathbb E[N_{u,v}(\Gamma)]$,
which follows from the preceding path-count identity.
Evaluating this identity at
$\mathbf W=\mathbf W_{\mathrm{SFT}}^{(j-1)}$ proves
\begin{align}
   \mathbb E[
      \mathbf g_j
      \mid
      \mathcal F_{j-1}
   ]
   &=
   \nabla F_\ell(
      \mathbf W_{\mathrm{SFT}}^{(j-1)}
   ).
   \label{eq:conditional-drift}
\end{align}
Using
\begin{align*}
   \KL(\cdot\|\cdot)\geq0,
   \quad-\log z\geq1-z,
\end{align*}
we obtain, for every $u$,
\begin{align*}
   F_\ell(\mathbf W)
   &\geq
   \omega_{\ell,u}
   [-\log P_{\ell,u}(\mathbf W)]
   \geq
   \frac{1}{|\mathcal P_\ell|}
   [1-P_{\ell,u}(\mathbf W)].
\end{align*}
Thus
\begin{align}
   \max_{u\in\mathcal U_\ell}
   [1-P_{\ell,u}(\mathbf W)]
   &\leq
   |\mathcal{P}_{\ell}| F_\ell(\mathbf W).
   \label{eq:support-reduction}
\end{align}
\end{proof}

\begin{lemma}[A finite comparator approaching the optimal support]
\label{lem:sft-comparator}
For every $R>0$, there is a finite $\mathbf W_R$ satisfying
$F_\ell(\mathbf W_R)\le C_0e^{-R}$ and
$\|\mathbf W_R-\mathbf W^{(0)}\|_{\mathrm F}^2
\le C_1(1+R+\log |\mathcal P_\ell|)^2$, with constants independent of $R,n,B,\delta$.
\end{lemma}
\begin{proof}
For $R>0$, set
\begin{align}
   z_{u,v}^{(R)}
   =
   \begin{cases}
      R+\log r_{\ell,u}(v),
      &v\in\mathcal A_\ell(u),
      \\
      0,
      &v\notin\mathcal A_\ell(u).
   \end{cases}
   \label{eq:comparator-logits}
\end{align}
Let
\begin{align*}
   m_{\ell,u}^{-}
   &:=
   |\mathcal N(u)\setminus\mathcal A_\ell(u)|.
\end{align*}
Then
\begin{gather*}
   \color{black}{
   \sum_{v\in\mathcal A_\ell(u)}
   e^{z_{u,v}^{(R)}}
   =
   e^R,}
   \\
   \color{black}{
   \sum_{v\notin\mathcal A_\ell(u)}
   e^{z_{u,v}^{(R)}}
   =
   m_{\ell,u}^{-},}
   \\
   \color{black}{
   P_{\ell,u}(\mathbf W_R)
   =
   \frac{e^R}{e^R+m_{\ell,u}^{-}}
   =
   \frac1{1+m_{\ell,u}^{-}e^{-R}},}
   \\
   \color{black}{
   \frac{
      \pi_{\mathbf W_R}(v\mid u,x_\ell)
   }{
      P_{\ell,u}(\mathbf W_R)
   }
   =
   r_{\ell,u}(v),
   \qquad
   v\in\mathcal A_\ell(u).}
\end{gather*}
Hence
\begin{align}
   F_\ell(\mathbf W_R)
   &=
   \sum_{u\in\mathcal U_\ell}
   \omega_{\ell,u}
   \log(1+m_{\ell,u}^{-}e^{-R})
   \notag\\
   &\leq
   e^{-R}
   \sum_{u\in\mathcal U_\ell}
   \omega_{\ell,u}m_{\ell,u}^{-}
   \leq
   |E|e^{-R}
   =:
   C_0e^{-R}.
   \label{eq:comparator-potential}
\end{align}

Set
\begin{align*}
   z_{u,v}^{(0)}
   &:=
   x_\ell^\top
   \mathbf W^{(0)}
   e_{u,v},
\end{align*}
and, for $u\in\mathcal U_\ell$, $v\in\mathcal N(u)$,
\begin{align*}
   \mathbf W_Re_{u,v}
   &:=
   \mathbf W^{(0)}e_{u,v}
   +
   \frac{
      z_{u,v}^{(R)}-z_{u,v}^{(0)}
   }{
      \|x_\ell\|_2^2
   }x_\ell,
\end{align*}
with all other columns unchanged.  Then
\begin{align*}
   \color{black}{
   x_\ell^\top\mathbf W_Re_{u,v}
   =
   z_{u,v}^{(R)}, \quad
   \|\mathbf W_R-\mathbf W^{(0)}\|_{\text{F}}^2
   =
   \frac1{\|x_\ell\|_2^2}
   \sum_{u\in\mathcal U_\ell}
   \sum_{v\in\mathcal N(u)}
   |z_{u,v}^{(R)}-z_{u,v}^{(0)}|^2.}
\end{align*}
Since
\begin{gather*}
   \|x_\ell\|_2^2\geq1,
   \\
   r_{\ell,u}(v)
   \geq
   \frac{1}{|\mathcal{P}_{\ell}|}
   \quad\Longrightarrow\quad
   |z_{u,v}^{(R)}|
   \leq
   R+\log |\mathcal P_\ell|,
   \\
   Z_{\ell,0}
   :=
   \max_{\substack{u\in\mathcal U_\ell\\v\in\mathcal N(u)}}
   |z_{u,v}^{(0)}|
   <\infty,
\end{gather*}
we obtain
\begin{align}
   \|\mathbf W_R-\mathbf W^{(0)}\|_{\text{F}}^2
   &\leq
   |E|
   (R+\log |\mathcal P_\ell|+Z_{\ell,0})^2
   \leq
   C_1(1+R+\log |\mathcal P_\ell|)^2.
   \label{eq:comparator-distance}
\end{align}
\end{proof}

Write
\begin{align*}
   \mathbf g_j
   =
   \frac1B
   \sum_{b=1}^{B}
   \mathbf h_j^{(b)},
   \quad
   \boldsymbol\zeta_j
   :=
   \mathbf g_j
   -
   \nabla F_\ell(
      \mathbf W_{\mathrm{SFT}}^{(j-1)}
   ),
\end{align*}
where
\begin{gather*}
   \color{black}{
   \mathbf h_j^{(b)}
   :=
   \mathbf h(
      \Gamma_j^{(b)};
      \mathbf W_{\mathrm{SFT}}^{(j-1)}
   ),}
   \\
   \color{black}{
   \boldsymbol\xi_j^{(b)}
   :=
   \mathbf h_j^{(b)}
   -
   \mathbb E[
      \mathbf h_j^{(b)}
      \mid
      \mathcal F_{j-1}
   ].}
\end{gather*}
Then
\begin{align*}
   \boldsymbol\zeta_j
   =
   \frac1B
   \sum_{b=1}^{B}
   \boldsymbol\xi_j^{(b)},
   \quad
   \mathbb E[
      \boldsymbol\zeta_j
      \mid
      \mathcal F_{j-1}
   ]=0.
\end{align*}
\begin{lemma}[Conditional mini-batch noise]
\label{lem:sft-noise}
The centered noise $\boldsymbol\zeta_j$ has conditional mean zero and variance
at most $G_\ell^2/B$. There is $\sigma_{\ell,B}^2=O_G(B^{-1})$ such that
\Eqref{eq:minibatch-noise-subgaussian} holds for
$0\le\lambda\le\sigma_{\ell,B}^{-2}$.
\end{lemma}
\begin{proof}
Conditional independence gives
\begin{align}
   \mathbb E[
      \|\boldsymbol\zeta_j\|_{\text{F}}^2
      \mid
      \mathcal F_{j-1}
   ]
   &=
   \frac1{B^2}
   \sum_{b=1}^{B}
   \mathbb E[
      \|\boldsymbol\xi_j^{(b)}\|_{\text{F}}^2
      \mid
      \mathcal F_{j-1}
   ]
   \leq
   \frac{G_\ell^2}{B}.
   \label{eq:minibatch-noise-variance}
\end{align}

Let $D:=d|E|$.  Since
\begin{align*}
   \|\boldsymbol\xi_j^{(b)}\|_{\text{F}}
   &\leq
   2G_\ell,
\end{align*}
coordinate-wise Hoeffding gives, for every $t>0$,
\begin{align*}
   \mathbb P\left(
      \|\boldsymbol\zeta_j\|_{\text{F}}
      \geq t
      \mid
      \mathcal F_{j-1}
   \right)
   &\leq
   2D
   \exp\left(
      -\frac{Bt^2}{8DG_\ell^2}
   \right).
\end{align*}
Hence, for a constant $c_D>0$ depending only on $D$,
\begin{gather*}
   \sigma_{\ell,B}^2
   :=
   \frac{c_DG_\ell^2}{B},
   \\
   \mathbb E\left[
      \exp\left(
         \frac{
            \|\boldsymbol\zeta_j\|_{\text{F}}^2
         }{
            \sigma_{\ell,B}^2
         }
      \right)
      \middle|
      \mathcal F_{j-1}
   \right]
   \leq e.
\end{gather*}
For $0\leq\lambda\leq\sigma_{\ell,B}^{-2}$,
put $\theta:=\lambda\sigma_{\ell,B}^2\in[0,1]$.  Then
\begin{align}
   \mathbb E\left[
      e^{
         \lambda
         \|\boldsymbol\zeta_j\|_{\text{F}}^2
      }
      \middle|
      \mathcal F_{j-1}
   \right]
   &=
   \mathbb E\left[
      \left(
         e^{
            \|\boldsymbol\zeta_j\|_{\text{F}}^2/
            \sigma_{\ell,B}^2
         }
      \right)^\theta
      \middle|
      \mathcal F_{j-1}
   \right]
   \notag\\
   &\leq
   \left(
      \mathbb E\left[
         e^{
            \|\boldsymbol\zeta_j\|_{\text{F}}^2/
            \sigma_{\ell,B}^2
         }
         \middle|
         \mathcal F_{j-1}
      \right]
   \right)^\theta
   \notag\\
   &\leq
   e^{\lambda\sigma_{\ell,B}^2},
   \label{eq:minibatch-noise-subgaussian}
\end{align}
and
\begin{align*}
   \sigma_{\ell,B}^2
   &=
   O_G(B^{-1}).
\end{align*}
\end{proof}

\begin{lemma}[Smoothness and a bounded initial transient]
\label{lem:sft-smoothness}
The potential is $L_\ell$-smooth with $L_\ell\le d$.
For $0<\nu<1$, there is a deterministic $j_0$, independent of $n,B,\delta$, after
which $\eta_j\le1/(2L_\ell)$. The finite comparator of
Lemma~\ref{lem:sft-comparator} satisfies \Eqref{eq:comparator-distance-shifted}
uniformly over the training history before $j_0$.
\end{lemma}
\begin{proof}
For a perturbation $H$, let
\begin{align*}
   a_u(H)
   :=
   x_\ell^\top H e_{u,\cdot},\quad
   p_u
   :=
   \pi_{\mathbf W}(\cdot\mid u,x_\ell).
\end{align*}
From Equation~(\ref{eq:potential-crossentropy-form}),
\begin{align*}
   \nabla^2F_\ell(\mathbf W)[H,H]
   &=
   \sum_{u\in\mathcal U_\ell}
   \omega_{\ell,u}
   a_u(H)^\top
   \left(
      \operatorname{Diag}(p_u)-p_up_u^\top
   \right)
   a_u(H)
   \\
   &\leq
   \sum_{u\in\mathcal U_\ell}
   \omega_{\ell,u}\|a_u(H)\|_2^2
   \leq
   \|x_\ell\|_2^2
   \|H\|_{\text{F}}^2
   \leq
   d\|H\|_{\text{F}}^2.
\end{align*}
Thus
\begin{align}
   L_\ell
   &\leq d.
   \label{eq:potential-smoothness}
\end{align}

Since $\eta_j=\Theta(j^{-\nu})$, there exists a fixed $j_0$ such that
\begin{align*}
   \eta_j
   &\leq
   \frac1{2L_\ell},
   \qquad
   j\geq j_0.
\end{align*}
For $n<j_0$, Equation~(\ref{eq:theorem-two-batch-explicit}) is made
trivial by enlarging $C_\ell$.  Hence assume $n\geq j_0$.  Also,
\begin{align*}
   \|\mathbf W_{\mathrm{SFT}}^{(j_0-1)}-\mathbf W_R\|_{\text{F}}
   &\leq
   \|\mathbf W^{(0)}-\mathbf W_R\|_{\text{F}}
   +
   G_\ell
   \sum_{j=1}^{j_0-1}\eta_j,
\end{align*}
so
\begin{align}
   \|\mathbf W_{\mathrm{SFT}}^{(j_0-1)}-\mathbf W_R\|_{\text{F}}^2
   &\leq
   C_2(1+R+\log |\mathcal P_\ell|)^2.
   \label{eq:comparator-distance-shifted}
\end{align}
\end{proof}

\begin{lemma}[A finite-comparator last-iterate bound]
\label{lem:sft-last-iterate}
For every deterministic finite comparator $\mathbf W_R$ and $n\ge j_0$,
the bound \Eqref{eq:minibatch-smooth-last-iterate} holds with probability
at least $1-\delta$. Its constants are independent of $n,B,\delta,R$.
\end{lemma}
\begin{proof}
We next verify explicitly that the last-iterate argument does not require
$F_\ell$ to attain its infimum at a finite parameter.  Fix the finite
comparator $\mathbf W_R$ constructed above.  Equation~\ref{eq:minibatch-sgd-recursion} gives
\begin{align}
   2\eta_j
   \left\langle
      \mathbf g_j,
      \mathbf W_{\mathrm{SFT}}^{(j-1)}-\mathbf W_R
   \right\rangle
   =
   \|\mathbf W_{\mathrm{SFT}}^{(j-1)}-\mathbf W_R\|_{\text{F}}^2
   -
   \|\mathbf W_{\mathrm{SFT}}^{(j)}-\mathbf W_R\|_{\text{F}}^2
   +
   \eta_j^2\|\mathbf g_j\|_{\text{F}}^2.
   \label{eq:finite-comparator-euclidean}
\end{align}
Moreover, convexity of $F_\ell$ and the definition of
$\boldsymbol\zeta_j$ imply
\begin{align}
   F_\ell(\mathbf W_{\mathrm{SFT}}^{(j-1)})-F_\ell(\mathbf W_R)
   &\leq
   \left\langle
      \nabla F_\ell(\mathbf W_{\mathrm{SFT}}^{(j-1)}),
      \mathbf W_{\mathrm{SFT}}^{(j-1)}-\mathbf W_R
   \right\rangle
   \notag\\
   &=
   \left\langle
      \mathbf g_j,
      \mathbf W_{\mathrm{SFT}}^{(j-1)}-\mathbf W_R
   \right\rangle
   -
   \left\langle
      \boldsymbol\zeta_j,
      \mathbf W_{\mathrm{SFT}}^{(j-1)}-\mathbf W_R
   \right\rangle.
   \label{eq:finite-comparator-convexity}
\end{align}
Equations~\ref{eq:finite-comparator-euclidean} and
\ref{eq:finite-comparator-convexity} are the comparator-dependent one-step
relations used in the Euclidean convex specialization of Lemma~4.3 of
\citet{liu2024lastiterate}.  The subsequent weighted telescoping and
martingale argument uses the reference point only through these relations and
the finite initial Euclidean distance.  Therefore, using
Equation~\ref{eq:minibatch-noise-subgaussian}, the same argument applies with
this arbitrary finite comparator $\mathbf W_R$ and does not require a finite
minimizer of $F_\ell$.

Specializing that argument to
\begin{align*}
   \psi(\mathbf W)
   :=
   \frac12\|\mathbf W\|_{\text{F}}^2,\quad
   \mu_f=\mu_h=M=0,\quad
   \gamma_j=\eta_j,
\end{align*}
thus gives, for every finite $\mathbf W_R$, with probability at least
$1-\delta$,
\begin{align}
   F_\ell(
      \mathbf W_{\mathrm{SFT}}^{(n)}
   )
   -
   F_\ell(\mathbf W_R)
   \leq
   C_3
   \Biggl[
      &
      \frac{
         \|\mathbf W_{\mathrm{SFT}}^{(j_0-1)}
         -
         \mathbf W_R\|_{\text{F}}^2
      }{
         \sum_{j=j_0}^{n}\eta_j
      }
      +
      \sigma_{\ell,B}^2
      \left(
         1+2\log\frac2\delta
      \right)
      \sum_{j=j_0}^{n}
      \frac{
         \eta_j^2
      }{
         \sum_{s=j}^{n}\eta_s
      }
   \Biggr].
   \label{eq:minibatch-smooth-last-iterate}
\end{align}
\end{proof}

\begin{lemma}[Polynomial step-size sums]
\label{lem:sft-step-sums}
For fixed $0<\nu<1$ and $\eta_j=\Theta(j^{-\nu})$,
$\sum_{j=j_0}^n\eta_j=\Theta(n^{1-\nu})$ for sufficiently large $n$ and
the weighted noise sum satisfies \Eqref{eq:minibatch-step-size-rates}.
\end{lemma}
\begin{proof}
\begin{align*}
   \sum_{j=j_0}^{n}\eta_j
   &=
   \Theta(n^{1-\nu}).
\end{align*}
Moreover, with
\begin{align*}
   S_n
   &:=
   \sum_{j=j_0}^{n}
   \frac{
      \eta_j^2
   }{
      \sum_{s=j}^{n}\eta_s
   }
   =
   S_n^{(1)}+S_n^{(2)},
\end{align*}
where
\begin{align*}
   S_n^{(1)}
   &:=
   \sum_{j=j_0}^{\lfloor n/2\rfloor}
   \frac{
      \eta_j^2
   }{
      \sum_{s=j}^{n}\eta_s
   },
   \qquad S_n^{(2)}
   :=
   \sum_{j=\lfloor n/2\rfloor+1}^{n}
   \frac{
      \eta_j^2
   }{
      \sum_{s=j}^{n}\eta_s
   },
\end{align*}
we have
\begin{align*}
   S_n^{(1)}
   &\leq
   \frac{C}{n^{1-\nu}}
   \sum_{j=1}^{\lfloor n/2\rfloor}
   j^{-2\nu}
   \\
   &=
   \begin{cases}
      O(n^{-\nu}),
      &0<\nu<1/2,
      \\
      O(\log(en)n^{-1/2}),
      &\nu=1/2,
      \\
      O(n^{-(1-\nu)}),
      &1/2<\nu<1,
   \end{cases}
\end{align*}
and, for $k=n-j+1$,
\begin{align*}
   S_n^{(2)}
   &\leq
   Cn^{-\nu}
   \sum_{k=1}^{\lceil n/2\rceil}
   \frac1k
   =
   O(n^{-\nu}\log(en)).
\end{align*}
Therefore
\begin{align}
   S_n
   &=
   O\left(
      \frac{
         \log(en)
      }{
         n^{m_\nu}
      }
   \right).
   \label{eq:minibatch-step-size-rates}
\end{align}
\end{proof}

\begin{lemma}[Simultaneous SFT successor-set learning]
\label{lem:sft-support-learning}
For any finite initial parameter and fixed $n\ge1$, the batch-explicit bound
\Eqref{eq:theorem-two-batch-explicit} holds with probability at least $1-\delta$.
No pretrained coverage condition is needed for this assertion.
\end{lemma}
\begin{proof}
Apply Lemmas~\ref{lem:sft-comparator}--\ref{lem:sft-step-sums}.
Combining
Equations~(\ref{eq:comparator-potential}),
(\ref{eq:minibatch-noise-subgaussian}),
(\ref{eq:comparator-distance-shifted}),
(\ref{eq:minibatch-smooth-last-iterate}),
and~(\ref{eq:minibatch-step-size-rates}),
\begin{align*}
   F_\ell(
      \mathbf W_{\mathrm{SFT}}^{(n)}
   )
   \leq
   C_4
   \left[
      e^{-R}
      +
      \frac{
         (1+R+\log |\mathcal P_\ell|)^2
      }{
         n^{1-\nu}
      }
      +
      \frac{
         (1+2\log(2/\delta))\log(en)
      }{
         B\,n^{m_\nu}
      }
   \right].
\end{align*}
Choose
\begin{align}
   R
   &:=
   \log(e+n|\mathcal P_\ell|).
   \label{eq:minibatch-R-choice}
\end{align}
Then
\begin{gather*}
   e^{-R}
   =
   \frac1{e+n|\mathcal P_\ell|}
   \leq
   \frac1n,
   \qquad 1+R+\log |\mathcal P_\ell|
   \leq
   3\log(e+n|\mathcal P_\ell|).
\end{gather*}
Hence
\begin{align}
   F_\ell(
      \mathbf W_{\mathrm{SFT}}^{(n)}
   )
   \leq
   C_5
   \left[
      \frac1n
      +
      \frac{
         \log^2(e+n|\mathcal P_\ell|)
      }{
         n^{1-\nu}
      }
      +
      \frac{
         (1+2\log(2/\delta))\log(en)
      }{
         B\,n^{m_\nu}
      }
   \right].
   \label{eq:minibatch-potential-final-rate}
\end{align}

Lemma~\ref{lem:sft-potential} converts this potential bound into
\Eqref{eq:theorem-two-batch-explicit}. The same scalar event controls every source,
so no additional union bound is required.
\end{proof}

\paragraph{Main proof of Theorem~\ref{thm:sft-single-task}.}
\label{proof:sft-single-task}
\begin{proof}
By Lemma~\ref{lem:sft-support-learning},
\Eqref{eq:theorem-two-batch-explicit} holds with probability at least $1-\delta$.
Since $0<\nu<1/2$, we have $m_\nu=\nu$.
For $n\ge1$ and $\delta\in(0,1)$,
\[
   \frac1n\leq\frac1{n^{1-\nu}},
   \qquad
   1+2\log\frac2\delta
   =O\!\left(\log\frac e\delta\right),
\]
where the implied constant is independent of $n$, $B$, and $\delta$.
Absorbing the factors
$\log^2(e+n|\mathcal P_\ell|)$ and $\log(en)$ into $\tilde{O}$,
the three-term bound in \Eqref{eq:theorem-two-batch-explicit} gives
\Eqref{eq:theorem-two-final-phase}, retaining the stated dependence on $B$ and $\delta$.
\end{proof}

\addtocontents{toc}{\protect\Needspace{5\protect\baselineskip}}
\subsection{Retention under continual SD}
\label{app:sd-retention}
\addtocontents{toc}{\protect\nopagebreak[4]}
\addtocontents{toc}{\protect\proofcontentsentry{proof:sd-no-forgetting}{Theorem~\protect\ref*{thm:sd-no-forgetting}. No forgetting under SD}}
\addtocontents{toc}{\protect\nopagebreak[4]}
\addtocontents{toc}{\protect\proofcontentschild{prop:sd-transition-learning}{Proposition~\protect\ref*{prop:sd-transition-learning}. Learning hint-supported transitions}}
\addtocontents{toc}{\protect\proofcontentschild{lem:restart-accuracy}{Lemma~\protect\ref*{lem:restart-accuracy}. From edge accuracy to rollout accuracy}}
\addtocontents{toc}{\protect\proofcontentschild{lem:sd-cross-task}{Lemma~\protect\ref*{lem:sd-cross-task}. Preservation across tasks}}
\sdretentiontheorem*

\paragraph{Sketch of proof of Theorem~\ref{thm:sd-no-forgetting}.}
\begin{itemize}
   \item Show that updates on other tasks preserve interior transitions and do not decrease correct transition probabilities, including before the target task is learned (Lemma~\ref{lem:sd-cross-task}).
   \item Apply the transition-learning argument conditionally on preceding-task training, using the preserved initialization bounds (Proposition~\ref{prop:sd-transition-learning} and Lemma~\ref{lem:sd-cross-task}).
   \item Preserve the learned transition bounds throughout all subsequent-task updates (Lemma~\ref{lem:sd-cross-task}).
   \item Convert the final transition bounds into old-task rollout accuracy on the same training event (Lemma~\ref{lem:restart-accuracy}).
\end{itemize}

\begin{lemma}[Preservation of correct transitions across tasks]
\label{lem:sd-cross-task}
Under Assumption~\ref{ass:transition-compatibility} and the nonnegative task inputs, an SD update on any
other task preserves all interior-source logits and does not decrease any correct
transition probability for task $r$. Therefore, after any sequence of SD updates,
each correct task-$r$ transition is at least its pretrained probability.
\end{lemma}
\begin{proof}
Lemma~\ref{lem:sd-gradient} gives, for $u\in\mathring C_k$,
\begin{align*}
\pi_{\mathbf W^+}(\cdot\mid u,x_r)
&=
\pi_{\mathbf W}(\cdot\mid u,x_r).
\end{align*}
Fix a correct old transition $(u,v)\in E_r$.  If a later update has no
hint-source at $u$, all logits at $u$ are unchanged.  Otherwise, let $(u,q)$
be a correct transition for the later task that contributes to the update at
$u$, so that
\begin{align*}
\mathbf W^+
&=
\mathbf W
+
\gamma_{s,u}x_s
\left(
e_{u,q}
-
\sum_{z\in\mathcal N(u)}
\pi_{\mathbf W}(z\mid u,x_s)e_{u,z}
\right)^\top,
\qquad
\gamma_{s,u}\ge0.
\end{align*}
If $\langle x_r,x_s\rangle=0$, then, for every $z\in\mathcal N(u)$,
\begin{align*}
x_r^\top(\mathbf W^+-\mathbf W)e_{u,z}
&=0.
\end{align*}
Hence all task-$r$ logits, and therefore the entire transition distribution at
$u$, are unchanged, irrespective of whether $q=v$.

If $\langle x_r,x_s\rangle>0$, Assumption~\ref{ass:transition-compatibility} gives $q=v$.  For $w\neq v$,
\begin{align*}
&x_r^\top\mathbf W^+(e_{u,v}-e_{u,w})
-
x_r^\top\mathbf W(e_{u,v}-e_{u,w})
\\
   &=
\gamma_{s,u}
\langle x_r,x_s\rangle
\left[
1-\pi_{\mathbf W}(v\mid u,x_s)
+
\pi_{\mathbf W}(w\mid u,x_s)
\right]
   \ge0,
\end{align*}
Hence
\begin{align*}
\pi_{\mathbf W^+}(v\mid u,x_r)
&\ge
\pi_{\mathbf W}(v\mid u,x_r),
\end{align*}
Applying these two cases to every later-task contribution and every subsequent
SD update gives the exact monotonicity
\begin{align*}
\pi_{\mathbf W_{\mathrm{SD}}^{[L]}}(v\mid u,x_r)
&\ge
\pi_{\mathbf W_{\mathrm{SD}}^{[r]}}(v\mid u,x_r),
\qquad
(u,v)\in E_r.
\end{align*}

The proof uses no ordering property of $r$ and $s$, so it also applies to updates
before task $r$ is learned. Iteration from $\mathbf W^{(0)}$ proves the final assertion.
\end{proof}

\paragraph{Main proof of Theorem~\ref{thm:sd-no-forgetting}.}
\label{proof:sd-no-forgetting}
\begin{proof}
Apply Lemma~\ref{lem:sd-cross-task} to the tasks preceding $\ell$.
At the start of task $\ell$, its interior kernel is still pretrained and its correct
transition probabilities are at least $c_{\mathrm{int}}\epsilon_2$.
The proof of Proposition~\ref{prop:sd-transition-learning}, using
Lemmas~\ref{lem:restart-odds}--\ref{lem:restart-concentration}, therefore applies
conditionally on the preceding history with the same uniform constants.
After integrating out that history, an event of probability at least $1-\delta$
gives $1-p_i\le C_{\mathrm{tr}}\xi_{\ell,n_\ell}^{(K)}(\delta)$ at task completion.
Lemma~\ref{lem:sd-cross-task} preserves all these inequalities during every later task.
On that same event, Lemma~\ref{lem:restart-accuracy} applied to the final predictor gives
\[
1-\pi_{\mathbf W_{\mathrm{SD}}^{[L]}}(y_T\in\tilde C_\ell\mid y_{\ell,0},x_\ell)
\le C_{\mathrm{bad}}(1+N_\ell)\epsilon_1
+N_\ell C_{\mathrm{tr}}\xi_{\ell,n_\ell}^{(K)}(\delta).
\]
Since $N_\ell\le KM$, this is \Eqref{eq:theorem-three-no-forgetting}.
The edgewise assertion is exactly Lemma~\ref{lem:sd-cross-task}.
\end{proof}

\addtocontents{toc}{\protect\Needspace{8\protect\baselineskip}}
\subsection{Catastrophic forgetting under SFT}
\label{app:sft-forgetting}
\addtocontents{toc}{\protect\nopagebreak[4]}
\addtocontents{toc}{\protect\proofcontentsentry{proof:sft-forgetting}{Theorem~\protect\ref*{thm:sft-forgetting}. Catastrophic forgetting under SFT}}
\addtocontents{toc}{\protect\nopagebreak[4]}
\addtocontents{toc}{\protect\proofcontentschild{lem:sft-support-learning}{Lemma~\protect\ref*{lem:sft-support-learning}. Simultaneous successor-set learning}}
\addtocontents{toc}{\protect\proofcontentschild{lem:sft-conflict-support}{Lemma~\protect\ref*{lem:sft-conflict-support}. Support separation and source overlap}}
\addtocontents{toc}{\protect\proofcontentschild{lem:sft-nonconflicted}{Lemma~\protect\ref*{lem:sft-nonconflicted}. Preservation of non-conflicted path mass}}
\addtocontents{toc}{\protect\proofcontentschild{lem:sft-transfer}{Lemma~\protect\ref*{lem:sft-transfer}. Exact transfer between task inputs}}
\addtocontents{toc}{\protect\proofcontentschild{lem:sft-conflicted-mass}{Lemma~\protect\ref*{lem:sft-conflicted-mass}. From successor suppression to old-path mass}}
\addtocontents{toc}{\protect\proofcontentsentry{proof:old-concentration}{Remark~\protect\ref*{rem:old-concentration}. The old-path concentration profile}}

\medskip\noindent\textbf{Notation for Theorem~\ref{thm:sft-forgetting}.}
For a path $\gamma=(u_0,\ldots,u_T)$, set
\begin{align*}
   \operatorname{Src}(\gamma)
   &:={}
   \{u_0,\ldots,u_{T-1}\}.
\end{align*}
Define
\begin{gather}
   \color{black}{
   \mathcal U_{\mathrm{new}}
   :={}
   \bigcup_{\Gamma\in\mathcal P_{\mathrm{new}}}
   \operatorname{Src}(\Gamma),}
   \label{eq:theorem-four-new-source-set}\\
   \color{black}{
   \mathcal A_{\mathrm{new}}(u)
   :={}
   \left\{
      v:\exists\Gamma\in\mathcal P_{\mathrm{new}}
      \text{ such that }(u,v)\in\Gamma
   \right\},
   \qquad u\in\mathcal U_{\mathrm{new}}.}
   \label{eq:theorem-four-new-successor-set}
\end{gather}
The conflicted old correct paths and their fraction are
\begin{gather}
   \mathcal C_{\mathrm{old}}
   :={}
   \left\{
      \gamma=(u_0,\ldots,u_T)\in\mathcal P_{\mathrm{old}}:
      \exists t<T,\ 
      u_t\in\mathcal U_{\mathrm{new}},\ 
      u_{t+1}\notin\mathcal A_{\mathrm{new}}(u_t)
   \right\},
   \label{eq:theorem-four-conflict-set}\\
   \theta
   :={}
   \frac{|\mathcal C_{\mathrm{old}}|}{|\mathcal P_{\mathrm{old}}|}.
   \label{eq:theorem-four-conflict-rate}
\end{gather}
Let
\begin{align}
   P_{\mathrm{pre}}^{\mathrm{old}}(\gamma)
   &:={}
   \pi_{\mathbf W_{\mathrm{new}}^{(0)}}
   (Y_{0:T}=\gamma\mid Y_0=y_{\mathrm{old},0},x_{\mathrm{old}}),
   \label{eq:theorem-four-pre-path-law}\\
   p_{\mathrm{old}}^{\mathrm{pre}}
   &:={}
   \pi_{\mathbf W_{\mathrm{new}}^{(0)}}\!\left(
      Y_T\in\tilde C_{\mathrm{old}}
      \mid Y_0=y_{\mathrm{old},0},x_{\mathrm{old}}
   \right)
   \notag\\
   &=
   \sum_{\gamma\in\mathcal P_{\mathrm{old}}}
   P_{\mathrm{pre}}^{\mathrm{old}}(\gamma),
   \label{eq:theorem-four-pre-accuracy}\\
   R_{\mathrm{pre}}
   &:={}
   \sum_{\gamma\in
      \mathcal P_{\mathrm{old}}\setminus\mathcal C_{\mathrm{old}}}
   P_{\mathrm{pre}}^{\mathrm{old}}(\gamma),
   \label{eq:theorem-four-pre-nonconflicted-mass}\\
   M_{\mathrm{old}}
   &:={}
   |\mathcal P_{\mathrm{old}}|.
   \notag
\end{align}
The second equality in Equation~\ref{eq:theorem-four-pre-accuracy} follows
because the acyclic nature of the DAG makes every student rollout node-simple and
$\mathcal P_{\mathrm{old}}$ contains every node-simple path ending in
$\tilde C_{\mathrm{old}}$.
For $s\in[0,1]$, define
\begin{align}
   \mathfrak C_{\mathrm{old}}(s)
   &:={}
   \max_{\substack{
      \mathcal S\subseteq\mathcal P_{\mathrm{old}}\\
      |\mathcal S|\le\lfloor sM_{\mathrm{old}}\rfloor
   }}
   \sum_{\gamma\in\mathcal S}
   P_{\mathrm{pre}}^{\mathrm{old}}(\gamma).
   \label{eq:theorem-four-concentration-profile}
\end{align}
Set
\begin{align}
   \lambda
   :={}
   \frac{\langle x_{\mathrm{old}},x_{\mathrm{new}}\rangle}
        {\|x_{\mathrm{new}}\|_2^2}>0,\quad
   m_\nu
   :={}
   \min\{\nu,1-\nu\}.
   \label{eq:theorem-four-lambda}
\end{align}
For $u\in\mathcal U_{\mathrm{new}}$ and $v\in\mathcal N(u)$, define
\begin{align}
   b_{u,v}
   &:={}
   x_{\mathrm{old}}^\top
   \mathbf W_{\mathrm{new}}^{(0)}e_{u,v}
   -
   \lambda x_{\mathrm{new}}^\top
   \mathbf W_{\mathrm{new}}^{(0)}e_{u,v},
   \label{eq:theorem-four-transfer-offset}\\
   b_{\min}(u)
   &:={}
   \min_{w\in\mathcal N(u)}b_{u,w},
   \notag\\
   \Delta_{\mathrm{map}}
   &:={}
   \max_{u\in\mathcal U_{\mathrm{new}}}
   \left[
      \max_{v\in\mathcal N(u)}b_{u,v}-b_{\min}(u)
   \right],
   \label{eq:theorem-four-transfer-range}\\
   C_{\mathrm{map}}
   &:={}
   e^{\Delta_{\mathrm{map}}}
   =
   \max_{u\in\mathcal U_{\mathrm{new}}}
   \max_{v\in\mathcal N(u)}
   e^{b_{u,v}-b_{\min}(u)}.
   \label{eq:theorem-four-transfer-constant}
\end{align}
Let $C_{\mathrm{new}}$ be the constant from the batch-explicit
bound of Lemma~\ref{lem:sft-support-learning} applied to $\mathcal D_{\mathrm{new}}$ from
$\mathbf W_{\mathrm{new}}^{(0)}$, and define
\begin{align}
   \bar\epsilon_n(\delta)
   &:={}
   C_{\mathrm{new}}|\mathcal P_{\mathrm{new}}|
   \left[
      \frac1n
      +
      \frac{\log^2(e+n|\mathcal P_{\mathrm{new}}|)}{n^{1-\nu}}
      +
      \frac{(1+2\log(2/\delta))\log(en)}{Bn^{m_\nu}}
   \right].
   \label{eq:theorem-four-learning-error}
\end{align}
Finally, define the relative largest and smallest old-path masses by
\begin{align}
   \kappa_{\max}^{\mathrm{old}}
   :={}
   \frac{M_{\mathrm{old}}}{p_{\mathrm{old}}^{\mathrm{pre}}}
   \max_{\gamma\in\mathcal P_{\mathrm{old}}}
   P_{\mathrm{pre}}^{\mathrm{old}}(\gamma),\quad
   \kappa_{\min}^{\mathrm{old}}
   :={}
   \frac{M_{\mathrm{old}}}{p_{\mathrm{old}}^{\mathrm{pre}}}
   \min_{\gamma\in\mathcal P_{\mathrm{old}}}
   P_{\mathrm{pre}}^{\mathrm{old}}(\gamma).
   \label{eq:theorem-four-relative-path-masses}
\end{align}

Write
\begin{align*}
   \mathbf W^{(j)}&:=\mathbf W_{\mathrm{new}}^{(j)}.
\end{align*}
For an old path $\gamma=(u_0,\ldots,u_T)$, set
\begin{align*}
   P_j^{\mathrm{old}}(\gamma)
   &:={}
   \pi_{\mathbf W^{(j)}}
   (Y_{0:T}=\gamma\mid Y_0=y_{\mathrm{old},0},x_{\mathrm{old}}).
\end{align*}

\sftforgettingtheorem*

\paragraph{Sketch of proof of Theorem~\ref{thm:sft-forgetting}.}
\begin{itemize}
   \item Characterize conflicted old paths by their overlap with new-task source nodes (Lemma~\ref{lem:sft-conflict-support}).
   \item Show that new-task SFT preserves every non-conflicted old-path mass exactly (Lemma~\ref{lem:sft-nonconflicted}).
   \item Bound the probability of unsupported successors under the new-task input (Lemma~\ref{lem:sft-support-learning}).
   \item Transfer successor suppression from the new-task input to the old-task input using the exact logit relation (Lemma~\ref{lem:sft-transfer}).
   \item Sum the preserved and suppressed path masses to obtain the old-task retention bounds (Lemmas~\ref{lem:sft-nonconflicted} and~\ref{lem:sft-conflicted-mass}).
\end{itemize}

\begin{lemma}[Support separation equals source overlap]
\label{lem:sft-conflict-support}
Under the disjoint terminal-cluster assumptions, an old correct path belongs to
$\mathcal C_{\mathrm{old}}$ if and only if its source set intersects
$\mathcal U_{\mathrm{new}}$. Thus \Eqref{eq:thm4-conflict-source-equivalence}
and \Eqref{eq:thm4-nonconf-source-disjoint} hold.
\end{lemma}
\begin{proof}
Let
\begin{align*}
   \mathcal S_{\mathrm{old}}
   &:={}
   \left\{
      \gamma\in\mathcal P_{\mathrm{old}}:
      \operatorname{Src}(\gamma)\cap\mathcal U_{\mathrm{new}}\neq\varnothing
   \right\}.
\end{align*}
By Equation~\ref{eq:theorem-four-conflict-set},
\begin{align*}
   \mathcal C_{\mathrm{old}}
   &\subseteq
   \mathcal S_{\mathrm{old}}.
\end{align*}
Conversely, fix
\begin{align*}
   \gamma=(u_0,\ldots,u_T)
   &\in
   \mathcal S_{\mathrm{old}}
\end{align*}
and let
\begin{align*}
   t_\star
   &:={}
   \max\{t<T:u_t\in\mathcal U_{\mathrm{new}}\}.
\end{align*}
Suppose
\begin{align*}
   u_{t_\star+1}
   &\in
   \mathcal A_{\mathrm{new}}(u_{t_\star}).
\end{align*}
Then there exists $\Gamma\in\mathcal P_{\mathrm{new}}$ containing
$(u_{t_\star},u_{t_\star+1})$.
If $u_{t_\star+1}$ is not the final node of $\Gamma$, then
\begin{align*}
   u_{t_\star+1}
   &\in
   \mathcal U_{\mathrm{new}},
\end{align*}
contradicting the maximality of $t_\star$.
If $u_{t_\star+1}$ is the final node of $\Gamma$, then
\begin{align*}
   u_{t_\star+1}
   &\in
   \widetilde C_{\mathrm{new}}.
\end{align*}
Since $\widetilde C_{\mathrm{new}}$ has no outgoing inter-cluster transition and
\begin{align*}
   \widetilde C_{\mathrm{old}}
   \cap
   \widetilde C_{\mathrm{new}}
   &=\varnothing,
\end{align*}
the old correct path $\gamma$ cannot subsequently reach
$\widetilde C_{\mathrm{old}}$, again a contradiction.
Therefore
\begin{align*}
   u_{t_\star+1}
   &\notin
   \mathcal A_{\mathrm{new}}(u_{t_\star}),
\end{align*}
so $\gamma\in\mathcal C_{\mathrm{old}}$. Hence
\begin{align}
   \mathcal C_{\mathrm{old}}
   &=
   \left\{
      \gamma\in\mathcal P_{\mathrm{old}}:
      \operatorname{Src}(\gamma)\cap\mathcal U_{\mathrm{new}}\neq\varnothing
   \right\}.
   \label{eq:thm4-conflict-source-equivalence}
\end{align}
Consequently,
\begin{align}
   \gamma\notin\mathcal C_{\mathrm{old}}
   &\quad\Longrightarrow\quad
   \operatorname{Src}(\gamma)\cap\mathcal U_{\mathrm{new}}=\varnothing.
   \label{eq:thm4-nonconf-source-disjoint}
\end{align}
\end{proof}

\begin{lemma}[Exact preservation of non-conflicted path mass]
\label{lem:sft-nonconflicted}
For every training history and $\gamma\notin\mathcal C_{\mathrm{old}}$,
$P_n^{\mathrm{old}}(\gamma)=P_{\mathrm{pre}}^{\mathrm{old}}(\gamma)$.
Moreover,
$p_{\mathrm{old}}^{\mathrm{pre}}-\mathfrak C_{\mathrm{old}}(\theta)
\le R_{\mathrm{pre}}\le\mathfrak C_{\mathrm{old}}(1-\theta)$.
\end{lemma}
\begin{proof}
Lemma~\ref{lem:sft-gradient} changes
only edge columns whose source node lies in $\mathcal U_{\mathrm{new}}$.
Thus Equation~\ref{eq:thm4-nonconf-source-disjoint} gives, for every
$\gamma\notin\mathcal C_{\mathrm{old}}$,
\begin{align}
   P_n^{\mathrm{old}}(\gamma)
   &=
   P_0^{\mathrm{old}}(\gamma)
   =
   P_{\mathrm{pre}}^{\mathrm{old}}(\gamma).
   \label{eq:thm4-nonconf-path-invariance}
\end{align}
Consequently,
\begin{align}
   \sum_{\gamma\notin\mathcal C_{\mathrm{old}}}
   P_n^{\mathrm{old}}(\gamma)
   &=
   R_{\mathrm{pre}}.
   \label{eq:thm4-nonconf-mass-exact}
\end{align}
Since
$|\mathcal P_{\mathrm{old}}\setminus\mathcal C_{\mathrm{old}}|
=(1-\theta)M_{\mathrm{old}}$,
the concentration profile gives
\begin{align}
   R_{\mathrm{pre}}
   &\leq
   \mathfrak C_{\mathrm{old}}(1-\theta).
   \label{eq:thm4-nonconf-mass}
\end{align}
Also,
$|\mathcal C_{\mathrm{old}}|=\theta M_{\mathrm{old}}$ implies
\begin{align}
   \sum_{\gamma\in\mathcal C_{\mathrm{old}}}
   P_{\mathrm{pre}}^{\mathrm{old}}(\gamma)
   &\leq
   \mathfrak C_{\mathrm{old}}(\theta),
\end{align}
and therefore
\begin{align}
   R_{\mathrm{pre}}
   &=
   p_{\mathrm{old}}^{\mathrm{pre}}
   -
   \sum_{\gamma\in\mathcal C_{\mathrm{old}}}
   P_{\mathrm{pre}}^{\mathrm{old}}(\gamma)
   \geq
   p_{\mathrm{old}}^{\mathrm{pre}}
   -
   \mathfrak C_{\mathrm{old}}(\theta).
   \label{eq:thm4-Rpre-lower}
\end{align}
Since all path probabilities are nonnegative,
Equations~\ref{eq:thm4-nonconf-mass-exact} and
\ref{eq:thm4-Rpre-lower} prove the late-phase lower bound.
\end{proof}

\begin{lemma}[Exact transfer between task inputs]
\label{lem:sft-transfer}
For every training history, source $u\in\mathcal U_{\mathrm{new}}$,
successor $v$, and iteration $j$, the probability identity
\Eqref{eq:thm4-global-probability-transfer} holds. In particular,
\[
\pi_{\mathbf W^{(j)}}(v\mid u,x_{\mathrm{old}})
\le C_{\mathrm{map}}
\pi_{\mathbf W^{(j)}}(v\mid u,x_{\mathrm{new}})^\lambda.
\]
\end{lemma}
\begin{proof}
For $u\in\mathcal U_{\mathrm{new}}$ and $v\in\mathcal N(u)$, define
\begin{align*}
   z_{i,u,v}^{(j)}
   &:={}
   x_i^\top\mathbf W^{(j)}e_{u,v},
   \qquad
   i\in\{\mathrm{old},\mathrm{new}\}.
\end{align*}
Every update of the column $e_{u,v}$ during new-task SFT lies in
$\operatorname{span}\{x_{\mathrm{new}}\}$. Hence
\begin{align*}
   z_{\mathrm{old},u,v}^{(j)}
   -z_{\mathrm{old},u,v}^{(j-1)}
   &=
   \lambda
   \left(
      z_{\mathrm{new},u,v}^{(j)}
      -z_{\mathrm{new},u,v}^{(j-1)}
   \right),
\end{align*}
and Equation~\ref{eq:theorem-four-transfer-offset} gives
\begin{align}
   z_{\mathrm{old},u,v}^{(j)}
   &=
   b_{u,v}
   +
   \lambda z_{\mathrm{new},u,v}^{(j)}.
   \label{eq:thm4-global-logit-transfer}
\end{align}
Writing
\begin{align*}
   p_{u,v}^{(j)}
   &:={}
   \pi_{\mathbf W^{(j)}}(v\mid u,x_{\mathrm{new}}),
\end{align*}
Equation~\ref{eq:thm4-global-logit-transfer} implies
\begin{align}
   \pi_{\mathbf W^{(j)}}(v\mid u,x_{\mathrm{old}})
   &=
   \frac{
      e^{b_{u,v}}(p_{u,v}^{(j)})^{\lambda}
   }{
      \displaystyle
      \sum_{w\in\mathcal N(u)}
      e^{b_{u,w}}(p_{u,w}^{(j)})^{\lambda}
   }.
   \label{eq:thm4-global-probability-transfer}
\end{align}
Since $x_{\mathrm{old}},x_{\mathrm{new}}\in\{0,1\}^d$,
\begin{align*}
   0<\lambda&\le1,
   &
   \sum_{w\in\mathcal N(u)}(p_{u,w}^{(j)})^{\lambda}
   &\ge
   \sum_{w\in\mathcal N(u)}p_{u,w}^{(j)}
   =1.
\end{align*}
Therefore the denominator in
Equation~\ref{eq:thm4-global-probability-transfer} satisfies
\begin{align}
   \sum_{w\in\mathcal N(u)}
   e^{b_{u,w}}(p_{u,w}^{(j)})^{\lambda}
   &\ge
   e^{b_{\min}(u)}
   \sum_{w\in\mathcal N(u)}
   (p_{u,w}^{(j)})^{\lambda}
   \ge
   e^{b_{\min}(u)}.
   \label{eq:thm4-transfer-denominator-lower}
\end{align}
For $i\in\{\mathrm{old},\mathrm{new}\}$, put
\begin{align*}
   p_{i,u,v}^{(0)}
   &:={}
   \pi_{\mathbf W_{\mathrm{new}}^{(0)}}
   (v\mid u,x_i).
\end{align*}
For every $u\in\mathcal U_{\mathrm{new}}$ and
$v,w\in\mathcal N(u)$,
\begin{align}
   e^{b_{u,v}-b_{u,w}}
   &=
   \frac{p_{\mathrm{old},u,v}^{(0)}}
        {p_{\mathrm{old},u,w}^{(0)}}
   \left(
      \frac{p_{\mathrm{new},u,w}^{(0)}}
           {p_{\mathrm{new},u,v}^{(0)}}
   \right)^{\lambda}.
   \label{eq:thm4-offset-ratio-form}
\end{align}
Consequently,
\begin{align}
   C_{\mathrm{map}}
   &=
   e^{\Delta_{\mathrm{map}}}
   =
   \max_{u\in\mathcal U_{\mathrm{new}}}
   \max_{v,w\in\mathcal N(u)}
   \frac{p_{\mathrm{old},u,v}^{(0)}}
        {p_{\mathrm{old},u,w}^{(0)}}
   \left(
      \frac{p_{\mathrm{new},u,w}^{(0)}}
           {p_{\mathrm{new},u,v}^{(0)}}
   \right)^{\lambda}.
   \label{eq:thm4-cmap-ratio-form}
\end{align}
Combining Equations~\ref{eq:thm4-global-probability-transfer}
and~\ref{eq:thm4-transfer-denominator-lower},
\begin{align}
   \pi_{\mathbf W^{(j)}}(v\mid u,x_{\mathrm{old}})
   &\le
   e^{b_{u,v}-b_{\min}(u)}
   (p_{u,v}^{(j)})^{\lambda}
   \le
   C_{\mathrm{map}}
   (p_{u,v}^{(j)})^{\lambda}.
   \label{eq:thm4-probability-transfer-upper}
\end{align}
\end{proof}

\begin{lemma}[From successor suppression to total old-path mass]
\label{lem:sft-conflicted-mass}
Fix any training history satisfying
\Eqref{eq:thm4-new-support-learning}. Then
\[
R_{\mathrm{pre}}\le\Pi(\mathbf W^{(n)})
\le R_{\mathrm{pre}}+C_{\mathrm{map}}\theta M_{\mathrm{old}}
\bar\epsilon_n(\delta')^\lambda.
\]
Here the premise is a deterministic property of the fixed history; it is supplied
with high probability by Lemma~\ref{lem:sft-support-learning}.
\end{lemma}
\begin{proof}
For every
$\gamma=(u_0,\ldots,u_T)\in\mathcal C_{\mathrm{old}}$,
Equation~\ref{eq:theorem-four-conflict-set} gives an index
$t_\gamma<T$ such that
\begin{align*}
   u_{t_\gamma}\in\mathcal U_{\mathrm{new}},\quad
   u_{t_\gamma+1}\notin\mathcal A_{\mathrm{new}}(u_{t_\gamma}).
\end{align*}
Thus Lemma~\ref{lem:sft-transfer} yields
\begin{align}
   P_n^{\mathrm{old}}(\gamma)
   &\leq
   C_{\mathrm{map}}
   \bar\epsilon_n(\delta')^{\lambda}.
   \label{eq:thm4-conflicted-path-bound}
\end{align}
Since
$|\mathcal C_{\mathrm{old}}|=\theta M_{\mathrm{old}}$,
\begin{align}
   \sum_{\gamma\in\mathcal C_{\mathrm{old}}}
   P_n^{\mathrm{old}}(\gamma)
   &\leq
   C_{\mathrm{map}}
   \theta M_{\mathrm{old}}
   \bar\epsilon_n(\delta')^{\lambda}.
   \label{eq:thm4-conflicted-mass}
\end{align}
Because the student rollout is node-simple and $\mathcal P_{\mathrm{old}}$
contains all node-simple correct paths, its terminal-cluster probability is
exactly
\begin{align}
   &\pi_{\mathbf W^{(n)}}\!\left(
      Y_T\in\tilde C_{\mathrm{old}}
      \mid Y_0=y_{\mathrm{old},0},x_{\mathrm{old}}
   \right)
   =
   \sum_{\gamma\in\mathcal P_{\mathrm{old}}}
   P_n^{\mathrm{old}}(\gamma).
   \label{eq:thm4-terminal-probability-path-sum}
\end{align}
Together with Equation~\ref{eq:thm4-nonconf-mass-exact}, this gives
\begin{align}
   R_{\mathrm{pre}}
   &\leq
   \pi_{\mathbf W^{(n)}}\!\left(
      Y_T\in\tilde C_{\mathrm{old}}
      \mid Y_0=y_{\mathrm{old},0},x_{\mathrm{old}}
   \right)
   \notag\\
   &\leq
   R_{\mathrm{pre}}
   +
   C_{\mathrm{map}}
   \theta M_{\mathrm{old}}
   \bar\epsilon_n(\delta')^{\lambda}.
   \label{eq:thm4-exact-finite-bound}
\end{align}
\end{proof}

\paragraph{Main proof of Theorem~\ref{thm:sft-forgetting}.}
\label{proof:sft-forgetting}
\begin{proof}
Fix $\delta'\in(0,1)$.
Apply Lemma~\ref{lem:sft-support-learning} to $\mathcal D_{\mathrm{new}}$ with initial point $\mathbf W_{\mathrm{new}}^{(0)}$.
With probability at least $1-\delta'$,
\begin{align}
   \max_{u\in\mathcal U_{\mathrm{new}}}
   \left[
      1-
      \pi_{\mathbf W^{(n)}}
      (\mathcal A_{\mathrm{new}}(u)\mid u,x_{\mathrm{new}})
   \right]
   &\le
   \bar\epsilon_n(\delta'),
   \label{eq:thm4-new-support-learning}
\end{align}
where $\bar\epsilon_n(\delta')$ is defined in
Equation~\ref{eq:theorem-four-learning-error} and satisfies
\begin{align*}
   \bar\epsilon_n(\delta')
   &=
   \widetilde O_{G,\mathbf W_{\mathrm{new}}^{(0)}}\!\left(
      |\mathcal P_{\mathrm{new}}|
      \left[
         \frac1{n^{1-\nu}}
         +
         \frac{\log(1/\delta')}{Bn^{m_\nu}}
      \right]
   \right)
   \longrightarrow0.
\end{align*}

By Lemmas~\ref{lem:sft-transfer} and~\ref{lem:sft-conflicted-mass}, this event
also gives \Eqref{eq:thm4-exact-finite-bound}.
Since $\nu \in (0, 1/2)$, 
Equation~\ref{eq:theorem-four-learning-error} implies
\begin{align}
   C_{\mathrm{map}}\theta M_{\mathrm{old}}
   \bar\epsilon_n(\delta')^{\lambda}
   &=
   \widetilde O\!\left(
      C_{\mathrm{map}}\theta M_{\mathrm{old}}C_{\mathrm{new}}^\lambda
      |\mathcal P_{\mathrm{new}}|^{\lambda}
      \left[
         \frac{
            1+B^{-1}\log(1/\delta')
         }{
            n^{\nu} 
         }
      \right]^{\lambda}
   \right).
   \label{eq:thm4-remainder-order}
\end{align}

For the late-phase statement, take $\delta'=\delta$ and combine
Equations~\ref{eq:thm4-exact-finite-bound},
\ref{eq:thm4-nonconf-mass}, and
\ref{eq:thm4-Rpre-lower}.

If $\theta=0$, Equations~\ref{eq:thm4-conflict-source-equivalence} and
\ref{eq:thm4-nonconf-path-invariance} give
\begin{align}
   \pi_{\mathbf W^{(n)}}\!\left(
      Y_T\in\tilde C_{\mathrm{old}}
      \mid Y_0=y_{\mathrm{old},0},x_{\mathrm{old}}
   \right)
   &=
   p_{\mathrm{old}}^{\mathrm{pre}}.
   \label{eq:theorem-four-no-conflict}
\end{align}
\end{proof}

\oldconcentrationremark*

\paragraph{Sketch of proof of Remark~\ref{rem:old-concentration}.}
\begin{itemize}
   \item Bound the mass of the largest admissible subset using the largest individual old-path mass.
   \item Obtain a second bound by subtracting the smallest omitted masses from the total, and combine the two bounds.
\end{itemize}

\paragraph{Main proof of Remark~\ref{rem:old-concentration}.}
\label{proof:old-concentration}
\begin{proof}
Let
\begin{align*}
   p_{(1)}^{\mathrm{old}}
   \geq
   \cdots
   \geq
   p_{(M_{\mathrm{old}})}^{\mathrm{old}}
\end{align*}
be the ordered values of
$\{P_{\mathrm{pre}}^{\mathrm{old}}(\gamma):
\gamma\in\mathcal P_{\mathrm{old}}\}$.
For $\phi\in[0,1]$, put
$q=\lfloor\phi M_{\mathrm{old}}\rfloor$. Then
\begin{align}
   \mathfrak C_{\mathrm{old}}(\phi)
   &=
   \sum_{i=1}^{q}p_{(i)}^{\mathrm{old}}
   \leq
   q
   \max_{\gamma\in\mathcal P_{\mathrm{old}}}
   P_{\mathrm{pre}}^{\mathrm{old}}(\gamma)
   \leq
   p_{\mathrm{old}}^{\mathrm{pre}}
   \kappa_{\max}^{\mathrm{old}}\phi.
   \label{eq:thm4-concentration-kmax}
\end{align}
Moreover,
$M_{\mathrm{old}}-q\geq(1-\phi)M_{\mathrm{old}}$, so
\begin{align}
   \mathfrak C_{\mathrm{old}}(\phi)
   &=
   p_{\mathrm{old}}^{\mathrm{pre}}
   -
   \sum_{i=q+1}^{M_{\mathrm{old}}}p_{(i)}^{\mathrm{old}}
   \leq
   p_{\mathrm{old}}^{\mathrm{pre}}
   \left[
      1-\kappa_{\min}^{\mathrm{old}}(1-\phi)
   \right].
   \label{eq:thm4-concentration-kmin}
\end{align}
Combining Equations~\ref{eq:thm4-concentration-kmax} and
\ref{eq:thm4-concentration-kmin} gives
\begin{align}
   \mathfrak C_{\mathrm{old}}(\phi)
   &\leq
   p_{\mathrm{old}}^{\mathrm{pre}}
   \min\left\{
      \kappa_{\max}^{\mathrm{old}}\phi,
      1-\kappa_{\min}^{\mathrm{old}}
      +\kappa_{\min}^{\mathrm{old}}\phi
   \right\}.
\end{align}
\end{proof}

\addtocontents{toc}{\protect\Needspace{20\protect\baselineskip}}
\subsection{Pre-training and intermediate restartability}
\label{app:restart}
\addtocontents{toc}{\protect\nopagebreak[4]}
\addtocontents{toc}{\protect\proofcontentsentry{prop:restart-positive}{Proposition~\protect\ref*{prop:restart-positive}. Uniform restartability}}
\addtocontents{toc}{\protect\nopagebreak[4]}
\addtocontents{toc}{\protect\proofcontentschild{lem:restart-concentration}{Lemma~\protect\ref*{lem:restart-concentration}. Simultaneous learning with finite mini-batches}}
\addtocontents{toc}{\protect\proofcontentschild{lem:restart-accuracy}{Lemma~\protect\ref*{lem:restart-accuracy}. From edge accuracy to rollout accuracy}}
\addtocontents{toc}{\protect\proofcontentschild{lem:restart-training-length}{Lemma~\protect\ref*{lem:restart-training-length}. From training length to effective drift}}
\addtocontents{toc}{\protect\proofcontentsentry{proof:restart-negative}{Proposition~\protect\ref*{prop:restart-negative}. Delayed restartability at a rare entrance}}
\addtocontents{toc}{\protect\nopagebreak[4]}
\addtocontents{toc}{\protect\proofcontentschild{lem:restart-coverage}{Lemma~\protect\ref*{lem:restart-coverage}. Update opportunities from retained coverage}}
\addtocontents{toc}{\protect\proofcontentschild{lem:restart-concentration}{Lemma~\protect\ref*{lem:restart-concentration}. Simultaneous learning with finite mini-batches}}
\addtocontents{toc}{\protect\proofcontentschild{lem:restart-accuracy}{Lemma~\protect\ref*{lem:restart-accuracy}. From edge accuracy to rollout accuracy}}
\addtocontents{toc}{\protect\proofcontentschild{lem:restart-training-length}{Lemma~\protect\ref*{lem:restart-training-length}. From training length to effective drift}}
\addtocontents{toc}{\protect\proofcontentschild{lem:restart-path}{Lemma~\protect\ref*{lem:restart-path}. Control of visitation growth}}
\addtocontents{toc}{\protect\proofcontentschild{lem:restart-nonvisit}{Lemma~\protect\ref*{lem:restart-nonvisit}. Nonvisitation throughout adaptive training}}
\addtocontents{toc}{\protect\proofcontentschild{lem:restart-frozen}{Lemma~\protect\ref*{lem:restart-frozen}. An unvisited downstream cluster remains unlearned}}
\addtocontents{toc}{\protect\proofcontentsentry{statement:restart-time}{Remark~\protect\ref*{rem:restart-time}. Separation of training times}}
\addtocontents{toc}{\protect\nopagebreak[4]}
\addtocontents{toc}{\protect\proofcontentschild{prop:restart-positive}{Proposition~\protect\ref*{prop:restart-positive}. Uniform restartability}}
\addtocontents{toc}{\protect\proofcontentschild{proof:restart-negative}{Proposition~\protect\ref*{prop:restart-negative}. Delayed restartability at a rare entrance}}
\addtocontents{toc}{\protect\proofcontentsentry{rem:restart-Mprime}{Remark~\protect\ref*{rem:restart-Mprime}. The scale of initial rarity}}
\addtocontents{toc}{\protect\nopagebreak[4]}
\addtocontents{toc}{\protect\proofcontentschild{lem:restart-rarity-scale}{Lemma~\protect\ref*{lem:restart-rarity-scale}. The coarse rarity scale}}
\addtocontents{toc}{\protect\proofcontentsentry{rem:restart-bypass}{Remark~\protect\ref*{rem:restart-bypass}. Correct paths bypassing the rare entrance}}
\addtocontents{toc}{\protect\nopagebreak[4]}
\addtocontents{toc}{\protect\proofcontentschild{proof:restart-negative}{Proposition~\protect\ref*{prop:restart-negative}. Delayed restartability at a rare entrance}}
\begin{proposition}[Uniform restartability]
\label{prop:restart-positive}
Suppose Assumptions~\ref{ass:transition-compatibility} and~\ref{ass:pretrained-diversity} hold, $\eta_j=\Theta(j^{-\nu})$ for fixed $0\le\nu<1/2$,
and $0<\epsilon_1,\epsilon_2,\delta,\epsilon_{\mathrm{rel}}<1$. If $n$ satisfies \Eqref{eq:restart-explicit-n} with a sufficiently large uniform constant, then with probability at least $1-\delta$ over training,
\begin{equation}
 \inf_{h:\mathcal S_{\ell,h}\ne\varnothing}\ \inf_{s\in\mathcal S_{\ell,h}}
 \pi_{\mathbf W^{(n)}}(y_T\in\tilde C_\ell\mid y_0=s,x_\ell)
 \ge 1-O(KM\epsilon_1),
 \label{eq:restart-positive}
\end{equation}
where $\mathcal S_{\ell,h}$ is the set of relevant starts in $C_h$.
\end{proposition}

\restartnegativeproposition*

\paragraph{Sketch of proofs of Propositions~\ref{prop:restart-positive} and~\ref{prop:restart-negative}.}
\begin{itemize}
   \item Show that the stated training length implies the effective-drift bound required for transition learning (Lemma~\ref{lem:restart-training-length}).
   \item Under full coverage, learn all correct transitions on a common event and convert their accuracy into successful reasoning from every relevant start (Lemmas~\ref{lem:restart-concentration} and~\ref{lem:restart-accuracy}).
   \item Under the rare-entrance condition, bound the increase in visitation probability throughout training (Lemma~\ref{lem:restart-path}).
   \item Use this visitation bound to establish nonvisitation across all adaptively generated training rollouts (Lemma~\ref{lem:restart-nonvisit}).
   \item Learn the transitions reachable without the rare entrance and bound the final error from the original start (Lemmas~\ref{lem:restart-coverage},~\ref{lem:restart-concentration}, and~\ref{lem:restart-accuracy}).
   \item Show that nonvisitation leaves the downstream cluster unchanged and bounds its restart accuracy (Lemma~\ref{lem:restart-frozen}).
   \item Intersect the learning and nonvisitation events to obtain the joint original-start and restart guarantees (Proposition~\ref{prop:restart-negative}).
\end{itemize}

\begin{lemma}[The stated training length implies the drift condition]
\label{lem:restart-training-length}
Under the learning-rate conditions of Proposition~\ref{prop:restart-positive},
\Eqref{eq:restart-explicit-n} with a sufficiently large uniform constant implies
\begin{equation}
 \mathfrak D_{\ell,n}^{(K)}(\delta/2)
 \ge\frac{\epsilon_1^{-1}-1}
 {c_{\mathrm{int}}^2\epsilon_2^2\log(1/\epsilon_{\mathrm{rel}})}.
 \label{eq:restart-training}
\end{equation}
\end{lemma}
\begin{proof}
For fixed $0\le\nu<1/2$ and uniform learning-rate constants,
\[
 \sum_{j=1}^n\eta_j=\Theta(n^{1-\nu}),\qquad
 \sum_{j=1}^n\eta_j^2=\Theta(n^{1-2\nu}).
\]
The ratio of the square-root term in \Eqref{eq:theorem-one-drift-K} to its positive drift term is
\[
 O\!\left(\sqrt{\frac{\log(2KM/\delta)}{B\underline\rho_{\ell,K, M}^{\,2}n}}\right).
\]
The first term of the maximum in \Eqref{eq:restart-explicit-n}, with a sufficiently large
constant, makes this ratio at most one half. Consequently the effective drift is
$\Omega(\underline\rho_{\ell,K, M}n^{1-\nu})$. The second term in that maximum, again with a
sufficiently large constant, makes this lower bound at least the right-hand side of
\Eqref{eq:restart-training}. This proves the implication.
\end{proof}

\begin{lemma}[Control of the increase in visitation]
\label{lem:restart-path}
Let $N=N_\ell=|E_\ell|$, the total number of correct inter-cluster edges of the task, be fixed.
Under the retained initialization condition, for every training history and iteration $j$,
\begin{align}
 &\pi_{\mathbf W^{(j)}}(\exists t<T:y_t=u_m\mid y_{\ell,0},x_\ell)
 \notag\\
 &\quad\le
 \pi_{\mathbf W^{(0)}}(\exists t<T:y_t=u_m\mid y_{\ell,0},x_\ell)
 \frac{c_{\mathrm{int}}^{-N}}{\epsilon_2^{N}}.
 \label{eq:restart-visit-growth}
\end{align}
\end{lemma}
\begin{proof}
Iterating the second inequality of \Eqref{eq:restart-ratio-step} from $1$ to $j$ gives
\[
 \frac{p_{u_i,v}^{(j)}}{p_{u_i,v}^{(0)}}
 \le\frac{p_i^{(j)}}{p_i^{(0)}}\le\frac1{p_i^{(0)}}
 \qquad ((u_i,v)\in E).
\]
This applies to every outgoing edge at an updated source, including intra-cluster edges.
At other sources the ratio is one. Factor the probability of a complete rollout into
its one-step Markov probabilities, using the same stopping rule for both predictors.
Each updated source occurs at most once by node simplicity. Sources that do not occur
can be included because $p_i^{(j)}/p_i^{(0)}\ge1$. Consequently,
\begin{equation}
 \pi_{\mathbf W^{(j)}}(y_{0:T}\mid y_{\ell,0},x_\ell)
 \le\pi_{\mathbf W^{(0)}}(y_{0:T}\mid y_{\ell,0},x_\ell)
       \prod_{i=1}^{N_\ell}\frac{p_i^{(j)}}{p_i^{(0)}}.
 \label{eq:restart-path-ratio}
\end{equation}
The retained initialization condition gives $p_i^{(0)}\ge c_{\mathrm{int}}\epsilon_2$, and hence
\begin{equation}
 \prod_{i=1}^{N_\ell}\frac1{p_i^{(0)}}
 \le c_{\mathrm{int}}^{-N}\epsilon_2^{-N}.
 \label{eq:restart-product-scale}
\end{equation}
Sum \Eqref{eq:restart-path-ratio} over the node-simple rollouts that visit $u_m$
before stopping, and use $p_i^{(j)}\le1$ together with \Eqref{eq:restart-product-scale}.
\end{proof}

\begin{lemma}[Nonvisitation throughout adaptive training]
\label{lem:restart-nonvisit}
Under Assumption~\ref{ass:restart-rare},
\begin{align}
 &\mathbb P\!\left(\sum_{j=1}^n\sum_{b=1}^B
 \mathbf1\{\tau_m^{(b,j)}<T^{(b,j)}\}>0\right)
 \notag\\
 &\quad\le Bn\,
 \pi_{\mathbf W^{(0)}}(\exists t<T:y_t=u_m\mid y_{\ell,0},x_\ell)
 \frac{c_{\mathrm{int}}^{-N}}{\epsilon_2^{N}}
 =O\!\left(\frac{Bn}{M'\epsilon_2^{N}}\right).
 \label{eq:restart-nonvisit-prob}
\end{align}
\end{lemma}
\begin{proof}
Conditional on $\mathcal F_{j-1}$, the student rollout has law $\pi_{\mathbf W^{(j-1)}}$.
Lemma~\ref{lem:restart-path} therefore gives
\begin{align*}
 &\mathbb E[\mathbf1\{\tau_m^{(b,j)}<T^{(b,j)}\}\mid\mathcal F_{j-1}]\\
 &\quad=\pi_{\mathbf W^{(j-1)}}(\exists t<T:y_t=u_m\mid y_{\ell,0},x_\ell)\\
 &\quad\le\pi_{\mathbf W^{(0)}}(\exists t<T:y_t=u_m\mid y_{\ell,0},x_\ell)
            \frac{c_{\mathrm{int}}^{-N}}{\epsilon_2^{N}}.
\end{align*}
Here $T^{(b,j)}$ is the random length of that sampled rollout.
Take total expectations and sum over $b,j$. The probability of a positive visit count
is at most its expectation, proving the claim. 
\end{proof}

\begin{lemma}[An unvisited downstream cluster remains unlearned]
\label{lem:restart-frozen}
Under Assumptions~\ref{ass:transition-compatibility} and~\ref{ass:restart-rare},
if no training rollout visits $u_m$, then
$\mathbf W^{(n)}e_{u,v}=\mathbf W^{(0)}e_{u,v}$ for every $u\in C_k$ and $(u,v)\in E$,
and \Eqref{eq:restart-restart-low} holds.
\end{lemma}
\begin{proof}
Every training rollout starts outside $C_k$. Its only entrance is $e_m$, so nonvisitation of
$u_m$ implies nonvisitation of all of $C_k$. Lemma~\ref{lem:restart-update}, iterated over the
training steps, therefore leaves all columns with source in $C_k$ unchanged, even if a hint
contains an edge inside $C_k$.

Success from $s\in\mathcal S_{\ell,k}$ requires leaving $C_k\ne\tilde C_\ell$.
Since $s$ is a relevant start, a correct path from $y_{\ell,0}$ has a prefix ending at $s$.
Concatenating this prefix with any successful continuation from $s$ gives a correct path:
a repeated vertex would create a directed cycle.
Consequently, the first exit of a successful continuation belongs to $E_\ell$.
By Assumption~\ref{ass:transition-compatibility}, at most one such exit leaves each
$u\in\partial_\ell^\star C_k$, so there are at most $|\partial_\ell^\star C_k|\le M$ such edges.
For each $(u,v)\in E_\ell$ with $u\in C_k$, the Markov property gives
\begin{align*}
 &\mathbb E_{\mathbf W^{(n)}}\!\left[\sum_{t<T}\mathbf1\{y_t=u,y_{t+1}=v\}
 \mid y_0=s,x_\ell\right]\\
 &\quad=\pi_{\mathbf W^{(n)}}(v\mid u,x_\ell)
 \mathbb E_{\mathbf W^{(n)}}\!\left[\sum_{t<T}\mathbf1\{y_t=u\}\mid y_0=s,x_\ell\right]
 \le\pi_{\mathbf W^{(0)}}(v\mid u,x_\ell).
\end{align*}
The last step uses parameter preservation and node simplicity. A union bound over these
correct exit edges and the retained initialization condition yield
\begin{align*}
 &\pi_{\mathbf W^{(n)}}(y_T\in\tilde C_\ell\mid y_0=s,x_\ell)\\
 &\quad\le\sum_{\substack{(u,v)\in E_\ell\\u\in C_k}}
       \pi_{\mathbf W^{(0)}}(v\mid u,x_\ell)
 =O\!\left(|\partial_\ell^\star C_k|\epsilon_2\right)
 =O(M\epsilon_2).
\end{align*}
The bound is uniform over $s\in\mathcal S_{\ell,k}$, proving \Eqref{eq:restart-restart-low}.
\end{proof}

\paragraph{Main proofs of Propositions~\ref{prop:restart-positive} and~\ref{prop:restart-negative}.}

\begin{proof}[Proof of Proposition~\ref{prop:restart-positive}]
\phantomsection\label{proof:restart-positive}
Lemma~\ref{lem:restart-training-length} first gives \Eqref{eq:restart-training} from the stated condition on $n$.
By Lemma~\ref{lem:restart-concentration} with failure budget $\delta/2$ and \Eqref{eq:restart-training}, all $N_\ell$ correct
edges have error at most $\epsilon_1$ with probability at least $1-\delta/2$.
Fix any such training history. Lemma~\ref{lem:restart-accuracy} gives, for every relevant start,
\[
 1-\pi_{\mathbf W^{(n)}}(y_T\in\tilde C_\ell\mid y_0=s,x_\ell)
 \le [C_{\mathrm{bad}}(1+N_\ell)+N_\ell]\epsilon_1
 \le(1+2C_{\mathrm{bad}})KM\epsilon_1.
\]
The same history works for every $s$, so no additional probabilistic union bound is needed.
This is stronger than the stated confidence $1-\delta$.
\end{proof}

\begin{proof}[Proof of Proposition~\ref{prop:restart-negative}]
\phantomsection\label{proof:restart-negative}
The common condition on $n$ implies \Eqref{eq:restart-training} by Lemma~\ref{lem:restart-training-length}.
Hold $N=N_\ell=|E_\ell|$ fixed. Let $C/M'$ be the initial visitation bound in
Assumption~\ref{ass:restart-rare}, and take the uniform constant in
\Eqref{eq:restart-budget} to be at least $2C c_{\mathrm{int}}^{-N}$.
Then \Eqref{eq:restart-budget} and initial rarity imply
\begin{align}
 &\pi_{\mathbf W^{(0)}}(\exists t<T:y_t=u_m\mid y_{\ell,0},x_\ell)
 \frac{c_{\mathrm{int}}^{-N}}{\epsilon_2^{N}}
 \le\min\!\left\{KM\epsilon_1,\frac{\delta}{2Bn}\right\}.
 \label{eq:restart-proof-budget}
\end{align}
Lemma~\ref{lem:restart-nonvisit} makes the probability that any training rollout visits $u_m$
at most $\delta/2$. Independently of whether nonvisitation holds, Lemma~\ref{lem:restart-concentration} with failure budget $\delta/2$
and \Eqref{eq:restart-training} give error at most $\epsilon_1$ for all edges specified in
Lemma~\ref{lem:restart-coverage}, except on a set of probability at most $\delta/2$.
A union bound gives both conclusions jointly with probability at least $1-\delta$;
the two events need not be independent.

Fix a history in their intersection. Lemma~\ref{lem:restart-path} and
\Eqref{eq:restart-proof-budget} bound the final evaluation visitation probability by
$KM\epsilon_1$. In \Eqref{eq:restart-accuracy-rare} this yields
\[
 1-\pi_{\mathbf W^{(n)}}(y_T\in\tilde C_\ell\mid y_{\ell,0},x_\ell)
 \le[C_{\mathrm{bad}}(1+KM)+2KM]\epsilon_1
 \le(2+2C_{\mathrm{bad}})KM\epsilon_1.
\]
Lemma~\ref{lem:restart-frozen} gives the restart upper bound on that same history.
This proves all three assertions, including the nonvisit event, within confidence $1-\delta$.
\end{proof}

\phantomsection\label{statement:restart-time}
\restarttimeremark*

\paragraph{Sketch of proof of Remark~\ref{rem:restart-time}.}
Express the rare-entrance budget as a nonvisitation horizon and compare it with
the sufficient training length for uniform restartability
(Propositions~\ref{prop:restart-negative} and~\ref{prop:restart-positive}).

\paragraph{Main proof of Remark~\ref{rem:restart-time}.}
\label{proof:restart-time}
\begin{proof}
With $N=N_\ell=|E_\ell|$ fixed and the constant choice in the proof of
Proposition~\ref{prop:restart-negative}, rearranging \Eqref{eq:restart-budget} gives
\begin{equation}
 M'=\Omega\!\left(\frac1{KM\epsilon_1}\epsilon_2^{-N}\right),
 \qquad n=O\!\left(\frac{\delta M'}B\epsilon_2^{N}\right).
 \label{eq:restart-window}
\end{equation}
Thus the allowed nonvisit horizon grows linearly in $M'$ with other quantities fixed.
In addition to the first condition, if that horizon is $\omega$ of the sufficient
learning scale in \Eqref{eq:restart-explicit-n}, their ratio diverges.
The comparison uses common bounds on the other problem parameters and confidence
$1-\delta$ in each setting at each fixed chosen horizon.
\end{proof}

\begin{remark}[The scale of initial rarity]
\label{rem:restart-Mprime}
The sufficient conditions of Proposition~\ref{prop:restart-negative} imply
$M'=\Omega(M^2\log M)$, so $M'$ is perfectly separated from $M$.
\end{remark}

\paragraph{Sketch of proof of Remark~\ref{rem:restart-Mprime}.}
Derive the minimum batch budget imposed by the learning condition and substitute
it into the rarity requirement to obtain the coarse rarity scale
(Lemma~\ref{lem:restart-rarity-scale}).

\begin{lemma}[The coarse rarity scale]
\label{lem:restart-rarity-scale}
The sufficient conditions \Eqref{eq:restart-explicit-n} and
\Eqref{eq:restart-budget} imply $M'=\Omega(M^2\log M)$ as $M\to\infty$.
\end{lemma}
\begin{proof}
Since $\epsilon_1<1$, \Eqref{eq:restart-training} forces the drift to be positive. Squaring
its defining strict inequality gives
\[
 \underline\rho_{\ell,K, M}^{\,2}\left(\sum_{j=1}^n\eta_j\right)^2
 >\frac{\log(2KM/\delta)}{2B}\sum_{j=1}^n\eta_j^2.
\]
Cauchy--Schwarz, $(\sum_j\eta_j)^2\le n\sum_j\eta_j^2$, gives
\begin{equation}
 Bn>\frac{\log(2KM/\delta)}{2\underline\rho_{\ell,K, M}^{\,2}}.
 \label{eq:restart-Bn}
\end{equation}
Substitution in \Eqref{eq:restart-budget} gives
\begin{equation}
 M'=\Omega\!\left(\left[\epsilon_2^{-N}\right]
 \frac{\log(2KM/\delta)}{\delta\underline\rho_{\ell,K, M}^{\,2}}\right).
 \label{eq:restart-strong-M}
\end{equation}
Since $\underline\rho_{\ell,K, M}\le c_{\mathrm{route}}/M$ and $\epsilon_2^{-N}\ge1$,
for $M\to\infty$ this implies
\begin{equation}
 M'=\Omega(M^2\log M),\qquad M'/M\to\infty.
 \label{eq:restart-M2}
\end{equation}
The fuller budget \Eqref{eq:restart-budget} remains
required for Proposition~\ref{prop:restart-negative}.
\end{proof}

\paragraph{Main proof of Remark~\ref{rem:restart-Mprime}.}
\label{proof:restart-Mprime}
\begin{proof}
Lemma~\ref{lem:restart-rarity-scale} gives the claimed scale
$M'=\Omega(M^2\log M)$. This coarse implication does not replace the full budget
\Eqref{eq:restart-budget} required for Proposition~\ref{prop:restart-negative}.
\end{proof}

\begin{remark}[Correct paths bypassing the rare entrance]
\label{rem:restart-bypass}
On the event in the proof of Proposition~\ref{prop:restart-negative}, Equations~\ref{eq:restart-visit-growth} and~\ref{eq:restart-proof-budget}
bound the final visitation probability of $u^*=u_m$ by $KM\epsilon_1$.
Subtracting this probability from the original-start accuracy bound proved above gives
\begin{align*}
 &\pi_{\mathbf W^{(n)}}\!\left(
 y_T\in\tilde C_\ell,\ y_t\ne u^*\ \forall\,0\le t<T
 \,\middle|\,y_{\ell,0},x_\ell\right)\\
 &\qquad\ge 1-(3+2C_{\mathrm{bad}})KM\epsilon_1.
\end{align*}
For $(3+2C_{\mathrm{bad}})KM\epsilon_1<1$, this probability is positive, so a correct path avoiding $u^*$ exists.
By the unique-entrance condition in Assumption~\ref{ass:restart-rare}, such a path also bypasses $C_k$.
\end{remark}

\section{Numerical experiment}\label{app:numerical}
The code for experiments \textbf{(A)} and \textbf{(B)} is available in the
\href{https://github.com/yuukiizayoi/continual-reasoning-experiments}{\texttt{continual-reasoning-experiments} GitHub repository}.
The settings and results for both experiments are detailed below.
{\color{black}
We conducted two kind of experiments. One is the experiment to validate the effectiveness of theoretical bounds in
Theorem~1 to 4. The other aims to support our theoretical remarks and our intuitions written in papers by using real-world LLM and datasets.
We refer to these two experiments as \textbf{(A)} and \textbf{(B)}, respectively, in what follows.
}

\subsection{Experiment settings of \textbf{(A)}}
\paragraph{Graph construction and initialization}
We instantiate the layered DAG constructed in Appendix~\ref{app:realizability}.
The graph parameters and the choices used to construct $\mathbf W^{(0)}$
are summarized separately in Tables~\ref{tab:experiment-graph-parameters}
and~\ref{tab:experiment-initialization-parameters}.
{\color{black}We set Task~1 to 4 as searching tasks.
Each task has two components, task signal $x$ and their target terminal cluster $\tilde{C}$.
Each goal is to reach its corresponding cluster $\tilde{C}$, so the model is required in our experiments to acquire all effective transitions
through continual learning. without forgetting.
}
All tasks use the same graph and all seeds use the same deterministic initialization.

\begin{table}[ht]
\centering
\begin{minipage}[t]{0.48\textwidth}
\vspace{0pt}
\centering
\caption{Graph parameters.}
\label{tab:experiment-graph-parameters}
\small
\setlength{\tabcolsep}{3pt}
\renewcommand{\arraystretch}{1.18}
\begin{tabular}{@{}p{0.60\linewidth}p{0.34\linewidth}@{}}
\hline
Parameter & Value\\
\hline
Tasks $L$ & $4$\\
Clusters $K$ & $13$\\
Node layers $(H+1)$ & $5$\\
Nodes $|V|$ & $123$\\
Directed edges $|E|$ & $214$\\
Correct paths $|\mathcal P_\ell|$ & $10$ per task\\
Maximum valid boundaries $M$ & $10$\\
Boundary nodes $\sum_k|\partial C_k|$ & $53$\\
\hline
\end{tabular}
\end{minipage}
\hfill
\begin{minipage}[t]{0.48\textwidth}
\vspace{0pt}
\centering
\caption{Initialization parameters.}
\label{tab:experiment-initialization-parameters}
\small
\setlength{\tabcolsep}{3pt}
\renewcommand{\arraystretch}{1.18}
\begin{tabular}{@{}p{0.60\linewidth}p{0.34\linewidth}@{}}
\hline
Parameter & Value\\
\hline
Routing error parameter $\epsilon_1$ & $0.02$\\
Inter-cluster probability $\epsilon_2$ & $0.04$\\
Input dimension $d$ & $36$\\
Task-specific input coordinates & $4$-D one-hot\\
Input squared norm $\|x_\ell\|_2^2$ & $33$\\
Residual successors $|R_k(u)|$ & $1$\\
Each start-to-entry probability & $1/10$\\
Boundary residual mass & $0.92$ or $0.88$\\
\hline
\end{tabular}
\end{minipage}
\end{table}

At a task start, $|\partial_\ell^\star C_k|=10$ and
$I_{\ell,k}(s)=\varnothing$, so Equation~\ref{eq:routing-start-kernel}
assigns probability $1/10$ to each entry boundary.
At a relevant shared router, $|\partial_\ell^\star C_k|=1$ and the remaining
successors form $I_{\ell,k}(s)$; the correct exit receives probability
$1/(1+\epsilon_1)$ and each other exit receives
$\epsilon_1/((1+\epsilon_1)|I_{\ell,k}(s)|)$.
For task inputs not using that router, its outgoing distribution is uniform.
At every boundary, Equation~\ref{eq:routing-boundary-kernel} assigns
probability $\epsilon_2$ to each inter-cluster edge and probability
$1-d_{u,k}^{\mathrm{out}}\epsilon_2$ to its unique same-cluster successor.
There are $51$ boundaries with inter-cluster outdegree $2$ and $2$ with outdegree $3$.
With $K=13$ and $\epsilon_2=0.04$, the sufficient condition in
Appendix~\ref{app:realizability} holds since $(K-1)\epsilon_2=0.48<1/2$.

The input matrix is $X=[\mathbf 1_{4\times32}\;\;I_4]$, where the first block
has all entries equal to one. Thus $XX^\top$ has diagonal entries $33$
and off-diagonal entries $32$ and is invertible.
We use Equation~\ref{eq:routing-parameter-construction} directly:
\begin{align}
\mathbf W^{(0)}e_{u,v}
&=X^\top(XX^\top)^{-1}\ell_{u,v},
&
\ell_{u,v}
&=\bigl(\log P_1(v\mid u),\ldots,\log P_4(v\mid u)\bigr)^\top.
\end{align}
The hint parameter $\epsilon_{\mathrm{rel}}=0.1$ is used during SD training,
not to construct $\mathbf W^{(0)}$.

\paragraph{Pairwise and cumulative conflict rates.}
Tables~\ref{tab:experiment-pairwise-theta} and~\ref{tab:experiment-total-theta}
report conflict rates, with rows denoting evaluated old tasks and columns
 denoting the tasks being learned. The pairwise rate $\theta$ compares only
 the row task and the column task. The cumulative rate $\theta_{\mathrm{total}}$
 is the fraction of the row task's correct paths that conflict with at least
 one task learned after it, up to and including the column task.
Thus, $\theta_{\mathrm{total}}$ is computed from a union of conflict-path sets:
 overlapping paths are counted only once, and the pairwise rates are not added.
Diagonal entries are set to zero; entries below the diagonal are omitted.

\begin{table}[ht]
\centering
\begin{minipage}[t]{0.48\textwidth}
\vspace{0pt}\centering
\caption{Pairwise conflict rate $\theta$.}
\label{tab:experiment-pairwise-theta}
\small
\setlength{\tabcolsep}{4pt}
\renewcommand{\arraystretch}{1.15}
\begin{tabular}{lcccc}
\hline
Old / New & Task 1 & Task 2 & Task 3 & Task 4\\
\hline
Task 1 & $0$ & $0.4$ & $0.6$ & $0.6$\\
Task 2 &  & $0$ & $0.9$ & $0.9$\\
Task 3 &  &  & $0$ & $0.7$\\
Task 4 &  &  &  & $0$\\
\hline
\end{tabular}
\end{minipage}
\hfill
\begin{minipage}[t]{0.48\textwidth}
\vspace{0pt}\centering
\caption{Cumulative conflict rate $\theta_{\mathrm{total}}$.}
\label{tab:experiment-total-theta}
\small
\setlength{\tabcolsep}{4pt}
\renewcommand{\arraystretch}{1.15}
\begin{tabular}{lcccc}
\hline
Old / New & Task 1 & Task 2 & Task 3 & Task 4\\
\hline
Task 1 & $0$ & $0.4$ & $0.7$ & $1.0$\\
Task 2 &  & $0$ & $0.9$ & $0.9$\\
Task 3 &  &  & $0$ & $0.7$\\
Task 4 &  &  &  & $0$\\
\hline
\end{tabular}
\end{minipage}
\end{table}

For example, Task~1 has pairwise conflict rates $0.4$ and $0.6$ with Tasks~2
and~3, respectively, whereas their cumulative conflict rate is $0.7$.
In the plots, the term labeled ``Limit'' uses
$\mathfrak C_{\mathrm{old}}(1-\theta_{\mathrm{total}})$,
with the concentration profile from Equation~\ref{eq:theorem-four-concentration-profile}
constructed from the old task's correct-path probabilities
at the end of its own training.
This profile is held fixed throughout subsequent-task training,
while $\theta_{\mathrm{total}}$ is updated to reflect the cumulative conflict.
Thus the plotted term uses the cumulative rate, rather than the pairwise
$\theta$, and is computed separately for each seed before averaging.
The label refers to this concentration-profile term, not necessarily to the
exact limiting correct generation probability.

\paragraph{Training and evaluation}
We train four tasks sequentially with $2\times10^6$ gradient updates per task,
using mini-batches of size $B=32$ and seeds $0,\ldots,9$ for each method.
The step size is $\eta_j=11^{-1}j^{-0.01}$, where $j$ starts at $1$ for each task.
The factor \(11^{-1}=3/\|x_\ell\|_2^2\) controls the update magnitude, while \(j^{-0.01}\)
provides a slow decay over successive updates.
Plots show the mean correct generation probability and the minimum--maximum
range across the ten seeds. Each probability is evaluated by summing the
probabilities of all ten correct paths.
For SD, the graph-specific population learning lower bound is shown during
own-task training. At its end, the plot switches to the conditional retention
lower bound anchored at the observed endpoint for each seed.
The population learning curve is evaluated at $B\to\infty$ for the stability.

\paragraph{Graph-specific evaluation of constants.}
We tighten intermediate bounds in the Appendix proofs using the experimental graph, without altering the theorem statements.
For Theorem~\ref{thm:sd-no-forgetting}, we use edge-specific hint and visitation probabilities in the odds bound and sum correct-path probabilities instead of applying the coarse coupling union bound.
For Theorem~\ref{thm:sft-forgetting}, we evaluate the source-wise optimization and cross-task probability-transfer bounds using the actual visitation frequencies, logit offsets, and unchanged path factors.
These refinements retain the underlying inequalities while avoiding worst-case estimates over the general graph family.
The SD learning and SFT upper reference curves are evaluated in the population limit $B\to\infty$. The SFT conditional lower and SD retention bounds are instead evaluated at the observed phase endpoints, separately for each seed.


\subsection{Experiment results of \textbf{(A)}}
Table~\ref{tab:experiment-final-probabilities} reports the correct-generation
probabilities after all four tasks have been learned.
SFT reaches a mean probability above $0.999997$ at the end of each task's own
training, but its final means on Tasks~1--3 fall to $0.002799$, $0.100312$,
and $0.304232$, respectively. SD attains approximately $0.980386$ on every
task and preserves its own-training endpoint through the final task for
every seed, to stored numerical precision. The final seed-to-seed range
for SD is below $1.2\times10^{-8}$ for each task.

\begin{table}[ht]
\centering
\caption{Final correct-generation probabilities after all four tasks.
Entries are means [minimum, maximum] over ten seeds, rounded to six decimal places.}
\label{tab:experiment-final-probabilities}
\small
\setlength{\tabcolsep}{5pt}
\renewcommand{\arraystretch}{1.15}
\begin{tabular}{lcc}
\hline
Task & SFT & SD\\
\hline
Task~1 & $0.002799\;[0.002701,\,0.002955]$ & $0.980386\;[0.980386,\,0.980386]$\\
Task~2 & $0.100312\;[0.065468,\,0.149365]$ & $0.980386\;[0.980386,\,0.980386]$\\
Task~3 & $0.304232\;[0.263488,\,0.334231]$ & $0.980386\;[0.980386,\,0.980386]$\\
Task~4 & $0.999998\;[0.999998,\,0.999998]$ & $0.980386\;[0.980386,\,0.980386]$\\
\hline
\end{tabular}
\end{table}

All saved evaluations satisfy the applicable SFT conditional lower and SD
retention bounds, within numerical tolerance. The SFT trajectories also
remain below the population upper references at the saved evaluations.
For SD, $12$ of $12{,}400$ saved learning evaluations lie below the population
learning reference, with a maximum shortfall of $7.83\times10^{-6}$.
This distinction is consistent with using $B=32$ for training: the
$B\to\infty$ learning reference does not guarantee finite-mini-batch accuracy.

The following graphs are enlarged versions of the plots in Figure~\ref{fig:training-dynamics}.
\begin{figure}[!htbp]
   \includegraphics[width=1.0\textwidth]{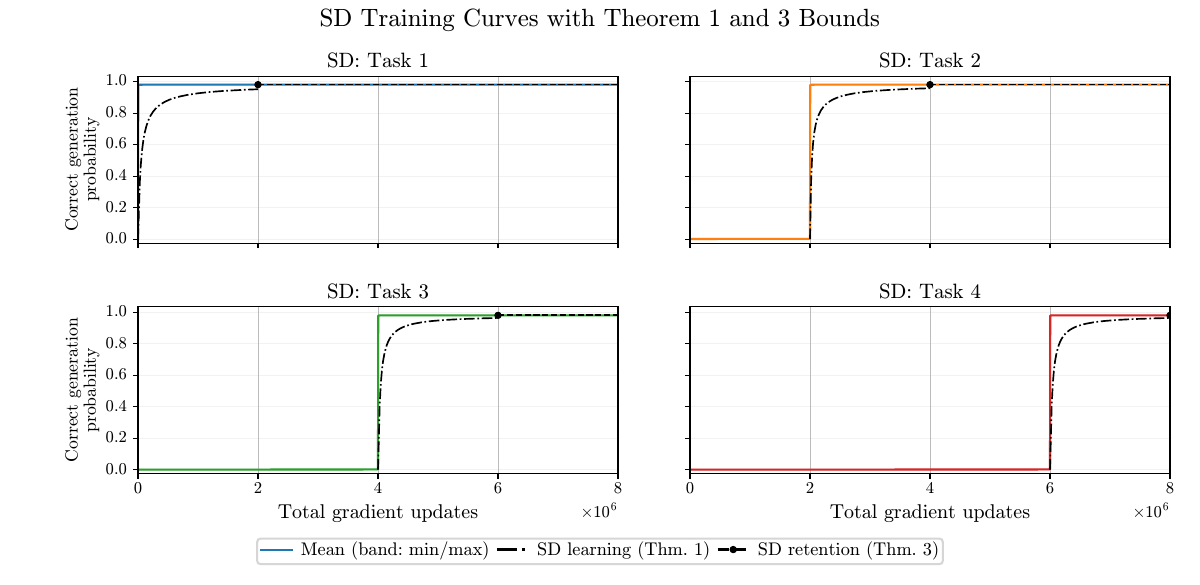}
   \caption{Training dynamics under SD for Tasks 1--4.
   The population learning lower reference is plotted during own-task training.
   At its end, the conditional retention lower bound is anchored at the observed
   endpoint and applies throughout subsequent-task training.
   }
   \label{fig:sd-training-dynamics}
\end{figure}

\begin{figure}[!htbp]
   \centering
   \includegraphics[width=1.0\textwidth]{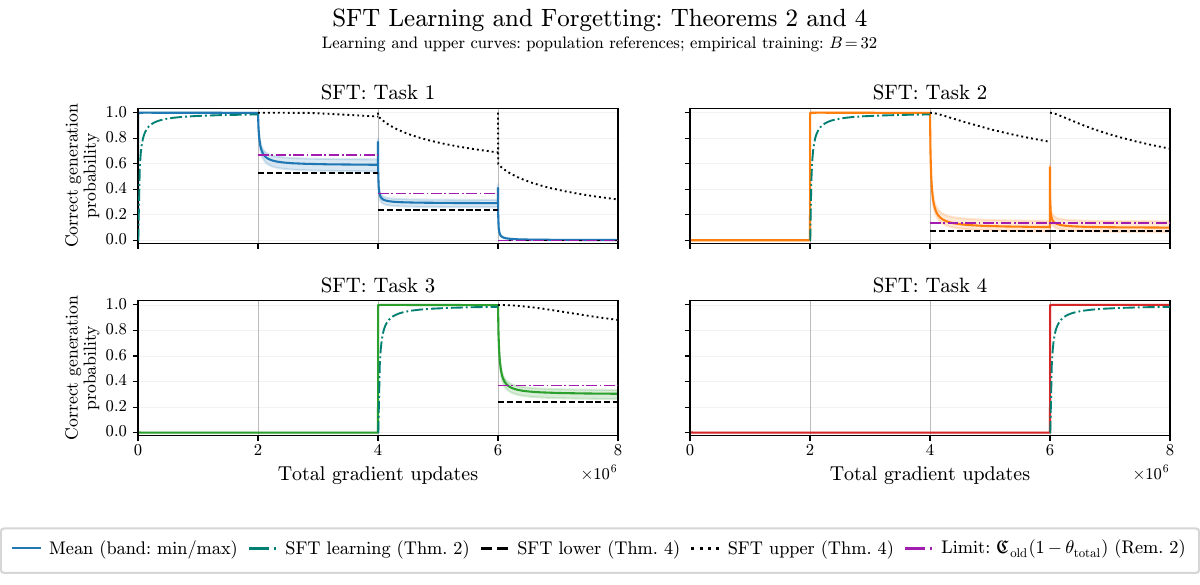}
   \caption{Training dynamics under SFT for Tasks 1--4. 
   Black curves show conditional retention lower bounds and upper
   references.
   The concentration-profile term
   $\mathfrak{C}_{\mathrm{old}}(1-\theta_{\mathrm{total}})$ reflects the cumulative
   conflict and, for Task~1, the accumulation of forgetting across subsequent tasks. 
   }
   \label{fig:sft-training-dynamics}
\end{figure}

\clearpage
\begingroup
\color{black}
\subsection{Experiment settings of \textbf{(B)}}
\label{app:empirical-settings}
We study continual learning on a shared reasoning graph derived from Qwen3-1.7B.
The model first learns Algebra transitions and then learns Probability transitions.
The experiment measures new-task acquisition, old-task retention, and changes
in transition distributions across the same graph.

\paragraph{Construction trajectories and reasoning states.}
We use 100 Algebra and 100 Counting and Probability problems at difficulty
levels 2--4 from the MATH training partition \citep{hendrycks2021math}.
Qwen3-1.7B supplies both the construction trajectories and the trainable policy
\citep{yang2025qwen3}. For each problem, we generate 16 construction CoTs and
eight held-out CoTs, totaling 4,800 trajectories. Sampling uses temperature
$0.7$, top-$p=0.8$, top-$k=20$, a 2,048-token limit, batches of eight, BF16,
and the model's non-thinking mode. The instruction requests one logical step
per line and a final boxed answer; Math-Verify compares that answer with the reference.

Following \citet{wang2024understanding}, we represent each reasoning prefix
by the mean final-layer hidden state over the question, chat template, and
generated steps up to that point. Line and sentence boundaries define steps,
with displayed mathematical expressions kept together. We extract
2,048-dimensional representations using the frozen model in FP32.
K-means assigns the 87,112 construction prefixes to 40 parent clusters
(20 initializations), followed by four child clusters per parent
(10 initializations); both stages use Lloyd updates and at most 300 iterations.
As in \citet{minegishi2025topology}, consecutive cluster assignments define
directed edges. The resulting graph has 160 internal states and 966 edges.
Four added terminals indicate a correct Probability answer (S160), a correct
Algebra answer (S161), an incorrect answer (S162), or an unparseable answer (S163).
EOS-terminated CoTs supply terminal edges. We assign the 43,618 held-out prefix
representations to the clusters learned from the construction data.

\begin{figure}[htbp]
\color{black}
\centering
\resizebox{\linewidth}{!}{
\begingroup
\color{black}
\begin{tikzpicture}[>=Stealth,font=\small,text=black,
  state/.style={circle,draw=black,fill=white,minimum size=9.5mm,inner sep=1pt,line width=.5pt},
  terminal/.style={rectangle,draw=black,fill=white,minimum width=13mm,minimum height=8.5mm,line width=.5pt},
  link/.style={->,draw=black,line width=.5pt},
  opt/.style={font=\small,fill=white,inner sep=1pt,text=black}]
\node[anchor=west,font=\small] at (0,.65) {(a) Policy input};
\node[anchor=north west,inner sep=0pt,text width=5.9cm] (prompt) at (0,0) {%
  {\ttfamily\scriptsize
  Task: counting\_and\_probability\\
  \textbf{Goal terminal cluster: S160}\\
  \textbf{Current state: S141}\\
  Region: A035\\[3pt]
  \textbf{Legal options:}\\
  A: S096\\
  B: S099\\
  C: S141\\
  D: S161\\[3pt]
  For task counting\_and\_probability,\\
  choose the next state toward goal S160.\\[3pt]
  Return one legal option label.}};
\draw[line width=.45pt] ($(prompt.north west)+(0,.18)$) -- ($(prompt.north east)+(0,.18)$);
\draw[line width=.45pt] ($(prompt.south west)+(0,-.16)$) -- ($(prompt.south east)+(0,-.16)$);
\node[anchor=west,font=\small] at (7.0,.65) {(b) Graph excerpt};
\begin{scope}[shift={(7.65,-1.65)}]
  \node[state,line width=.9pt] (u) at (0,0) {S141};
  \node[state] (a) at (1.9,.95) {S096};
  \node[state] (b) at (1.9,-1.25) {S099};
  \node[terminal] (old) at (4.25,.95) {S161};
  \node[terminal] (target) at (4.25,-1.25) {S160};
  \draw[link] (u) -- node[opt,above,sloped]{A} (a);
  \draw[link,line width=1pt] (u) -- node[opt,below,sloped]{B} (b);
  \draw[link] (u.145) .. controls (-.85,1.35) and (-.05,1.7) .. (u.105);
  \node[opt] at (-.62,1.15) {C};
  \draw[link] (u.60) .. controls (1.2,2.05) and (3.6,2.05) .. (old.north);
  \node[opt] at (2.55,1.98) {D};
  \draw[link] (a) -- (old);
  \draw[link,line width=1pt] (b) -- (target);
  \draw[link] (b) -- (old);
  \node[font=\scriptsize,align=left,anchor=west] at (5.0,.95) {Algebra\\terminal};
  \node[font=\scriptsize,align=left,anchor=west] at (5.0,-1.25) {Probability\\goal};
  \node[font=\scriptsize,align=center] at (0,-.75) {Current state};
\end{scope}
\end{tikzpicture}
\endgroup}
\caption{A policy input and the corresponding excerpt of the empirical CoT graph.
Labels A--D identify the four outgoing choices at current state S141.
The requested Probability goal is S160; the bold route S141--S099--S160 reaches it.
Every displayed edge occurs in the constructed graph.}
\label{fig:empirical-input}
\end{figure}

\paragraph{Input and actions.}
Each decision uses the task, goal terminal, current state, parent region,
and legal successors shown in Figure~\ref{fig:empirical-input}.
This explicit state--goal conditioning follows the design principle of
\citet{schaul2015uvfa}. Successors are sorted by ID and mapped to single-token
option labels. The LLM logits are normalized over these labels, and the
sampled label deterministically specifies the next state.
Rollouts use temperature one and a horizon of $H=192$ transitions.
Reaching the task's correct terminal within this horizon defines success.

\paragraph{Data splits and old-task initialization.}
After fixing the graph, we split the 100 problem IDs in each domain into
60 training, 20 calibration, and 20 test problems.
Training uses their construction trajectories; calibration and test each use
160 held-out trajectories per domain. Each trajectory's first generated
step defines its start state. Their frequencies define the split-specific
start distributions on the fixed graph.
Algebra transition counts yield a reference policy: empirical frequencies
receive weight $0.999$ and the uniform legal-action distribution receives
weight $0.001$. States with zero Algebra counts use the uniform distribution.
We fit Qwen3 to this reference by forward KL, drawing 16 states per update
with AdamW at learning rate $10^{-5}$. Calibration is checked every 25 updates,
requiring mean KL at most $0.08$ and goal-reaching probability within $0.05$
of the reference. The shared checkpoint $\theta_A$ satisfies these criteria
after 175 updates.

\paragraph{Selecting the high-interference training pool.}
Following gradient-based analyses of continual learning
\citep{lopezpaz2017gem}, we measure interference between a new-task update
and the previous task. Let $J_A$ denote Algebra calibration success and
$\Delta\theta_B$ one candidate SFT update from $\theta_A$. We compute
\begin{equation}
 I_B=\nabla_\theta[-\log J_A(\theta_A)]^\top\Delta\theta_B .
 \label{eq:empirical-interference}
\end{equation}
We evaluate 16 candidate batches of eight successful Probability training
paths, each with a fresh optimizer and the same initialization.
Eight candidates emphasize starts with high Probability calibration mass;
eight use mixed batches. Paths satisfy the stored 4,096-token teacher-template
filter. Candidates with $I_B>0$ and measured Algebra degradation exceeding
$10^{-6}$ are ranked by that degradation. The top four yield 30 distinct
paths from 18 problems and 10 start states. A four-update SFT confirmation
reduces Algebra calibration success from $90.84\%$ to $88.00\%$.
Initialization and pool selection use the domain/state prompt;
all subsequent comparisons use the explicit-goal input in
Figure~\ref{fig:empirical-input}, the same $\theta_A$, and this fixed pool.

\begin{table}[htbp]
\color{black}
\centering\small
\setlength{\tabcolsep}{4pt}
\renewcommand{\arraystretch}{1.12}
\hyphenpenalty=10000
\caption{Paired SD and SFT settings for Probability adaptation.}
\label{tab:empirical-controls}
\begin{tabular}{@{}>{\raggedright\arraybackslash}p{.22\linewidth}
>{\raggedright\arraybackslash}p{.36\linewidth}
>{\raggedright\arraybackslash}p{.36\linewidth}@{}}
\toprule
Item & SD & SFT\\
\midrule
Initialization & Common $\theta_A$; EMA initialized to $\theta_A$ & Common $\theta_A$\\
Environment/input & Fixed graph; task, goal, state;\newline legal options & Same\\
Training examples & Eight student rollouts\newline per update & Eight successful\newline demonstrations per update\\
Start sampling & Shared per-seed start sequence & Same sequence\\
Training targets & Locally corrected EMA distributions & Loop-erased path transitions\\
Loss & Reverse KL; unhinted\newline weight $\lambda$ & Cross entropy at every\newline retained transition\\
Normalization & Sum of visit losses\newline divided by eight & Sum of transition losses\newline divided by eight\\
Trainable weights & All 1.72B model parameters & Same\\
Optimizer & AdamW, learning rate $10^{-6}$\newline \mbox{$(\beta_1,\beta_2)=(0.9,0.999)$} & Same\\
Numerics & FP32; weight decay $0$;\newline norm clip $1$ & Same\\
Training length & 16 optimizer updates & Same\\
\bottomrule
\end{tabular}
\end{table}

\paragraph{SFT targets and paired sampling.}
We apply chronological loop erasure to each successful demonstration,
retaining its start, correct goal, and observed edges. For example,
$(a,b,b,c,b,d,g)$ becomes $(a,b,d,g)$.
SFT minimizes the summed negative log probability of every remaining
transition, divided by the eight trajectories in its batch.
For each seed, we sample a pool path uniformly to obtain a start, then sample
a demonstration among paths sharing that start. SD uses those starts for
current-student rollouts, following on-policy distillation
\citep{agarwal2024onpolicy}. We use four seeds, which vary minibatches and
student rollouts while keeping the graph, pool,
initial model, and hyperparameters fixed.

\paragraph{Local improvement signal.}
Let $d(u)$ be the shortest directed distance from state $u$ to the Probability
goal, computed by reverse breadth-first search. At rollout step $t$, set
$h=H-t$ and define the feasible successors
\begin{equation}
 \mathcal V(u,h)=\{v\in N(u):d(v)\le h-1\}.
 \label{eq:empirical-feasible}
\end{equation}
A hint is triggered when the student samples $u\to v$ with $d(u)\le h$
and $v\notin\mathcal V(u,h)$: the transition loses the possibility of timely
goal completion. Among feasible successors with $d(v)<d(u)$, we promote the
one with highest current student probability, denoted $v^*$.
This is the first edge of a shortest repair route. Each visit is checked
using its remaining horizon, so repeated visits receive the same decision rule.
The hint acts through a numerical correction to the teacher's distribution
at that visit. This implements feedback-informed self-teaching
\citep{hubotter2026sdpo} with an explicitly identified corrective transition.

\paragraph{EMA teacher and bounded correction.}
Let $r_\phi(\cdot\mid u,B)$ be the EMA teacher's distribution over legal
successors for the Probability task. After each student update, we set
$\phi\leftarrow0.95\phi+0.05\theta$, following the teacher-stabilization
principles of \citet{shenfeld2026sdft,hubotter2026sdpo}.
The local hint gives likelihood one to $v^*$ and $\epsilon=10^{-6}$ to the
other successors. Its tempered target is
\begin{equation}
 q_\beta(v\mid u,B)=
 \frac{r_\phi(v\mid u,B)\exp(\beta\,\mathbf1\{v=v^*\})}
 {\sum_{w\in N(u)}r_\phi(w\mid u,B)\exp(\beta\,\mathbf1\{w=v^*\})},
 \qquad 0\le\beta\le\log(1/\epsilon).
 \label{eq:empirical-hint}
\end{equation}
We take the largest $\beta$ satisfying
$\mathrm{KL}(q_\beta\Vert r_\phi)\le0.1$ nats.
This uses the KL-constrained exponential-reweighting structure of
\citet{abdolmaleki2018mpo} with a one-edge improvement signal.
When required, 48 float64 bisection steps determine $\beta$.
The nonpromoted actions retain their relative
odds. At unhinted visits the target remains $r_\phi$.
For example, the first update of the original-seed $\lambda=0.01$ run
promotes S141--S099: its teacher probability increases from $0.04935$ to
$0.17173$, with the route S141--S099--S160 verifying goal reachability.

\paragraph{SD loss and the localization coefficient.}
For rollout $b$ and visit $t$, let $e_{bt}$ indicate a hinted visit,
$q_{bt}$ its target, and $p_{bt}=\pi_\theta(\cdot\mid u_{bt},B)$.
We use the reverse-KL self-distillation objective
\citep{shenfeld2026sdft,hubotter2026sdpo}:
\begin{equation}
 \mathcal L_{\rm SD}
 =\frac{1}{8}\sum_{b=1}^{8}\sum_t
 \bigl[e_{bt}+\lambda(1-e_{bt})\bigr]\,
 \mathrm{KL}\!\left(p_{bt}\Vert\operatorname{stopgrad}(q_{bt})\right),
 \quad
 q_{bt}=
 \begin{cases}
 q_\beta,&e_{bt}=1,\\
 r_\phi,&e_{bt}=0.
 \end{cases}
 \label{eq:empirical-sd}
\end{equation}
We evaluate $\lambda\in\{1,0.1,0.01,0\}$.
The coefficient controls consistency pressure at unhinted student visits;
hinted visits always have unit weight. Thus $\lambda=0$ concentrates the
direct loss on hinted visits, whereas $\lambda=1$ includes every visit equally.
We evaluate KL over all legal actions, detach teacher targets and sampled
trajectories, accumulate gradients in microbatches of two, clip the global
norm, update the student, and then update the EMA teacher.

\paragraph{Evaluation and aggregation.}
For each task, we form the student's complete transition table.
Dynamic programming computes success within 192 transitions, treating the
requested goal as an absorbing success and the other terminals as failures.
We average these probabilities under the fixed test-start distribution.
New-task gain and old-task change are differences from update zero in
percentage points. Each curve reports the mean and sample standard deviation
over four seeds. Each condition uses the saved checkpoints common to its
four runs; all 20 runs complete 16 updates. Training uses one A100 80GB per
run, PyTorch 2.8.0, and Transformers 4.56.2.

\paragraph{Transition-distribution localization.}
Let $U$ be the 114 internal states that receive zero hints across all
16 SD runs, held fixed at every checkpoint and shared with SFT.
For domain $d$, we report
$D_d(k)=|U|^{-1}\sum_{u\in U}\mathrm{KL}
\bigl(P_{0,d}(u,\cdot)\Vert P_{k,d}(u,\cdot)\bigr)$.
Legal-action rows are normalized and compared using natural logarithms.
We first average over states within each run and then aggregate the four seeds.

\subsection{Experiment results of \textbf{(B)}}
\label{app:empirical-results}
The common initialization achieves $50.51\%$ Probability and $94.86\%$
Algebra success. Figure~\ref{fig:empirical-accuracy} shows adaptation,
and Table~\ref{tab:empirical-final} reports final performance.

\begin{figure}[htbp]
\color{black}
\centering
\includegraphics[width=\linewidth]{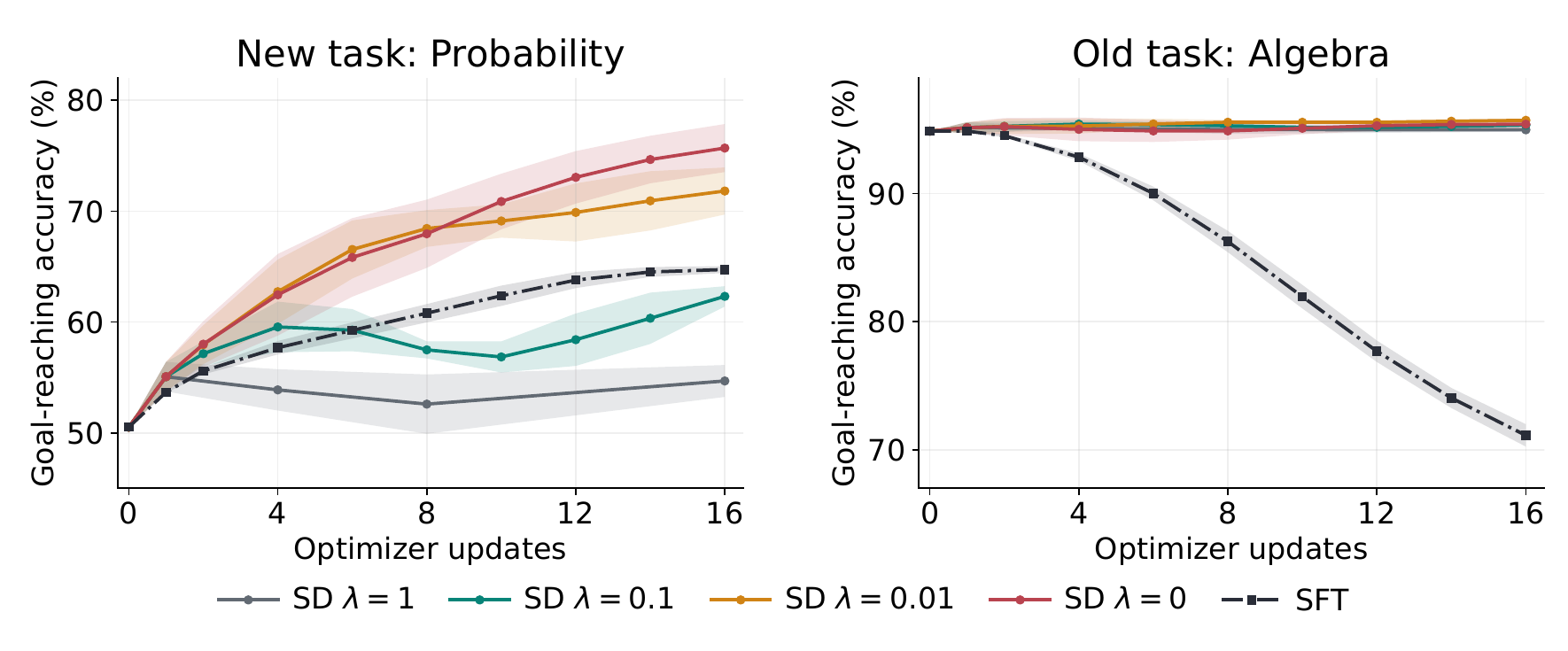}
\caption{New-task learning and old-task retention during Probability training.
Lines are four-seed means; shaded bands are one sample standard deviation.
}
\label{fig:empirical-accuracy}
\end{figure}

\begin{table}[htbp]
\color{black}
\centering\small
\setlength{\tabcolsep}{4pt}
\caption{Final performance after 16 updates: four-seed mean $\pm$ sample
standard deviation. Changes are percentage points (pp) from initialization.}
\label{tab:empirical-final}
\begin{tabular}{@{}lrrrr@{}}
\toprule
Method & Probability (\%) & Algebra (\%) & New gain (pp) & Old change (pp)\\
\midrule
SD $\lambda=1$ & $54.68\pm1.44$ & $94.97\pm0.15$ & $+4.16$ & $+0.10$ \\
SD $\lambda=0.1$ & $62.31\pm0.92$ & $95.33\pm0.16$ & $+11.80$ & $+0.47$ \\
SD $\lambda=0.01$ & $71.82\pm2.12$ & $95.71\pm0.13$ & $+21.30$ & $+0.84$ \\
SD $\lambda=0$ & $75.69\pm2.17$ & $95.40\pm0.39$ & $+25.18$ & $+0.54$ \\
SFT & $64.73\pm0.30$ & $71.12\pm0.87$ & $+14.22$ & $-23.74$ \\
\bottomrule
\end{tabular}

\end{table}

\begin{figure}[htbp]
\color{black}
\centering
\includegraphics[width=\linewidth]{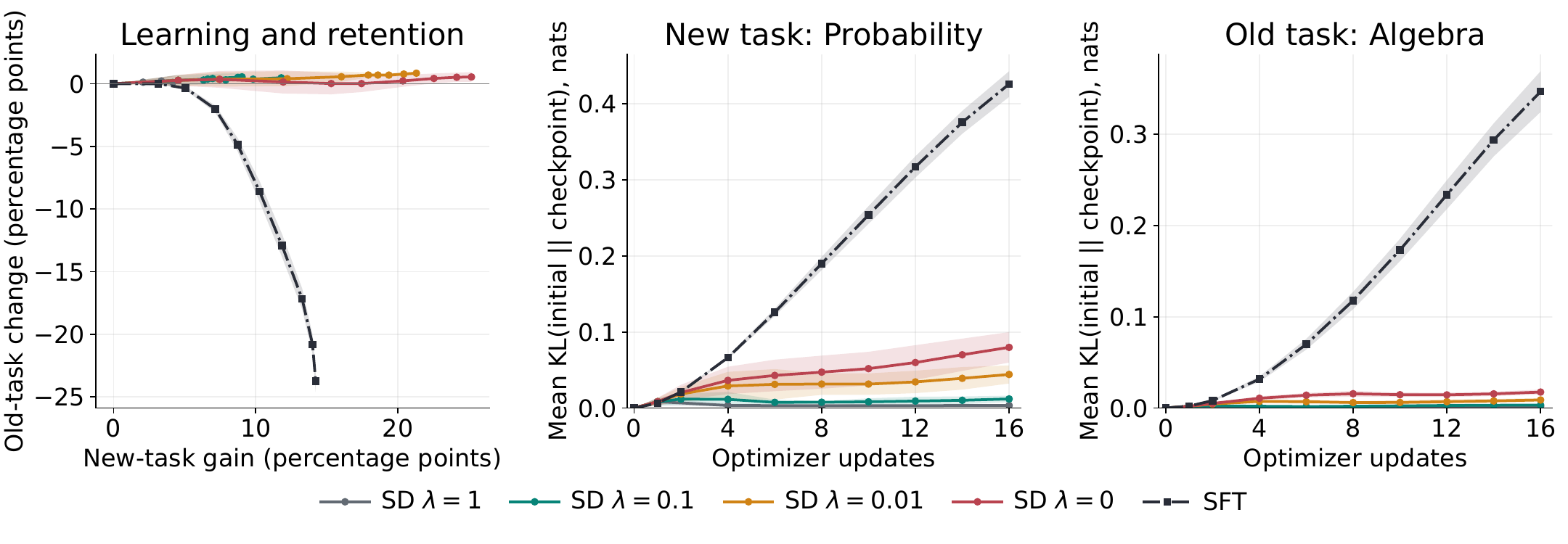}
\caption{Learning--retention trade-off and transition-distribution drift.
The left band is the sample standard deviation of old-task change at each
checkpoint's mean new-task gain. The other bands are sample standard deviations
of mean KL over the 114 common unhinted states. Each curve uses four seeds.}
\label{fig:empirical-full-comparison}
\end{figure}

SD gains $4.16$--$25.18$ pp on Probability while maintaining mean Algebra
performance between $94.97\%$ and $95.71\%$.
SFT gains $14.22$ pp on Probability and loses $23.74$ pp on Algebra.
Smaller $\lambda$ strengthens new-task acquisition in this sweep.
The small off-hint changes in Figure~\ref{fig:empirical-full-comparison}
support the connection between local correction, concentrated distributional
adaptation, and retention.
\endgroup

\end{document}